%% file: ms.tex
\documentclass[10pt,journal,compsoc]{IEEEtran}

\ifCLASSOPTIONcompsoc
  \usepackage[nocompress]{cite}
\else
  \usepackage{cite}
\fi
\ifCLASSINFOpdf
\else
\fi
\usepackage{array}
\usepackage{url}
\input{notation}
\begin{document}
%
\title{Robust Kronecker Component Analysis}
%
%
%
%

\author{Mehdi~Bahri,~\IEEEmembership{Student Member,~IEEE,}
        Yannis~Panagakis,
        and~Stefanos~Zafeiriou,~\IEEEmembership{Member,~IEEE}
\IEEEcompsocitemizethanks{\IEEEcompsocthanksitem The authors are with the Department
of Computing, Imperial College London, London,
UK, SW7 2RH.\protect\\
E-mail: mehdi.bahri15@imperial.ac.uk
\IEEEcompsocthanksitem Y. Panagakis is also with Middlesex University London, UK.
\IEEEcompsocthanksitem S. Zafeiriou is also with the University of Oulu, Finland.}
\thanks{Manuscript received April 19, 2005; revised August 26, 2015.}}

%
%

\markboth{Journal of \LaTeX\ Class Files,~Vol.~14, No.~8, August~2015}%
{Shell \MakeLowercase{\textit{et al.}}: Bare Advanced Demo of IEEEtran.cls for IEEE Computer Society Journals}
\IEEEtitleabstractindextext{%
\begin{abstract}
Dictionary learning and component analysis models are fundamental for learning compact representations that are relevant to a given task (feature extraction, dimensionality reduction, denoising, etc.). The model complexity is encoded by means of specific structure, such as sparsity, low-rankness, or nonnegativity.
Unfortunately, approaches like K-SVD - that learn dictionaries for sparse coding via Singular Value Decomposition (SVD) - are hard to scale to high-volume and high-dimensional visual data, and fragile in the presence of outliers. Conversely, robust component analysis methods such as the Robust Principal Component Analysis (RPCA) are able to recover low-complexity (e.g., low-rank) representations from data corrupted with noise of unknown magnitude and support, but do not provide a dictionary that respects the structure of the data (e.g., images), and also involve expensive computations. In this paper, we propose a novel Kronecker-decomposable component analysis model, coined as Robust Kronecker Component Analysis 
(RKCA), that combines ideas from sparse dictionary learning and robust component analysis. RKCA has several appealing properties, including robustness to gross corruption; it can be used for low-rank modeling, and leverages separability to solve significantly smaller problems. We design an efficient learning algorithm by drawing links with a restricted form of tensor factorization, and analyze its optimality and low-rankness properties. The effectiveness of the proposed approach is demonstrated on real-world applications, namely  background subtraction and image denoising and completion, by performing  a thorough comparison with the current state of the art.
\end{abstract}

}

\maketitle

\IEEEdisplaynontitleabstractindextext

%
\IEEEpeerreviewmaketitle

\input{introduction}

\input{preliminaries}

\input{model_derivation}

\input{degree_3_prov}

\input{low_rankness}

\input{implementation_details}

\input{experimental_evaluation/experimental_evaluation}

\section{Conclusion}
In this work we presented RKCA, a framework for robust low-rank representation learning on separable dictionaries that expresses the problem as a tensor factorization. We proposed three different regularizers and derived the steps for both ADMM and LADMM solvers. We showed the factorization can be expressed in an equivalent way that gives optimality guarantees when coupled with a regularizer that is positively-homogeneous of degree 3. We then further discussed the low-rank properties of the models, and their practical implementation. We reached the conclusion that our model with a degree 2 regularizer, achieved a good trade-off between experimental effectiveness and parameter tuning difficulty.

Future work can seek to develop a supervised variant of RKCA that could be applied to the same tasks. Another extension that we leave for future work is to consider non-convex regularizers: it is well-known that letting $p \rightarrow 0$ in the element-wise $\ell_p$ and in the Scatten-$p$ norms recover respectively the $\ell_0$ pseudo-norm and the rank functions. We refer to \cite{papamakarios_generalised_2014} for an example in the matrix RPCA case. Finally, we believe RKCA could be extended with deep learning by replacing the factorization of the low-rank component by a deep neural network, such as an auto-encoder. Effectively, the bilinear factorization we use can be seen as an elementary linear auto-encoder, and replacing it with a more involved non-linear model would be straightforward.

\ifCLASSOPTIONcompsoc
  \section*{Acknowledgments}
\else
  \section*{Acknowledgment}
\fi
Mehdi Bahri was partially funded by the Department of Computing, Imperial College London. 
The work of Y. Panagakis has been partially supported by the European Community Horizon 2020 [H2020/2014-2020] under Grant Agreement No. 645094 (SEWA). 
S. Zafeiriou was partially funded by EPSRC Project EP/N007743/1 (FACER2VM) and also partially funded by a Google Faculty Research Award.

\ifCLASSOPTIONcaptionsoff
  \newpage
\fi


%



\bibliographystyle{IEEEtran}
\bibliography{IEEEabrv,bibliography}
%







\begin{IEEEbiography}[{\includegraphics[width=1in,height=1.25in,clip,keepaspectratio]{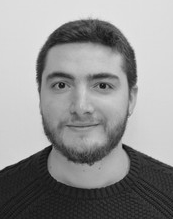}}]{Mehdi Bahri}

is a PhD student with Stefanos Zafeiriou at the Department of Computing, Imperial College London. 
He received the Dipl\^ome d'Ing\'enieur in Applied Mathematics with Honours from Grenoble Institute of Technology - Ensimag and the MSc in Advanced Computing with Distinction from Imperial College London, both in 2016. He published his research in ICCV and spent one year in the industry. His research interests include representation learning, geometric deep learning, and Bayesian learning.
\end{IEEEbiography}

\begin{IEEEbiography}[{\includegraphics[width=1in,height=1.25in,clip,keepaspectratio]{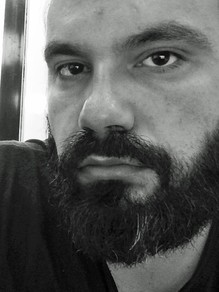}}]{Yannis Panagakis} is a Senior Lecturer (Associate Professor equivalent) in Computer Science at Middlesex University London and a Research Fellow at the Department of Computing, Imperial College London. His research interests lie in machine learning and its interface with signal processing, high-dimensional statistics, and computational optimization. Specifically, Yannis is working on models and algorithms for robust and efficient learning from high-dimensional data and signals representing audio, visual, affective, and social information. He has been awarded the prestigious Marie-Curie Fellowship, among various scholarships and awards for his studies and research. Dr Panagakis currently serves as the Managing Editor of the Image and Vision Computing Journal. He co-organized the BMVC 2017 conference and several workshops and special sessions in top venues such as ICCV. He received his PhD and MSc degrees from the Department of Informatics, Aristotle University of Thessaloniki and his BSc degree in Informatics and Telecommunication from the University of Athens, Greece.

\end{IEEEbiography}


\begin{IEEEbiography}[{\includegraphics[width=1in,height=1.25in,clip,keepaspectratio]{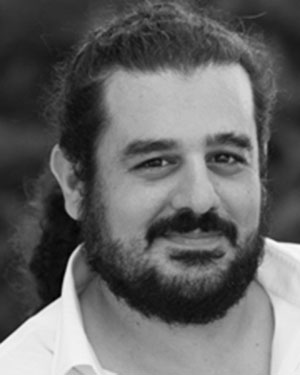}}]{Stefanos Zafeiriou} (M'09) is currently a Reader in Machine Learning and Computer Vision with the Department of Computing, Imperial College London, London, U.K, and a Distinguishing Research Fellow with the University of Oulu under the Finnish Distinguishing Professor Programme. He was a recipient of the Prestigious Junior Research Fellowships from Imperial College London in 2011 to start his own independent research group. He was the recipient of the President’s Medal for Excellence in Research Supervision for 2016. He is recipient of many best paper awards including the best student paper award in FG'2018. In 2018, he received an EPSRC Fellowship. He currently serves as an Associate Editor of the IEEE Transactions on Affective Computing and Computer Vision and Image Understanding journal. In the past he held editorship positions in IEEE Transactions on Cybernetics the Image and Vision Computing Journal. He has been a Guest Editor of over six journal special issues and co-organised over 13 workshops/special sessions on specialised computer vision topics in top venues, such as CVPR/FG/ICCV/ECCV (including three very successfully challenges run in ICCV'13, ICCV'15 and CVPR'17 on facial landmark localisation/tracking). He has co-authored over 60 journal papers mainly on novel statistical machine learning methodologies applied to computer vision problems, such as 2-D/3-D face analysis, deformable object fitting and tracking, shape from shading, and human behaviour analysis, published in the most prestigious journals in his field of research, such as the IEEE T-PAMI, the International Journal of Computer Vision, the IEEE T-IP, the IEEE T-NNLS, the IEEE T-VCG, and the IEEE T-IFS, and many papers in top conferences, such as CVPR, ICCV, ECCV, ICML. His students are frequent recipients of very prestigious and highly competitive fellowships, such as the Google Fellowship x2, the Intel Fellowship, and the Qualcomm Fellowship x3. He has more than 6700 citations to his work, h-index 42. He was the General Chair of BMVC 2017.

\end{IEEEbiography}


\vfill


\cleardoublepage

\appendices
\input{appendices/root}



\end{document}

%% file: notation.tex
\usepackage{amsmath}
\usepackage{amsthm}
\usepackage{bm}
\usepackage{mathtools}
\usepackage{amssymb}
\usepackage{graphicx}
\usepackage{algorithm}
\usepackage{algpseudocode}
\usepackage{varwidth}
\usepackage{tabularx}
\usepackage{subcaption}
\usepackage{siunitx}

\usepackage{mathtools}
\DeclarePairedDelimiter{\ceil}{\lceil}{\rceil}

\newtheoremstyle{break}
  {\topsep}{\topsep}%
  {\itshape}{}%
  {\bfseries}{}%
  {\newline}{}%

\theoremstyle{plain}
\newtheorem{thm}{Theorem}[section] 

\newtheorem{pstion}{Proposition}[section]

\theoremstyle{definition}
\newtheorem{defn}[thm]{Definition} 

\theoremstyle{break}

\newcommand*\mat[1]{\mathbf{#1}}
\newcommand*\vv[1]{\mathbf{#1}}
\newcommand*\ten[1]{\bm{\mathcal{#1}}}

\newcommand{\grad}{\nabla}
\newcommand{\shrink}{\mathcal{S}}
\newcommand{\vvec}{\mathrm{vec}}

\newcommand*\inner[2]{\langle #1, \, #2 \rangle}

\newcommand*\trp[1]{ #1^{\mathsf{T}} }

\newcommand*\sprox[3]{\mathrm{prox}_{#1 #2}(#3)}
\newcommand*\prox[2]{\mathrm{prox}_{#1}(#2)}

\newcommand*\rk[1]{\mathrm{rank}(#1)}
\newcommand*\spn[1]{\mathrm{Span}(#1)}

\def\Image{\textnormal{Im}}

\DeclareMathOperator*{\argmin}{argmin}
\DeclareMathOperator*{\diag}{diag}

\newcommand{\st}{\mathrm{s.t}}

\newcommand{\RR}{\mathbb{R}}
\newcommand{\NN}{\mathbb{N}}

\newcommand{\bSigma}{\bm{\Sigma}}

\newcommand*\Fro[1]{ || #1 ||_{\mathrm{F}} }
\newcommand*\Nuc[1]{ || #1 ||_{*} }
\newcommand*\One[1]{ || #1 ||_{1} }
\newcommand*\Two[1]{ || #1 ||_{2} }
\newcommand*\Twosq[1]{ || #1 ||_{2}^2 }

\newcommand*\Frosq[1]{ || #1 ||_{\mathrm{F}}^2 }

\newcommand{\A}{\mat{A}}
\newcommand{\At}{\trp{\A}}
\newcommand{\B}{\mat{B}}
\newcommand{\Bt}{\trp{\B}}
\newcommand{\C}{\mat{C}}
\newcommand{\D}{\mat{D}}
\newcommand{\R}{\mat{R}}
\newcommand{\G}{\mat{G}}
\newcommand{\mH}{\mat{H}}

\newcommand{\X}{\mat{X}}
\newcommand{\Y}{\mat{Y}}
\newcommand{\Z}{\mat{Z}}
\newcommand{\T}{\mat{T}}
\newcommand{\E}{\mat{E}}
\newcommand{\LL}{\mat{\Lambda}}
\newcommand{\II}{\mat{I}}
\newcommand{\K}{\mat{K}}
\newcommand{\U}{\mat{U}}
\newcommand{\V}{\mat{V}}

\newcommand{\mS}{\mat{S}}

\newcommand{\Vt}{\trp{\mat{V}}}

\newcommand{\Ri}{\mat{R}_i}

\newcommand{\Ei}{\mat{E}_i}

\newcommand{\XXi}{\mat{X}_i}

\newcommand{\Yi}{\mat{Y}_i}

\newcommand{\LLi}{\mat{\Lambda}_i}
\newcommand{\Ki}{\mat{K}_i}

\newcommand{\tX}{\ten{X}}

\newcommand{\tE}{\ten{E}}
\newcommand{\tL}{\ten{L}}

\newcommand{\tG}{\ten{G}}

\newcommand{\tY}{\ten{Y}}

\newcommand{\tK}{\ten{K}}
\newcommand{\tR}{\ten{R}}

\newcommand*\matr[2]{\mathbf{#1}_{[#2]}}

\newcommand\Tstrut{\rule{0pt}{2.6ex}}         
\newcommand\Bstrut{\rule[-0.9ex]{0pt}{0pt}}   

%% file: introduction.tex
\ifCLASSOPTIONcompsoc
\IEEEraisesectionheading{\section{Introduction}\label{sec:introduction}}
\else
\section{Introduction}
\label{sec:introduction}
\fi

\IEEEPARstart{C}{omponent} analysis models and representation learning methods are flexible data-driven alternatives to analytical dictionaries (e.g., Fourier analysis, or wavelets) for signal and data representation. The underlying principle is to solve an optimization problem that encodes the desired properties of the representations and of the bases, and describes what task(s) the representation should help solving. Starting from Principal Component Analysis \cite{pearson_lines_1901,hotelling_analysis_1933}, a rich set of algorithms has been developed for feature extraction, dimensionality reduction, clustering, classification, or denoising - to name but a few. The importance of learned components and representations cannot be overstated, and neither can their effectiveness in dramatically improving machine perception. Prominent examples are Convolutional Neural Networks \cite{y._lecun_backpropagation_1989,y._lecun_gradient-based_1998}, which through hierarchical feature extraction build ad-hoc representations enabling state of the art performance on a wide range of problems \cite{lecun_deep_2015}.

In this work, we propose the Robust Kronecker Component Analysis (RKCA) family of algorithms for the unsupervised learning of compact representations of tensor data. Our method offers to bridge (multilinear) Robust PCA \cite{n._xue_robust_2017,goldfarb_robust_2014} and Sparse Dictionary Learning \cite{r._rubinstein_dictionaries_2010,olshausen_sparse_1997} from the perspective of a robust low-rank tensor factorization. Although our method is generic enough to be presented for arbitrary tensors, we focus on 3-dimensional tensors, and especially those obtained by concatenation of 2-dimensional tensor observations (i.e., matrices). We present a framework for jointly learning a (Kronecker) separable dictionary and sparse representations in the presence of outliers, such that the learned dictionary is low-rank, and the outliers are separated from the data. The double perspective adopted allows us to draw on recent work from the tensor factorization literature to provide some theoretical optimality guarantees, discussed in Section \ref{sec:kdrsdl_degree_three}.

\subsection{Robust PCA and sparse dictionary learning}

Assuming a set of $N$ data samples  $\mathbf{x}_1,\ldots,\mathbf{x}_n \in \mathbb{R}^{m}$ represented as the columns of a matrix $\mathbf{X}$, structured matrix factorizations seek to decompose $\X$ into meaningful components of a given structure, by solving a regularization problem of the form:
\begin{equation}\label{eq:structured_mf_generic}
\min_\Z l(\X, \Z) + \lambda\:g(\Z),
\end{equation}
where $\Z$ is an approximation of the data with respect to a loss $l(\cdot)$, $g(\cdot)$ is a possibly non-smooth regularizer that encourages the desired structure, and $\lambda \geq 0$ is a regularization parameter balancing the two terms. Popular instances of (\ref{eq:structured_mf_generic}) include
Principal Component Analysis (PCA) \cite{pearson_lines_1901,hotelling_analysis_1933} and its variants, e.g., Sparse PCA \cite{zou_sparse_2006}, Robust PCA (RPCA) \cite{candes_robust_2011}, as well as 
sparse dictionary learning \cite{r._rubinstein_dictionaries_2010, j._wright_sparse_2010, olshausen_sparse_1997}.

Concretely, when $\Z$ is taken to be factorized as $\Z = \D\R$ we obtain a range of different models depending on the choice of the regularization and of the properties of $\D$. 
For instance, assuming $\Z= \D\R$ is a low-rank approximation of $\X$ and $\lambda=0$, (\ref{eq:structured_mf_generic}) yields PCA, while by imposing sparsity on $\D$, Sparse PCA \cite{zou_sparse_2006} is obtained.
To handle data corrupted by sparse noise of large magnitude, RPCA \cite{candes_robust_2011} assumes that the observation matrix, $\mathbf{X}$, is the sum of a low-rank matrix $\mathbf{A}$ and of a sparse matrix $\mathbf{E}$ that collects the gross errors, or outliers. This model is actually a special instance of (\ref{eq:structured_mf_generic}) when 
$\mathbf{Z} = \mathbf{A} + \mathbf{E}$, $g(\Z) = ||\mathbf{A}||_* + \lambda ||\mathbf{E}||_1$, and $l(\cdot)$ is the Frobenius norm. Here, $||\cdot||_*$ denotes the low-rank promoting nuclear norm and $||\cdot||_1$ denotes the $\ell_1$ norm that enforces sparsity. Matrix RPCA has been extended to tensors in multiple ways, relying on varying definitions of the tensor rank. We refer to \cite{n._xue_robust_2017,goldfarb_robust_2014} for an overview of CP-based and Tucker-based tensor RPCA models, and to Section \ref{sec:experiments} for specifics on the models compared in this paper.

Assuming $\D$ is over-complete and requiring $\R =[\mathbf{r}_1,\ldots,\mathbf{r}_n]$ to be sparse, (\ref{eq:structured_mf_generic}) leads to sparse dictionary learning by solving the non-convex optimization problem:
\begin{equation}\label{eq:matching_pursuit_problem}
    \min_{\mathbf{R},\mathbf{D}} ||\mathbf{X}-\mathbf{D}\mathbf{R}||_\mathrm{F}^2 + \lambda \sum_{i=1}^n||\mathbf{r}_i||_0,
\end{equation}
where $||\cdot||_\mathrm{F}$ is the  Frobenius norm and $||\cdot||_0$ is the $\ell_0$ pseudo-norm, counting the number of non-zero elements. In K-SVD and its variants, problem (\ref{eq:matching_pursuit_problem}) is solved in an iterative manner that alternates between sparse coding of the data samples on the current dictionary, and a process of updating the dictionary atoms to better fit the data using the Singular Value Decomposition (SVD) \cite{m._aharon_k-svd:_2006,j._mairal_sparse_2008}. 

\subsection{Limits of the classical approach}

While the Robust PCA and the sparse dictionary learning paradigms have been immensely successful in practice, they suffer from limitations that can limit their applicability.

The recovery guarantees of Robust PCA and of similar compressed-sensing approaches have been derived under strong incoherence assumptions, that may not be satisfied even if the data strictly follows the low-rank assumption. In fact, whenever the data lies on a low-dimensional linear subspace but isn't uniformly distributed, the \textit{extra structure} can induce significant coherence. \cite{liu_low-rank_2016} studies this problem in details, and in \cite{liu_low-rank_2016, liu_blessing_2017}, the authors show the adverse effects of coherence can be mitigated in existing low-rank modeling methods, such as \cite{g._liu_robust_2013}, by choosing the representation dictionary to be low-rank.


Modern representation learning methods have to deal with increasingly large amounts of increasingly high-dimensional data. Classical low-rank modeling and dictionary learning models tend to be too expensive for that setting: typically, algorithms of the K-SVD family suffer from a high computational burden, preventing their applicability to high-dimensional and large scale data.

To overcome the issue of scalability in dictionary learning, a separable structure on the dictionary can be enforced. For instance, Separable Dictionary Learning (SeDiL) \cite{hawe_separable_2013} considers a set of samples in matrix form, namely, $\mathcal{X} = (\XXi)_i$, admitting sparse representations on a pair of bases
$\A, \B$, of which the Kronecker product constructs the dictionary. The corresponding objective is:
\begin{equation}\label{eq:sedil}
\min_{\A, \B, \tR} \frac{1}{2} \sum_i \Frosq{\XXi - \A \Ri \Bt} + \lambda g(\tR) + \kappa r(\A) + \kappa r(\B),
\end{equation}
where the regularizers $g(\cdot)$ and $r(\cdot)$ promote sparsity in the representations, and low mutual-coherence of the dictionary $\D = \B \otimes \A$, respectively. Here, $\D$ is constrained to have orthogonal columns, i.e., the pair $\A, \B$ shall lie on the product manifold of two product of sphere manifolds. A different approach is taken in \cite{s._hsieh_2d_2014}: a separable 2D dictionary is learned in a two-step strategy similar to that of K-SVD. Each matrix observation $\XXi$ is represented as $\A \Ri \Bt$. In the first step, the sparse representations $\Ri$ are found by 2D Orthogonal Matching Pursuit (OMP) \cite{fang_2d_2012}. In the second step, a CANDECOMP/PARAFAC (CP) \cite{hitchcock_expression_1927,harshman_foundations_1970,carroll_analysis_1970} decomposition is performed on a tensor of residuals via Regularized Alternating Least Squares to solve $\min_{\A, \B, \tR} \Fro{\tX - \tR \times_1 \A \times_2 \B}$\footnote{cf. Definition \ref{def:mode_n_prod} for the product $\times_n$}. However, those models learn over-complete dictionaries on a multitude of small patches extracted from the images. For large images or large number of images, the number of patches can easily become prohibitively large, undermining the scalability benefits of learning separable dictionaries. Moreover, none of these models is robust to gross corruption in the dataset.

Beyond scalability, separable dictionaries have been shown to be theoretically appealing when dealing with tensor-valued observations. According to \cite{shakeri_identification_2017}, the necessary number of samples for accurate (up to a given error) reconstruction of a Kronecker-structured dictionary within a local neighborhood scales with the sum of the product of the dimensions of the constituting dictionaries when learning on tensor data (see \cite{shakeri_minimax_2016} for 2-dimensional data, and \cite{shakeri_minimax_2018, shakeri_sample_2017} for $N$-order tensor data), compared to the product for vectorized observations. This suggests better performance is achievable compared to classical methods on tensor observations.

\subsection{Outline and contributions}

Here, we propose novel methods for separable dictionary learning based on robust tensor factorisations that learn simultaneously the dictionary and the sparse representations. We do not seek overcompleteness, but rather promote low-rankness in a pair of dictionaries, and sparsity in the codes to learn a low-rank representation of the input tensor. In this regard, our methods combine ideas from both Sparse Dictionary Learning and Robust PCA, as well as tensor factorizations. Our solvers are based on the Alternating Direction of Multipliers Method (ADMM) \cite{boyd_distributed_2011}.

A preliminary version of this work has been presented in \cite{m._bahri_robust_2017}. This paper offers the following novelties:
\begin{itemize}
\item We generalize the results of \cite{m._bahri_robust_2017} and propose new regularizers that yield stronger optimality properties, we call the resulting framework Robust Kronecker Component Analysis (RKCA).
\item Unlike previous models that rely solely on regularization to impose low-rankness, our work presents surrogates of specific definitions of the tensor rank that impose low-rankness through both regularisation and structure. Specifically, our method can be seen as modeling a low-rank tensor by tensor sparse coding with additional regularization on the factors.
\item We show that RKCA with these well-chosen regularizers can be reformulated in the framework of \cite{haeffele_structured_2014,haeffele_global_2015} allowing us to provide global optimality guarantees. It is worth mentioning that our proof applies to any Tucker factorization problem.
\item We demonstrate that RKCA can perform tensor completion in the presence of gross corruption.
\item We derive a Linearized ADMM (LADMM) algorithm for RKCA to improve scalability.
\item The experimental evaluation has been enriched to include the Low-Rank Representation method \cite{g._liu_robust_2013}, as well as comparisons with the recent Deep Image Prior \cite{ulyanov_deep_2018}.
\item Finally, we offer new perspectives on the low-rank promoting properties of RKCA.
\end{itemize}

To the best of our knowledge, our method is the first to leverage dictionaries that are both low-rank and separable for robust representation learning.

The rest of the manuscript is organized as follows. Section \ref{sec:preliminaries} reviews fundamental definitions and results of tensor algebra and of the Kronecker product. Section \ref{sec:model_derivation} is dedicated to deriving the RKCA model, relating RKCA to separable dictionary learning, and deriving RKCA with missing values. In Section \ref{sec:kdrsdl_degree_three}, we discuss optimality guarantees by formulating RKCA as an equivalent CP factorization with duplicated factors. Perspectives on the low-rank promoting properties can be found in Section \ref{sec:low-rankness}, and a discussion on the computational cost and implementation details of the methods in Section \ref{seq:implem_details}. Finally, we present in Section \ref{sec:experiments} experimental evidence of the effectiveness of RKCA on synthetic and real-world data.

%% file: preliminaries.tex
\section{Preliminaries}
\label{sec:preliminaries}

In this section we review fundamental properties of tensor algebra and of the Kronecker product, and present the notations and conventions followed in the paper.

\subsection{Tensor algebra}

We refer to multidimensional arrays of real numbers of dimension $I_1 \times I_2 \times \ldots \times I_N$ as $N$-dimensional, or $N$-way \textit{real tensors}. The \textit{order} of the tensor is the number of indices required to address its elements. Consequently, each element of an $N^{th}$-order
tensor $\bm{\mathcal{X}}$ is addressed by $N$ indices, i.e.,
$(\bm{\mathcal{X}})_{i_{1}, i_{2}, \ldots, i_{N}} \doteq x_{i_{1}, i_{2}, \ldots, i_{N}}$. We denote tensors by bold calligraphic letters, e.g., $\tX$.

The sets of real and integer numbers are denoted by $\mathbb{R}$ and $\mathbb{Z}$, respectively.    An $N^{th}$-order  real-valued tensor $\bm{\mathcal{X}}$ is  defined over the
tensor space $\mathbb{R}^{I_{1} \times I_{2} \times \cdots \times
I_{N}}$, where $I_{n} \in \mathbb{Z}$ for $n=1,2,\ldots,N$. 

Matrices (vectors) are second order (first order) tensors and are denoted by uppercase (lowercase) bold letters, e.g., $\X$ ($\mathbf{x}$). The $i^{th}$ \textit{column} of a matrix $\X$ will be written $\mathbf{x}_i$ for convenience.

\subsubsection{Elementary definitions}
Tensor fibers and tensor slices are special subsets of the elements of a tensor.

\begin{figure}
    \centering
    \includegraphics[width=.25\columnwidth]{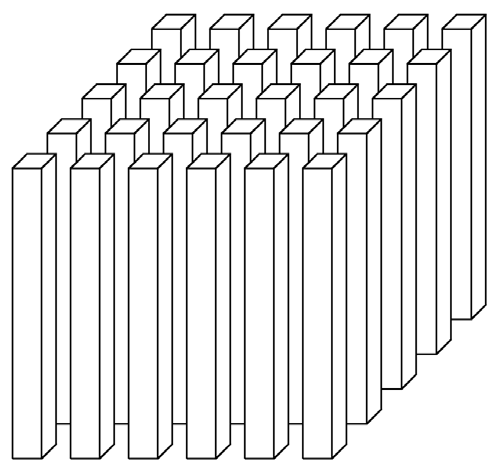}
    \hspace{3em}
    \includegraphics[width=.25\columnwidth]{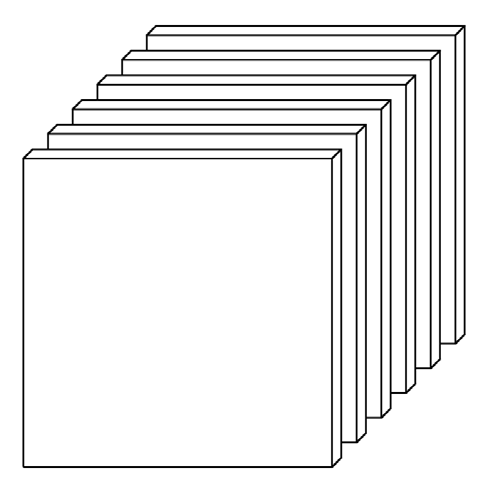}
    \caption{Mode-1 fibers (tubes) and mode-3 slices of a 3-way tensor. Adapted from \cite{kolda_tensor_2009}.}
    \label{fig:slices_fibres}
\end{figure}

\begin{defn}[Tensor slices and fibers]
We define the mode-$n$ tensor fibers as all subsets of the tensor elements obtained by fixing all but the $n^{th}$ index to given values.
Similarly, tensor slices are obtained by fixing all but two indices.
\end{defn}
Figure \ref{fig:slices_fibres} illustrates the concepts of fibers and slices for a third order tensor. For a third order tensor, the slices are effectively matrices obtained by "slicing" the tensor along a mode.


A tensor's elements can be re-arranged to form a new tensor of different dimension. We define the tensor \textit{vectorisation} and tensor \textit{matricisations} (or \textit{unfoldings}) as re-arrangements of particular interest.

\begin{defn}[Tensor vectorisation]
Let $\ten{X}\in\RR^{I_1\times I_2\times\ldots\times I_N}$ be an $N^\text{th}$-order tensor. 
The \textit{tensor vectorisation, or flattening, operator} $\tX \mapsto \vvec(\tX)$ maps each element $x_{i_1, i_2, \ldots, i_N}$ of $\tX$ to a unique element $x_j$ of a real-valued vector $\vvec(\tX)$ of dimension $\prod_{m=1}^{N}{I_m}$ with the following bijection:
\begin{equation}
j = 1 + \sum_{k=1}^{N}{\left(i_k - 1\right)J_k}\quad\text{and}\quad
J_k = \prod_{m=1}^{k-1}{I_m}.
\end{equation}
\end{defn}

\begin{defn}[Mode-$n$ unfolding and tensorisation]
Let $\ten{X}\in\RR^{I_1\times I_2\times\ldots\times I_N}$ be an $N^\text{th}$-order tensor. The mode-$n$ tensor \textit{matricisation} - or \textit{unfolding}, with $n\in\{1,2,\ldots,N\}$, is the matrix $\matr{X}{n}$ of dimensions $(I_n, \prod_{k\neq n}{I_k})$ such that tensor entry $x_{i_1, i_2, \ldots, i_N}$ is mapped to a unique element $x_{i_n, j}$, with the following bijection:
\begin{equation}\label{matricisation_bijection}
j = 1 + \sum_{\substack{k=1\\k\neq n}}^{N}{\left(i_k - 1\right)J_k}\quad\text{and}\quad
J_k = \prod_{\substack{m=1\\m\neq n}}^{k-1}{I_m}.
\end{equation}

The inverse operator of the mode-$n$ unfolding is the mode-$n$ tensorisation and is denoted $\mathrm{fold}_n(\matr{X}{n}) = \tX$.
\end{defn}

In other words, the mode-$n$ unfolding is the matrix whose columns are the mode-$n$ fibers obtained by first varying $I_1$, then $I_2$, up until $I_N$ with the exception of $I_n$.


Having defined both tensor fibers, slices, and unfoldings; we can now define a higher-order analogue for matrix-vector multiplication as the \textit{tensor mode-$n$ product}.

\begin{defn}[Mode-$n$ product]\label{def:mode_n_prod}
The mode-$n$ product of $\tX \in \RR^{I_1 \times I_2 \times \cdots \times I_N}$ with $\U \in \RR^{J \times I_n}$ is the tensor $\tX \times_n \U \in \RR^{I_1 \times \cdots \times I_{n-1} \times J \times I_{n+1} \times \cdots \times I_N}$ such that
\begin{equation}
    (\tX \times_n \U)_{i_1 \cdots i_{n-1} j i_{n+1} i_N} = \sum_{ i_{n=1} }^{I_n} x_{i_1 i_2 \cdots i_N}u_{j{i_n}}.
\end{equation}
Effectively, each mode-$n$ fiber is multiplied by $\U$.
\end{defn}


The mode-$n$ product can be expressed as a matrix product as per Proposition \ref{prop:matri_mode_n}.

\begin{pstion}
\label{prop:matri_mode_n}
With the notations of Definition \ref{def:mode_n_prod}:
\begin{equation}
(\tX \times_{i=1}^N \U_i)_{[n]} = \U_n \X_{[n]}\trp{(\otimes_{i=N, i \neq n}^1 \U_i) },
\end{equation}
\end{pstion}
where $\otimes$ is the Kronecker product of Definition \ref{def:kronprod}. 

\subsubsection{Tensor rank}

Contrary to matrices, the tensor rank is not uniquely defined. In this paper, we are interested in the tensor \textit{Tucker rank}, which we define as the vector of the ranks of a tensor's mode-$n$ unfoldings (i.e., its mode-$n$ ranks), and the tensor \textit{multi-rank} \cite{z._zhang_novel_2014,c._lu_tensor_2016} defined as the vector of the ranks of its frontal slices. 

More details about tensors, such as the definitions of the tensor CP-rank and information about common decompositions can be found in \cite{kolda_tensor_2009}, for example.

\subsection{Properties of the Kronecker product}

We now remind the reader of the definition and of important properties of the Kronecker product.

\begin{defn}[Kronecker product]\label{def:kronprod}
Let $\A \in \RR^{m \times n}$ and $\B \in \RR^{p \times q}$ then
\begin{equation}\label{eq:kron}
    \A \otimes \B = \begin{bmatrix}
    a_{11}\B & a_{12}\B & \ldots & a_{1n}\B\\
    \vdots & \vdots & \vdots & \vdots\\
    a_{n1}\B & a_{n2}\B & \ldots & a_{nn}\B\\
    \end{bmatrix}.
\end{equation}

\end{defn}

Proposition \ref{prop:vec_kronecker} is a fundamental result relating the Kronecker product and tensor vectorisation.
\begin{pstion}\label{prop:vec_kronecker}
Let $\tY = \tX \times_1 \U_1 \times_2 \U_2 \ldots \times_N \U_N$ then
\begin{equation}
\vvec (\tY) = \left( \otimes_{i=N}^1 \U_i \right) \vvec (\tX).
\end{equation}
In particular, if $\Y = \A \X \Bt$ then 
\begin{equation}
\vvec (\Y) = (\B \otimes \A) \vvec (\X).
\end{equation}
\end{pstion}

The Kronecker product is \textit{compatible} with most common matrix decomposition. For the Singular Value Decomposition, we have Proposition \ref{prop:compat_kron_svd}.
\begin{pstion}\label{prop:compat_kron_svd}
Let $\Y = \U \mS \Vt$ and $\X = \mH \T \trp{\G}$ then
\begin{equation}
    \Y \otimes \X = (\U \otimes \mH)(\mS \otimes \T)\trp{(\V \otimes \G)}.
\end{equation}
Where we used the identity $\trp{(\A \otimes \B)} = \trp{\A} \otimes \trp{\B}$.
\end{pstion}

We refer the reader to one of the many linear algebra references for further information about the Kronecker product and its properties.

%% file: model_derivation.tex
\section{Model derivation}
\label{sec:model_derivation}

In this section we start by describing RKCA as a structured tensor factorization. We then show its equivalence to a dictionary learning problem, discuss how to extend the model to handle missing values, and finally derive a first algorithm that will serve as the basis for the rest of the discussion.

\subsection{The tensor factorization perspective}

Consider a set of $N$ two dimensional observations (i.e., matrices) $\X_i \in \RR^{m \times n}, \, i = 1, \ldots, N$ stacked as the frontal slices of a tensor $\tX$. We study \textit{Tensor Robust PCA} problems:

\begin{align}\label{eq:constrained_pb_general}
    \begin{matrix*}[l]
    \min_{\tL, \tE} & f(\tL) + g(\tE),\\
    \st & \tX = \tL + \tE
    \end{matrix*}
\end{align}
where $\tL$ and $\tE$ are respectively the low-rank and sparse components of $\tX$. We will define $f(\cdot)$ and $g(\cdot)$ to be regularization functions, possibly non-smooth and non-convex, meant to promote structural properties of $\tL$ and $\tE$ - in our case, low-rankness of $\tL$, and sparsity in $\tE$.

The Robust Kronecker Component Analysis (RKCA) is obtained by assuming $\tL$ factorizes in a restricted form of Tucker factorization, and defining $f(\cdot)$ as a combination of penalties on the factors. More specifically, we assume:
\begin{equation}
\tL = \tR \times_1 \A \times_2 \B.
\end{equation}
Figure \ref{fig:decomposition} illustrates the decomposition.

\begin{figure}[h]
\centering
\includegraphics[width=.9\columnwidth]{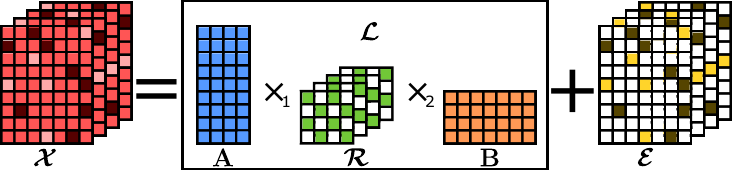}
\caption{Illustration of the decomposition.}
\label{fig:decomposition}
\end{figure}

We construct $f$ to promote low-rankness in the first two modes of $\tL$ and therefore in each of its frontal slices. In this work, we will discuss three different choices of $f$ depending on the \textit{positive-homogeneity} degree of the full regularization. 

\begin{defn}[{Def. 3 of \cite[page 6]{haeffele_global_2015}}]
A function $\theta : \RR^{D^1} \times \ldots \times \RR^{D^N} \rightarrow \RR^D$ is \textbf{positively homogeneous with degree p} if $\theta(\alpha x^1,\ldots,\alpha x^N) = \alpha^p \theta(x^1,\ldots,x^N), \  \forall \alpha \geq 0$. Notably $\theta(0) = 0$ holds.
\end{defn}

The following choice of $f$ is the \textit{degree 2 regularizer} and recovers the model presented in \cite{m._bahri_robust_2017}:
\begin{equation}\label{eq:degree_two_f}
f(\tL) = \alpha \One{\tR} + \Fro{\B \otimes \A},
\end{equation}
while choice
\begin{equation}\label{eq:degree_three_f_fro}
f(\tL) = \alpha \One{\tR}\Fro{\B \otimes \A},
\end{equation}
and
\begin{equation}\label{eq:degree_three_f_nuc}
f(\tL) = \alpha \One{\tR}\Nuc{\B \otimes \A},
\end{equation}
are \textit{degree 3 regularizers}.

In all three cases, the regularizer comprises a norm on the Kronecker product of the Tucker factors, hence the proposed method is coined as \textit{RKCA}. The interpretation of this is made clear in Section \ref{sec:dl_interpretation}.

\subsection{Robust separable dictionary learning}
\label{sec:dl_interpretation}

We now establish how to interpret RKCA as a robust separable dictionary learning scheme for tensor-valued observations. We remind the reader that we do not seek overcompleteness, but rather choose to learn low-rank dictionaries as a mean of modeling low-rankness in the observations. We impose separability to preserve the spatial correlation within the tensor inputs, and for scalability (c.f., Section \ref{sec:kdrsdl}). These choices are further motivated by the work of \cite{liu_low-rank_2016} (for low-rankness), and \cite{shakeri_identification_2017} (for separability).

Consider the following Sparse Dictionary Learning problem with Frobenius-norm regularization on the dictionary $\D$, where we decompose $N$ observations $\vv{x}_i \in \RR^{mn}$ on $\D \in \RR^{mn \times r_1 r_2}$ with representations $\vv{r}_i \in \RR^{r_1 r_2}$:
\begin{align}\label{eq:sparse_dictionary_learning}
    \min_{\D, \mat{R}}{ \sum_i \Twosq{\vv{x}_i - \D \vv{r}_i } } + \lambda \sum_i \One{\vv{r}_i} + \Fro{\D}.
\end{align}
We assume a Kronecker-decomposable dictionary $\D = \B \otimes \A$ with $\A \in \RR^{m \times r_1}, \B \in \RR^{n \times r_2}$. To model the presence of outliers, we introduce a set of vectors $\vv{e}_i \in \RR^{mn}$ and, with $d = r_1 r_2 + mn$, define the block vectors and matrices:

\begin{equation}
\vv{y}_i = \begin{bmatrix}\vv{r}_i\\ \vv{e}_i\end{bmatrix}  \in \RR^{d} \quad \C = \begin{bmatrix} \B \otimes \A & \II \end{bmatrix} \in \RR^{mn \times d}.
\end{equation}

We obtain a two-level structured dictionary $\mat{C}$ and the associated sparse encodings $\vv{y}_i$. Breaking-down the variables to reduce dimensionality and discarding $\Fro{\II}$:

\begin{equation}
\begin{split}\label{eq:two_level_structured_sparse_dictionary_learning_explicit}
\min_{\A, \B, \mat{R}, \mat{E}} \sum_i \Twosq{\vv{x}_i - (\B \otimes \A) \vv{r}_i - \vv{e}_i } \\ + \lambda \sum_i \One{\vv{r}_i} + \lambda \sum_i \One{\vv{e}_i} + \Fro{\B \otimes \A}.
\end{split}
\end{equation}
Suppose now that the observations $\vv{x}_i$ were obtained by vectorizing two-dimensional data such as images, i.e., $\vv{x}_i = \vvec({\XXi}), \XXi \in \RR^{m \times n}$. Without loss of generality, we choose $r_1 = r_2 = r$ and $\vv{r}_i = \vvec(\Ri), \Ri \in \RR^{r \times r}$, and recast the problem as:
\begin{equation}\label{eq:matrix_2d_rpca}
\begin{split}
\min_{\A, \B, \ten{R}, \ten{E}} \sum_i \Frosq{\XXi - \A \Ri \Bt - \Ei }  \\ 
    + \lambda \sum_i \One{\Ri} + \lambda \sum_i \One{\Ei} + \Fro{\B \otimes \A}.
\end{split}
\end{equation}
Equivalently, enforcing the equality constraints and concatenating the matrices $\XXi, \Ri$, and $\Ei$ as the frontal slices of 3-way tensors, we obtain problem (\ref{eq:constrained_pb}).

\begin{align}\label{eq:constrained_pb}
    \begin{matrix*}[l]
    \min_{\A, \B, \ten{R}, \ten{E}} & \alpha \One{\ten{R}} + \lambda \One{\tE} + \Fro{\B \otimes \A},\\
    \st & \tX = \ten{R} \times_1 \A \times_2 \B + \tE
    \end{matrix*}
\end{align}
where $\tR$ is a tensor whose $i^{th}$ frontal slice encodes the representation of $\X_i$ on the column basis $\A$ and on the row basis $\B$. Problem (\ref{eq:constrained_pb}) corresponds to choosing $f$ to be (\ref{eq:degree_two_f}) and $g$ to be the element-wise $\ell_1$ norm in (\ref{eq:constrained_pb_general}).

We verify $r \in \NN, r \leq \min(m, n)$ is a natural \textit{upper bound} on the rank of each frontal slice of $\tL$ and on its mode-1 and mode-2 ranks:
\begin{equation}\label{eq:frontal_slice}
\mat{L}_i = \A \R_i \Bt.
\end{equation}

We have thus shown how the two perspectives are equivalent.

\subsection{Tensor completion with RKCA}
\label{sec:tensor_comp}

Handling outliers is important for real-world applications, but data corruption is polymorphic, and missing values are common. We now present an extension to the original problem where we assume we are only given incomplete and possibly grossly-corrupted observations, and wish to perform simultaneous removal of the outliers and imputation of the missing values.

Given a set $\Omega \subset \NN^{I_1 \times I_2 \times \ldots I_N}$ we define the sampling operator $\pi_\Omega : \RR^{I_1 \times I_2 \times \ldots I_N} \rightarrow \RR^{I_1 \times I_2 \times \ldots I_N}$ as the projection on the space of $N$-way tensors whose only non-zero entries are indexed by $N$-tuples in $\Omega$, i.e., $\forall \tX \in \RR^{I_1 \times I_2 \times \ldots I_N}$:
\begin{equation*}
\pi_\Omega(x_{i_1, i_2, \ldots, i_N}) = \begin{cases} 
      1 & (i_1, i_2, \ldots, i_N) \in \Omega \\
      0 & (i_1, i_2, \ldots, i_N) \notin \Omega
   \end{cases}.
\end{equation*}

It is simple to show $\pi_\Omega$ is an orthogonal projection. Clearly, $\pi_\Omega$ is linear and idempotent ($\pi_\Omega \circ \pi_\Omega = \pi_\Omega$). A tensor $\tX$ is in the null space of $\pi_\Omega$ if and only if none of its non-zero entries are indexed by $N$-tuples of $\Omega$. Denoting by $\bar{\Omega}$ the complement of $\Omega$, we have $\forall \tX, \langle \pi_\Omega(\tX), \pi_{\bar{\Omega}}(\tX) \rangle = 0$.

With partial observations, we solve:
\begin{align}\label{eq:constrained_pb_mv_preliminary}
    \begin{matrix*}[l]
    \min_{\tL, \tE} & f(\tL) + \One{\tE},\\
    \st & \pi_\Omega(\tX) = \pi_\Omega(\tL) + \pi_\Omega(\tE).\\
    \end{matrix*}
\end{align}

By the orthogonality of $\pi_\Omega$, $\One{\tE} = \One{\pi_\Omega(\tE)} + \One{\pi_{\bar{\Omega}}(\tE)}$ . Without loss of generality we follow \cite{goldfarb_robust_2014,shang_robust_2014} and assume $\pi_{\bar{\Omega}}(\tX) = 0$ such that $\pi_{\bar{\Omega}}(\tX) = \pi_{\bar{\Omega}}(\tL) + \pi_{\bar{\Omega}}(\tE)$. This implies that we do not seek to recover possible corruption on the missing values but directly the missing element. Problem (\ref{eq:constrained_pb_mv_preliminary}) is therefore equivalent to:
\begin{align}\label{eq:constrained_pb_mv}
    \begin{matrix*}[l]
    \min_{\tL, \tE} & f(\tL) + \One{\pi_\Omega(\tE)},\\
    \st & \tX = \tL + \tE.\\
    \end{matrix*}
\end{align}



To solve (\ref{eq:constrained_pb_mv}) we need to compute the proximal operator of $\One{\pi_\Omega(\tE)}$, we show (proof in Appendix \ref{appendix:selective_shrinkage}) that the corresponding operator is the selective shrinkage $\shrink^\Omega_\lambda$:
\begin{equation}\label{eq:select_shrink}
\forall \tX \in \RR^{I_1 \times I_2 \times \ldots I_N}, \shrink^\Omega_\lambda(\tX) = \shrink_\lambda(\pi_\Omega(\tX)) + \pi_{\bar{\Omega}}(\tX).
\end{equation}











We present tensor completion results in Section \ref{sec:tensor_completion_experiments}.

\subsection{An efficient algorithm for the degree 2 regularizer}
\label{sec:kdrsdl}

Finally, we derive an efficient algorithm for the degree 2-regularized problem. This discussion will serve as a basis for the other regularizers which only require minor modifications of this algorithm.

Problem (\ref{eq:constrained_pb}) is not jointly convex, but is convex in each component individually. We resort to an alternating-direction method and propose a non-convex ADMM procedure that operates on the frontal slices.


Minimizing $\Fro{\B \otimes \A}$ presents a challenge: the product is high-dimensional, the two bases are coupled, and the loss is non-smooth. Let 
$||.||_p$ denote the Schatten-$p$ norm\footnote{The Schatten-$p$ norm of $\A$ is the $\ell_p$ norm of its singular values.}. Using the identity $||\A \otimes \B||_p = ||\A||_p ||\B||_p$ \footnote{From the compatibility of the Kronecker product with the SVD.} and remarking that $\Fro{\B} \Fro{\A} \leq \frac{\Frosq{\A} + \Frosq{\B}}{2}$, we minimize a simpler upper bound\footnote{$\Nuc{\A\Bt}$ \cite{recht_guaranteed_2010}.}. The resulting sub-problems are smaller, and therefore easier to solve computationally. In order to obtain exact proximal steps for the encodings $\Ri$, we introduce a split variable $\Ki$ such that $\forall i, \; \Ki = \Ri$. Thus, we solve:
\begin{align}\label{eq:constrained_pb_rpca_upper_bound}
    \begin{matrix*}[l]
    \displaystyle \min_{\A, \B, \ten{R}, \tK, \ten{E}} & \alpha \One{\ten{R}} + \lambda \One{\tE} + \frac{1}{2}(\Frosq{\A} + \Frosq{\B}),\\
    \st & \tX = \ten{K} \times_1 \A \times_2 \B + \tE,\\
    \st & \tR = \tK.
    \end{matrix*}
\end{align}

By introducing the  Lagrange multipliers $\LL$ and $\tY$, such that the $i^{th}$ frontal slice corresponds to the $i^{th}$ constraint, and the dual step sizes $\mu$ and $\mu_{\tK}$, we formulate the Augmented Lagrangian of problem (\ref{eq:constrained_pb_rpca_upper_bound}):
\begin{align}\label{eq:augmented_lagrangian_pb_split}
    & L(\A, \B, \tR, \tE, \tK, \LL, \tY, \mu, \mu_{\tK}) =  \lambda \sum_i \One{\Ri} + \lambda \sum_i \One{\Ei} +\notag \\ 
    & \frac{1}{2}(\Frosq{\A} + \Frosq{\B}) + \sum_i \inner{\LLi}{\XXi - \A \Ki \Bt - \Ei} + \notag\\
    & \sum_i \inner{\Yi}{\Ri - \Ki} + \frac{\mu}{2} \sum_i \Frosq{\XXi - \A \Ki \Bt - \Ei } + \notag\\
    & \frac{\mu_{\tK}}{2} \sum_i \Frosq{\Ri - \Ki}.
\end{align}
We can now derive the ADMM updates. Each $\Ei$ is given by shrinkage after rescaling \cite{candes_robust_2011}:
\begin{equation}\label{eq:upd_E}
\Ei = \shrink_{\lambda/\mu} (\XXi - \A \Ki \Bt + \frac{1}{\mu} \LLi).
\end{equation}
From the developments of Section \ref{sec:tensor_comp}, extending our algorithms to handle missing value only involves using the selective shrinkage operator (\ref{eq:select_shrink}) in (\ref{eq:upd_E}).

A similar rule is immediate to derive for $\Ri$, and solving for $\A$ and $\B$ is straightforward with some matrix algebra. We therefore focus on the computation of the split variable $\Ki$. Differentiating, we find $\Ki$ satisfies:
\begin{align}\label{eq_updt_K_opti}
    & \mu_{\tK} \Ki + \mu \At\A \Ki \Bt\B =\\
    & \quad \At ( \LLi + \mu ( \XXi - \Ei ) ) \B + \mu_{\tK} \Ri + \Yi \nonumber.
\end{align}
The key is here to recognize Equation (\ref{eq_updt_K_opti}) is a \textit{Stein} equation, and can be solved in cubical time and quadratic space in $r$ by solvers for discrete-time Sylvester equations - such as the Hessenberg-Schur method \cite{g._golub_hessenberg-schur_1979} - instead of the naive $O(r^6)$ time, $O(r^4)$ space solution of vectorizing the equation in a size-$r^2$ linear system. We obtain Algorithm \ref{alg:RKCA}.


\begin{algorithm}[ht]
\footnotesize
\begin{algorithmic}[1]
\Procedure{RKCA}{$\tX,r, \lambda, \alpha$}

\State $\A^0, \B^0, \tE^0, \ten{R}^0, \tK^0, \LL^0, \tY^0, \mu^0, \mu_{\tK}^0 \gets$ \Call{Initialize}{$\tX$}

\While{not converged}

\State $\tE^{t+1} \gets \shrink_{\lambda / \mu^t} ( \tX - \tK^{t} \times_1 \mat{A}^{t} \times_2 \mat{B}^{t} + \frac{1}{\mu^t} \LL^t)$

\State $\tilde{\tX}^{t+1} \gets \tX - \tE^{t+1}$

\State \begin{varwidth}[t]{\linewidth} $\mat{A}^{t+1}\gets (\sum_i (\mu^t \tilde{\X}_i^{t+1} + \LLi^{t}) \mat{B}^{t} \trp{(\Ki^t)}) \; /$ \\ \hspace*{2em} $(\II + \mu^t \sum_i \Ki^t \trp{(\mat{B}^{t})} \mat{B}^{t} \trp{(\Ki^t)})$ \end{varwidth}

\State \begin{varwidth}[t]{\linewidth}  $\mat{B}^{t+1}\gets (\sum_i \trp{(\mu^t \tilde{\X}_i^{t+1} + \LLi^t)} \mat{A}^{t+1} \Ki^t) \; /$ \\ \hspace*{2em} $(\II + \mu^t \sum_i \trp{(\Ki^t)} \trp{(\mat{A}^{t+1})} \mat{A}^{t+1} \Ki^t)$ \end{varwidth}

\ForAll{i}
    \State
        \begin{varwidth}[t]{\linewidth}
            $\mat{K}_i^{t+1} \gets$ \textsc{Stein}($-\frac{\mu^t}{\mu_{\tK}^t} \trp{(\mat{A}^{t+1})}\mat{A}^{t+1}$,
                \; $\trp{(\mat{B}^{t+1})} \mat{B}^{t+1}$, \hspace*{2em} $\frac{1}{\mu_{\tK}^t} \left[ \trp{(\mat{A}^{t+1})} ( \mat{\Lambda}_i^{t} + \mu^t \tilde{\mat{X}}_i^{t+1} ) \mat{B}^{t+1} + \mat{Y}_i^t \right] + \Ri^t $)
        \end{varwidth}
    \State $\mat{R}_i^{t+1} \gets \shrink_{\alpha / \mu_{\tK}^t} (\mat{K}_i^{t+1} - \frac{1}{\mu_{\tK}^t} \mat{Y}_i^t)$
\EndFor

\State $\LL^{t+1} \gets \LL^{t} + \mu^t (\tilde{\tX}^{t+1} - \tK^{t+1} \times_1 \mat{A}^{t+1} \times_2 \mat{B}^{t+1})$
\State $\tY^{t+1} \gets \tY^{t} + \mu_{\tK}^t (\tR^{t+1} - \tK^{t+1})$
\State $\mu^{t+1} \gets \min(\mu^*,  \rho \mu^t)$
\State $\mu_{\tK}^{t+1} \gets \min(\mu_{\tK}^*, \rho \mu_{\tK}^t)$

\EndWhile

\State \textbf{return} $\A, \B, \ten{R}, \tE$ 
\EndProcedure
\end{algorithmic}
\caption{RKCA with degree 2 regularization.}
\label{alg:RKCA}
\end{algorithm}


%% file: degree_3_prov.tex
\section{RKCA and global optimality}
\label{sec:kdrsdl_degree_three}

The work of \cite{haeffele_structured_2014, haeffele_global_2015} suggests global optimality can be achieved from any initialization in tensor factorization models given that the factorization and regularization mappings match in certain ways. We summarize the main results of this work and show our model with regularizer (\ref{eq:degree_three_f_fro}) or (\ref{eq:degree_three_f_nuc}) respects the conditions for global optimality.

\subsection{Review of the main results}



We first review some of the concepts manipulated in \cite{haeffele_structured_2014, haeffele_global_2015} and give an overview of the main results.

\begin{defn}[{Def. 1 of \cite[page 6]{haeffele_global_2015}}]
A \textbf{size-r set of K factors} $(\tX^1,\ldots,\tX^K)_r$ is defined to be a set of $K$ tensors where the final dimension of each tensor is equal to $r$: $(\tX^1,\ldots,\tX^K)_r \in \RR^{(D^1 \times r)} \times \ldots \times \RR^{(D^K \times r)}$.
\end{defn}

\begin{defn}[{Def. 6 of \cite[page 6]{haeffele_global_2015}}]
An \textbf{elemental mapping}, $\phi : \RR^{D^1} \times \ldots \times \RR^{D^K} \rightarrow \RR^D$ is any mapping which is positively homogeneous with degree $p \neq 0$.  The \textbf{r-element factorization mapping} $\Phi_r : \RR^{ (D^1 \times r) } \times \ldots \times \RR^{(D^K \times r) } \rightarrow \RR^D$ is defined as:
\begin{equation}
\label{eq:Phi_r_def}
\Phi_r(\tX^1,\ldots,\tX^K) = \sum_{i=1}^r \phi(\tX^1_i,\ldots,\tX^K_i).
\end{equation}
\end{defn}

\begin{defn}[{Def. 7 of \cite[page 8]{haeffele_global_2015}}]
\label{def:elemental_reg}
An \textbf{elemental regularization function} $g : \RR^{D^1} \times \ldots \times \RR^{D^K} \rightarrow \RR_+ \cup \infty$, is defined to be any function which is positive semidefinite and positively homogeneous.
\end{defn}

\begin{defn}[{Def. 8 of \cite[page 9]{haeffele_global_2015}}]
\label{def:nondegenerate}
Given an elemental mapping $\phi$ and an elemental regularization function $g$, the authors define $(\phi,g)$ to be a \textbf{nondegenerate pair} 
if 1) $g$ and $\phi$ are both positively homogeneous with degree $p$, for some $ p\neq 0$ and 2) $\forall X \in \Image(\phi) \backslash 0, \ \ \exists \mu \in (0,\infty]$ and $(\bm{\tilde{z}}^1,\ldots,\bm{\tilde{z}}^K)$ such that $\phi(\bm{\tilde{z}}^1,\ldots,\bm{\tilde{z}}^K) = \tX$, $g(\bm{\tilde{z}}^1,\ldots,\bm{\tilde{z}}^K)=\mu$, and $g(\bm{z}^1,\ldots,\bm{z}^K) \geq \mu$ for all $(\bm{z}^1,\ldots,\bm{z}^K)$ such that $\phi(\bm{z}^1,\ldots,\bm{z}^K) = \tX$.
\end{defn}

The main result, Theorem 15 of \cite{haeffele_global_2015}, provides a characterization of the global optima of CP-based tensor factorization problems.
\begin{thm}[{Theorem 15 of \cite[page 15]{haeffele_global_2015}}]
\label{thm:0_local_min}

Given a function $f_r(\tX^1,\ldots,\tX^K,Q)$, any local minimizer of the optimization problem:
\begin{align}
	& \min_{(\tX^1,\ldots,\tX^K)_r,Q} \ f_r(\tX^1,\ldots,\tX^K,Q) \equiv \\ &\ell(\Phi_r(\tX^1,\ldots,\tX^K),Q)+\lambda \sum_{i=1}^r g(\tX^1_i,\ldots,\tX^K_i) + H(Q)\notag
\end{align}
such that $(\tX^1_{i_0},\ldots,\tX^K_{i_0})=(0,\ldots,0)$ for some $i_0 \in \{1,\ldots,r\}$ is a global minimizer.
\end{thm}
Where $Q$ is an optional set of non-factorized variables (in our case $\tE$), $H$ is convex and possibly non-smooth, $\ell$ is jointly convex and once-differentiable in $(\tX, Q)$, and $(\phi, g)$ is a nondegenerate pair.

\subsection{Outline of the proof}
In order to apply the aforementioned results, we reformulate RKCA in an equivalent formulation that satisfies the hypotheses of \cite{haeffele_structured_2014, haeffele_global_2015}. The arguments we develop actually hold for the general Tucker factorization. This section outlines the proof of equivalence, the detailed developments can be found in Appendix \ref{appendix:proof}.

The main challenge is to find a factorization function that can be expressed as a sum of elemental mappings. Problem (\ref{eq:constrained_pb}) as-is does not lend itself to such a formulation: even though the factors $\A$ and $\B$ form a size-$r$ set of 2 factors, the core tensor $\tR$ is of size $N$ on its last dimension. We note however that the set of frontal slices of $\tR$ is a size-$r$ set of $N$ factors, but this formulation doesn't satisfy Proposition 10 of \cite{haeffele_global_2015} and it is not immediately obvious how to define concatenation of the factors of $\tX$ and $\tY$ and how to verify the convexity of the factorization-regularization function $\Omega$ (Equation (18) and Proposition 11 of \cite{haeffele_global_2015} omitted here for brevity). Additionally, the results of \cite{haeffele_structured_2014, haeffele_global_2015} have been proved in the case of a single sum over the last dimension of the factors, but our factorization being a special case of Tucker decomposition is most naturally described as:
\begin{align}\label{eq:explicit_sums_factorization}
\tL & = \tR \times_1 \A \times_2 \B \\ 
	& = \tR \times_1 \A \times_2 \B \times_3 \II_N \\
    & = \sum_{i} \sum_j \sum_k r_{i, j, k} \cdot \vv{a}_i \otimes \vv{b}_j \otimes \vv{c}_k.
\end{align}

First, we note the order of the sums can be permuted such that it is clear our model expresses $\tL$ as the sum of $N$ tensors of size $m \times n \times N$, where all frontal slices are the null matrix except for one:
\begin{equation}\label{eq:sum_N_tensors}
\tL = \sum_{k} \sum_i \sum_j r_{i, j, k} \cdot \vv{a}_i \otimes \vv{b}_j \otimes \vv{c}_k.
\end{equation}
With $(\vv{c}_k)_i = \delta_{i, k} = \begin{cases}
1 & i = j\\
0 & i \neq j
\end{cases}$.

Next, we seek a transformation such that (\ref{eq:sum_N_tensors}) does not involve cross-terms. Our idea is to unfold the Tucker factorization by duplicating elements of the factors to express it in the form (\ref{eq:Phi_r_def}) where the sum is over a size-$s$ set of $3$ factors (or $M$ factors for the general Tucker).
\begin{itemize}
\item We define $\sigma = \vvec(\tR)$.
\item We construct $\tilde{\A}$ and $\tilde{\B}$ from $\A$ and $\B$ by duplicating columns.
\item $\II_N$ doesn't depend on the data and can be injected directly in the elemental mapping.
\end{itemize}
We show the factorization function is defined over the size-$Nr^2$ set of 3 factors $(\bm{\sigma}, \tilde{\A}, \tilde{\B})$ by:
\begin{equation}
\Phi_{Nr^2} (\bm{\sigma}, \tilde{\A}, \tilde{\B}) = \sum_{l=1}^{Nr^2} \sigma_l \cdot \tilde{\vv{a}}_l \otimes \tilde{\vv{b}}_l \otimes \bm{\delta}_{l, (\ceil{\frac{l}{r^2}} - 1) \, \text{mod} \, r^2 + 1}.
\end{equation}

\subsection{RKCA with degree 3 regularization}
\label{sec:kdrsdl_degree_three_regulatization}

We now check that the \textit{degree 3 regularizers} introduced in Section \ref{sec:model_derivation} are compatible with our reformulated factorization. Recall both regularizers are of the form:
\begin{equation}
g(\tL) = \alpha \One{\tR} ||\A||_q ||\B||_q.
\end{equation}
With $||.||_p$ the Schatten-$p$ norm, $p = 1$ for the Nuclear norm and $p = 2$ for the Frobenius norm. As we shall see in Section \ref{sec:fro_nuc_equivalence}, $g$ is low-rank promoting.

In the case of the Frobenius penalty, the resulting optimization problem is:
\begin{align}\label{eq:constrained_pb_deg_three}
    \begin{matrix*}[l]
    \min_{\A, \B, \ten{R}, \ten{E}} & \lambda \One{\ten{R}}\Fro{\B} \Fro{\A} + \lambda \One{\tE},\\
    \st & \tX = \ten{R} \times_1 \A \times_2 \B + \tE.
    \end{matrix*}
\end{align}

The regularizer $g$ adapted to our transformation is:
\begin{equation}
g(\bm{\sigma}, \tilde{\A}, \tilde{\B}) = \One{\bm{\sigma}} \frac{\Fro{\tilde{\A}}}{Nr} \frac{\Fro{\tilde{\B}}}{Nr}.
\end{equation}

This stems from the fact that the Frobenius norm is equivalent to the element-wise $\ell_2$ norm, and each element of $\A$ and $\B$ appears $Nr$ times in $\tilde{\A}$ and $\tilde{\B}$.

In the case of the Nuclear norm, we simply observe that the numbers of linearly independent columns in $\tilde{\A}$ and $\tilde{\B}$ are clearly the same as in $\A$ and $\B$, so our transformation preserves the ranks of the target matrices, and the regularization function is simply:
\begin{equation}
f(\bm{\sigma}, \tilde{\A}, \tilde{\B}) = \One{\bm{\sigma}} \Nuc{\tilde{\A}} \Nuc{\tilde{\B}}.
\end{equation}

It should be clear that the newly defined $\phi$ and $g$ are both positively homogeneous of degree 3, and form a nondegenerate pair. Property 10 of \cite{haeffele_global_2015} also holds.

Hence, we argue that RKCA with a product of norms regularization enjoys the optimality guarantees of \cite{haeffele_global_2015}.

%% file: low_rankness.tex
\section{RKCA learns low-rank dictionaries}
\label{sec:low-rankness}



In this Section we explain the low-rank promoting behavior of RKCA from complementary perspectives. First, we show the regularizers defined in equations (\ref{eq:degree_two_f}), (\ref{eq:degree_three_f_fro}), and (\ref{eq:degree_three_f_nuc}) directly provide upper bounds on the mode-1 and mode-2 ranks of the low-rank component, and thus on the rank of each of its frontal slices. This perspective was first presented in \cite{m._bahri_robust_2017}. The second approach to explaining the models' properties studies the optimization sub-problems associated with the two bases $\A$ and $\B$. Based on recent work \cite{x._peng_connections_2018}, we show these sub-problems are equivalent to rank-minimization problems and admit closed-form solutions that involve forms of singular value thresholding. Finally, we discuss connections with recent work on low-rank inducing norms.

\subsection{Direct penalization of the rank of $\tL$}
Seeing the models from the perspective of Robust PCA, which seeks a low-rank representation $\A$ of the dataset $\X$, we minimize the rank of the low-rank tensor $\tL$. More precisely, we show in Theorem \ref{thm:rank} that we simultaneously penalize the Tucker rank and the multi-rank of $\tL$.

\begin{thm}\label{thm:rank}
RKCA encourages low mode-$1$ and mode-$2$ rank, and thus, low-rankness in each frontal slice of $\tL$, for suitable choices of the parameters $\lambda$ and $\alpha$.
\end{thm}
\begin{proof}
We minimize either $\lambda \One{\tE} + \alpha \One{\tR} + \Fro{\A \otimes \B}$, $\lambda \One{\tE} + \alpha \One{\tR}\Fro{\A \otimes \B}$, or $\lambda \One{\tE} + \alpha \One{\tR}\Nuc{\A \otimes \B}$. From the equivalence of norms in finite-dimensions, $\exists k \in \RR_+^*, \Nuc{\A \otimes \B} \leq k \Fro{\A \otimes \B}$. In the case of the Frobenius norm, we can choose $\alpha = \frac{\alpha'}{k}, \lambda = \frac{\lambda'}{k}$ to reduce the problem to that of the nuclear norm. In all cases, we penalize $\rk{\A \otimes \B} = \rk{\A}\rk{\B}$. Given that the rank is a non-negative integer, $\rk{\A}$ or $\rk{\B}$ decreases necessarily. Therefore, we minimize the mode-$1$ and mode-$2$ ranks of $\tL = \tR \times_1 \A \times_2 \B$. Additionally, $\forall i, \rk{\A \Ri \Bt} \leq \min(\rk{\A}, \rk{\B}, \rk{\Ri})$.

Exhibiting a valid $k$ may help in the choice of parameters. We show $\forall \A \in \RR^{m \times n}, \; \Nuc{\A} \leq \sqrt{\min(m, n)}\Fro{\A}$: We know $\forall \vv{x} \in \RR^n, \One{\vv{x}} \leq \sqrt{n} \Two{\vv{x}}$. Recalling the nuclear norm and the Frobenius norm are the Schatten-$1$ and Schatten-$2$ norms, and $\A$ has $\min(m, n)$ singular values, the result follows.
\end{proof}
In practice, we find that on synthetic data designed to test the models, we effectively recover the ranks of $\A$ and $\B$ regardless of the choice of $r$, as seen in Figure \ref{fig:sample_ranks}.
\begin{figure}
\centering
\includegraphics[width=.9\columnwidth]{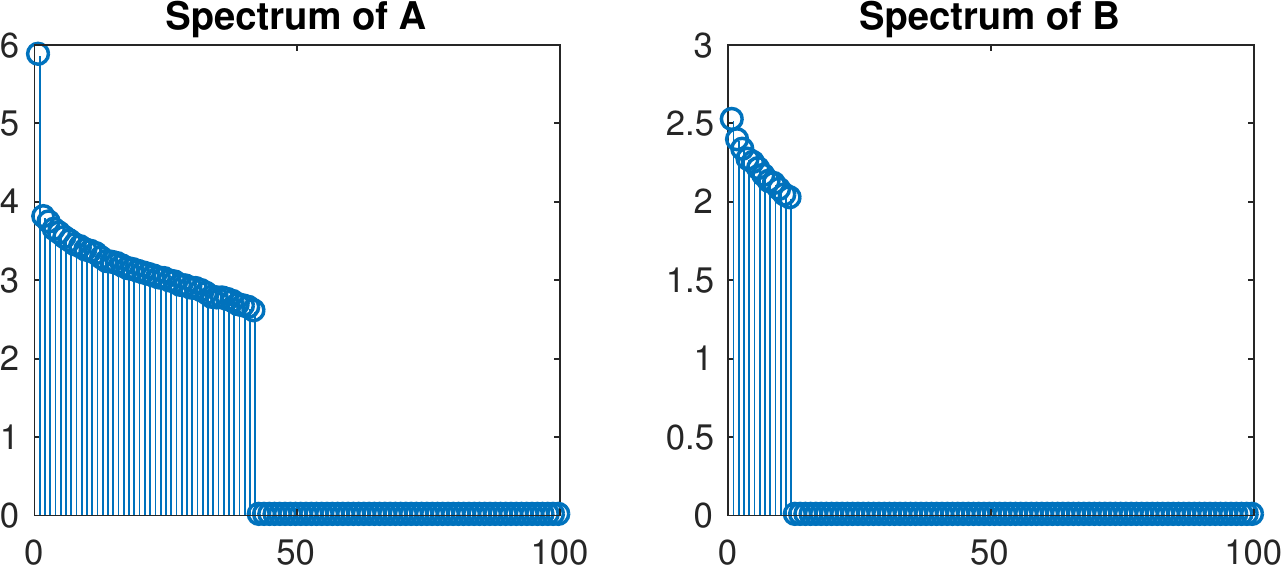}
\caption{Sample spectrums of $\A$ and $\B$. Ground truth attained ($\A$: 42, $\B$: 12). $r = 100$. Degree 2 regularizer.}
\label{fig:sample_ranks}
\end{figure}

\subsection{Analysis of the sub-problems in A and B}
\label{sec:fro_nuc_equivalence}

We present an alternative interpretation of the low-rank properties of our methods. In \cite{x._peng_connections_2018} the authors show that Frobenius norm and Nuclear norm regularizations either lead to the same optimal solution - even in the presence of noise - when the dictionary provides enough representative power, or are equivalent in the sense that they describe two different bases on the column space of the dictionary. We explain the behavior of our algorithms by showing the sub-problems in $\A$ and $\B$ are rank-minimization problems.

\subsubsection{Sub-problems in A and B}

For generality purposes, let us begin with the problem:
\begin{align}\label{eq:q}
    \begin{matrix*}[l]
    \displaystyle \min_{\A, \B, \ten{R}, \ten{E}} & \beta(\tR) \Fro{\A}\Fro{\B} + \lambda \One{\tE} + W(\tR),\\
    \st & \tX = \tR \times_1 \A \times_2 \B \times_3 \II_N + \tE.
    \end{matrix*}
\end{align}

In the degree 3 case we have $\beta(\tR) = \alpha \One{\ten{R}}$ and $W(\tR) = 0$, in the degree 2 case, $\beta(\tR) = 1$ and $W(\tR) = \alpha \One{\ten{R}}$. We now omit the dependency in $\ten{R}$ for $\beta$ and $W$ as they are constant in the sub-problems in $\A$ and $\B$.

The sub-problem in $\A$ is:
\begin{align}\label{eq:qq}
    \begin{matrix*}[l]
    \displaystyle \min_{\A} & \left[ \beta\Fro{\B} \right] \Fro{\A},\\
    \st & \tX = \tR \times_1 \A \times_2 \B \times_3 \II_N + \tE.
    \end{matrix*}
\end{align}

Equivalently, in a matricized way using Property (\ref{prop:matri_mode_n}):
\begin{align}\label{eq:qqq}
    \begin{matrix*}[l]
    \displaystyle \min_{\A} & \left[ \beta \Fro{\B} \right] \Fro{\A},\\
    \st & \X_{[1]} = \A \R_{[1]} \trp{(\II_N \otimes \B)} + \E_{[1]}.
    \end{matrix*}
\end{align}

Noting that $\Fro{\A} = \Fro{\trp{\A}}$ we reformulate the problem in the form of Equation (4) of \cite{x._peng_connections_2018}:

\begin{align}\label{eq:qqqq}
    \displaystyle \min_{\A} \gamma \Fro{\trp{\A}}, \quad \st \quad \trp{\tilde{\X}_{[1]}} =  \D_{\A} \trp{\A}.
\end{align}
With $\trp{\tilde{\X}_{[1]}} = \trp{\X_{[1]}} - \trp{\E_{[1]}}$, $\gamma = \left[ \beta\Fro{\B} \right]$, $\D_{\A} = (\II_N \otimes \B) \trp{\R_{[1]}}$.

The sub-problem in $\B$ is very similar:

\begin{align}\label{eq:qqqqq}
    \begin{matrix*}[l]
    \displaystyle \min_{\B} & \left[ \beta \Fro{\A} \right] \Fro{\B},\\
    \st & \tX = \tR \times_1 \A \times_2 \B \times_3 \II_N + \tE.
    \end{matrix*}
\end{align}

Matricizing on the second mode:
\begin{align}\label{eq:qqqqqq}
    \displaystyle \min_{\B} \kappa \Fro{\trp{\B}}, \quad \st \quad \trp{\tilde{\X}_{[2]}} =  \D_{\B} \trp{\B}.
\end{align}
With $\trp{\tilde{\X}_{[2]}} = \trp{\X_{[2]}} - \trp{\E_{[2]}}$, $\kappa = \left[ \beta \Fro{\A} \right]$, $\D_{\B} = (\II_N \otimes \A) \trp{\R_{[2]}}$.

\subsubsection{Optimal solutions and interpretation}

According to \cite{g._liu_robust_2013} the optimal solution of $\min \Nuc{\C}, \, \st \, \X = \D \C$ assuming feasible solutions exist and $\D \neq \vv{0}$ is $\C^* = \V_r \trp{\V_r}$ where $\U_r \mat{\Sigma}_r \trp{\V_r}$ is the skinny SVD of $\D$. \cite{h._zhang_flrr:_2014} showed it is also the optimal solution of $\min \Fro{\C}, \, \st \, \X = \D \C$.

It is therefore easy to show that the optimal solution of $\min \Nuc{\trp{\C}} \, \st \, \X = \D \trp{\C}$ under similar conditions is $\U_r \trp{\U_r}$.
\begin{proof}
Letting $\X = \U \mat{\Sigma} \trp{\V}$ we have $\trp{\X} = \V \mat{\Sigma} \trp{\U}$.
\end{proof}

Letting $\U^{\A}_p \mat{\Sigma}^{\A}_p \trp{{\V^{\A}}_p}$ the skinny SVD of $\D_{\A}$ and $\U^{\B}_q \mat{\Sigma}^{\B}_q \trp{{\V^{\B}}_q}$ the skinny SVD of $\D_{\B}$ the optimal solutions for $\A$ and $\B$ are therefore $\U^{\A}_p \trp{{\U^{\A}}_p}$ and $\U^{\B}_q \trp{{\U^{\B}}_q}$.

For the sake of brevity, we shall describe the interpretation for $\A$ as the same holds for $\B$ by symmetry. It should be noted that $\U_p^{\A}$ is the PCA basis of $\D_{\A}$ and that $\U^{\A}_p \trp{{\U^{\A}}_p}$ is the matrix of the orthogonal projection onto that basis. Thus, the optimal $\A$ is the projection of $\D_{\A}$ onto its principal components. Remembering that $\D_{\A} = (\II_N \otimes \B) \trp{\R_{[1]}}$, $\D_{\A}$ is the product of $\B$ with each code $\R_i$ matricized such that the $i^th$ line of $\D_{\A}$ is the concatenation of all the $i^{th}$ columns of all the partial reconstructions $\R_k \Bt$, and all the columns of $\D_{\A}$ are all the lines of $\R_k \Bt$, which can be seen as the partial reconstruction of the low-rank images in their common row space $\spn{\B}$.

Hence, we can see the process of updating $\A$ and $\B$ as alternating PCAs in the row and column spaces, respectively.

\subsection{A note on low-rank inducing norms}
\label{sec:low_rank_inducing_norms}

The recently published \cite{grussler_low-rank_2018} suggests one more interpretation of our regularizers. In \cite{grussler_low-rank_2018}, the authors show a family of low-rank inducing norms, based on solving rank-constrained Frobenius norm minimization, or rank-constrained Spectral norm minimization problems can outperform the standard nuclear norm relaxation in rank-minimization problems. Given that the role of the parameter $r$ in our model is to impose an upper bound on the ranks of the dictionaries and of the reconstructions, we suggest the choice of Frobenius regularization instead of nuclear norm regularization is further justified. Deeper analysis of this connection is required and is left for future work.

%% file: implementation_details.tex
\section{Implementation details and complexity}
\label{seq:implem_details}

In this section, we discuss the computational complexity of our algorithms, and the scalability issues depending on the regularization. We discuss parameter tuning in Appendix \ref{appendix:param_tuning}. The initialization procedure has already been discussed in \cite{m._bahri_robust_2017}, and is included in Appendix \ref{appendix:initialization}.

\subsection{Computational complexity and LADMM}
\label{sec:ladmm}

The substitution method used in Section \ref{sec:kdrsdl} is effective but comes with an additional cost that limits its scalability. Notably, the cost of solving $N$ Stein equations cannot be neglected in practice. Additionally, the added parameters can make tuning difficult. Linearization provides an alternative approach that can scale better to large dimensions and large datasets. In Appendix \ref{appendix:ladmm}, we give detailed derivations of how RKCA can be implemented with LADMM.


The time and space complexity per iteration of Algorithm \ref{alg:RKCA} (i.e., with the degree two regularizer) are $O(N(mnr + (m + n)r + mn + \min(m, n)r^2 + r^3 + r^2))$ and $O( N(mn + r^2) + (m + n)r + r^2  )$. Since $r \leq \min(m, n)$, the terms in $r$ are asymptotically negligible, but in practice it is useful to know how the computational requirements scale with the size of the dictionary. Similarly, the initialization procedure has cost $O(N(m n \min (m, n) + (\min(m ,n))^3 + mn) + mn)$ in time and needs quadratic space per slice, assuming a standard algorithm is used for the SVD\cite{chan_improved_1982}. 

Switching to an LADMM update for $\tR$ eliminates the need of solving costly Stein equations. The soft-shrinkage operator is applied to the tensor $(\tR \times_1 \A \times_2 \B - \bm{\Delta}) \times_1 \trp{\A} \times_2 \trp{\B}$, which has $Nr^2$ elements, and can be computed in $O(N(\min(m, n)r^2 + mn + r^2 + r^3) + mr^2 + nr^2)$ by remembering that $(\tX \times_1 \A) \times_1 \B = \tX \times_1 \B \A$. The space complexity of the update is $O(N(\min(m,n)r + r^2) + r^2)$.

Updating $\A$ and $\B$ with a Frobenius norm or with a nuclear norm penalty requires computing a proximal operator, and solving a linear system in the case of ADMM with substitution. We focus on the computation of the proximal operator in Section \ref{sec:fro_vs_nuclear}. The substitution adds an additional time and space complexity of $O(mr + nr)$.

Several key steps of Algorithm \ref{alg:RKCA} and its variants, such as summations of independent terms, are trivially distributed in a \textit{MapReduce} \cite{dean_mapreduce:_2004} way. Proximal operators are separable in nature and are therefore parallelizable. Consequently, highly parallel and distributed implementations are possible, and computational complexity can be further reduced by adaptively adopting sparse linear algebra structures.

\subsection{Frobenius and nuclear norm penalties}
\label{sec:fro_vs_nuclear}

In Equation (\ref{eq:constrained_pb_rpca_upper_bound}) we used an upper bound on $\Fro{\A}\Fro{\B}$ to obtain smooth sub-problems. In the degree 3-regularized case, we did not apply the same bound. Given the non-smoothness of the Frobenius norm, we must resort to other approaches such as substitution or linearization (c.f. Appendix \ref{appendix:substitution_updates_degree_three} and Appendix \ref{appendix:ladmm}), as with the Nuclear norm.


In Section \ref{sec:fro_nuc_equivalence}, we then showed how Frobenius and Nuclear norm penalties can yield the same optimal solutions, and how in our case they would be equivalent in solving the sub-problems associated with $\A$ and $\B$. It is therefore natural to wonder why we should choose one over the other, and what the practical implications of this choice are.

It can be shown that the proximal operator of the Schatten-$p$ norm has a closed form expression that requires the computation of the SVD of the matrix (Proposition \ref{prop:prox_operator_schatten_p}, Appendix \ref{appendix:schatten_p_prox}). Computing the SVD is costly, with a typical algorithm \cite{chan_improved_1982} requiring $O(mn \min (m, n) + (\min (m, n))^3)$ floating-point operations. Storing the resulting triple $(\U, \mS, \V)$ requires $m \times \min(m, n) + \min(m,n)^2 + n \times \min(m, n)$ space. Although faster algorithms have been developed, scalability remains a concern.

However, the Frobenius norm is equal to the matrix element-wise $\ell_2$ norm, whose proximal operator is only the projection on the unit ball and doesn't require any costly matrix decomposition. The choice of the Frobenius norm is therefore justified in the high dimensional setting. A comparable use of the Frobenius norm in lieu of the Nuclear norm for representation learning can be seen in \cite{x._peng_automatic_2017}, also motivated by scalability.

\begin{figure}
\centering
\includegraphics[width=.9\columnwidth]{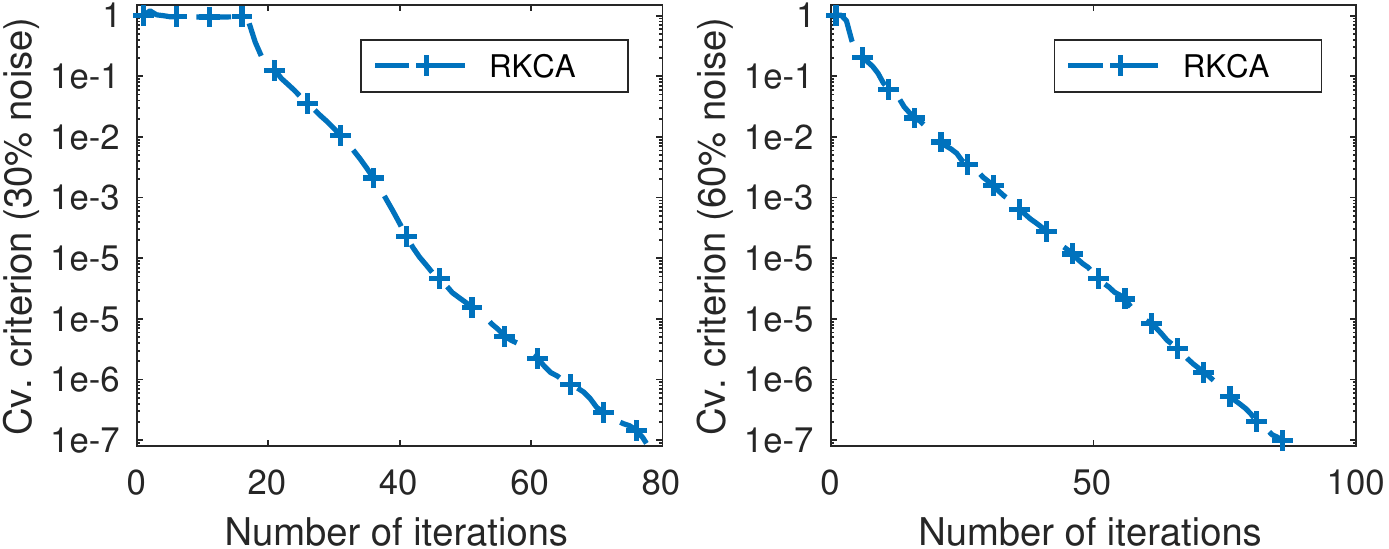}
\caption{Convergence on synthetic data with 30\% and 60\% corruption.}
\label{fig:sample_convergence}
\end{figure}

%% file: experimental_evaluation/experimental_evaluation.tex
\section{Experimental evaluation\protect}
\label{sec:experiments}

We first provide experimental verification of the correctness of Algorithm \ref{alg:RKCA} (with degree 2 regularization) on synthetic data. We then compared the performance of RKCA with degree 2 regularization against a range of state-of-the-art tensor decomposition algorithms on four low-rank modeling computer vision benchmarks: two for image denoising, and two for background subtraction. As a baseline, we report the performance of matrix Robust PCA implemented via inexact ALM (\textit{RPCA}) \cite{candes_robust_2011, lin_augmented_2010}, of Non-Negative Robust Dictionary Learning (\textit{RNNDL}) \cite{pan_robust_2014}, and of the Low-Rank Representations algorithm \cite{g._liu_robust_2013} implemented with both exact ALM (\textit{LRRe}) and inexact ALM (\textit{LRRi}). We chose the following methods to include recent representatives of various existing approaches to low-rank modeling on tensors: The singleton version of Higher-Order Robust PCA (\textit{HORPCA-S}) \cite{goldfarb_robust_2014} optimizes the Tucker rank of the tensor through the sum of the nuclear norms of its unfoldings. In \cite{y._yang_robust_2016}, the authors consider a similar model but with robust M-estimators as loss functions, either a Cauchy loss or a Welsh loss, and support both hard and soft thresholding; we tested the soft-thresholding models (\textit{Cauchy ST} and \textit{Welsh ST}). Non-convex Tensor Robust PCA (\textit{NC TRPCA}) \cite{anandkumar_tensor_2016} adapts to tensors the matrix non-convex RPCA \cite{netrapalli_non-convex_2014}. Finally, the two Tensor RPCA algorithms \cite{c._lu_tensor_2016, z._zhang_novel_2014} (\textit{TRPCA '14} and \textit{TRPCA '16}) work with slightly different definitions of the tensor nuclear norm as a convex surrogate of the tensor multi-rank. In addition to the aforementioned low-rank matrix and tensor factorizations, we include in the comparison the recently proposed Deep Image Prior \cite{ulyanov_deep_2018} with two different architectures (\textit{DIP 1} and \textit{DIP 2}) as an example of deep learning based approach for denoising. Our choice of comparing to DIP is motivated by the fact that DIP is fully unsupervised, designed for denoising (among other tasks), does not require a large dataset for training, and is close in spirit to tensor factorizations. Given the small sizes of the datasets at hand, training deep neural networks would have been impractical.

\textit{For all models but RNNDL, we used the implementation made available by their respective authors, either publicly or on request. Our implementation of RNNDL was tested for conformity with the original paper.}

For each model, with the exception of DIP, we identified a maximum of two parameters to tune via grid-search in order to keep parameter tuning tractable. When criteria or heuristics for choosing the parameters were provided by the authors, we chose the search space around the value obtained from them. In all cases, the tuning process explored a wide range of parameters to maximize performance. In the special case of DIP, we re-used the two architectures implemented in the denoising example in the code made available online by the authors (c.f. Appendix \ref{appendix:dip}). We then trained for 2500 iterations while keeping track of the best reconstruction so-far measured by the PSNR. The best reconstruction is then saved for further comparison. Following \cite{ulyanov_deep_2018}, the input is chosen to be random. As in \cite{g._liu_robust_2013}, the data matrix itself was used as the dictionary for LRRe and LRRi.

When the performance of one method was significantly worse than that of the other, the result is not reported so as not to clutter the text (see Appendix \ref{appendix:additional_experiments}). This is the case of Separable Dictionary Learning \cite{hawe_separable_2013} whose drastically different nature renders unsuitable for robust low-rank modeling, but was compared for completeness. For the same reason, we did not compare our method against K-SVD \cite{m._aharon_k-svd:_2006}, or \cite{s._hsieh_2d_2014}.

Finally, we provide tensor completion experiments with and without gross corruption in Section \ref{sec:tensor_completion_experiments}.

\input{experimental_evaluation/synthetic_data}

\input{experimental_evaluation/background_modelling/background_modelling}

\input{experimental_evaluation/image_denoising/image_denoising}

\input{experimental_evaluation/tensor_completion}

%% file: experimental_evaluation/synthetic_data.tex
\subsection{Validation on synthetic data}

We generated synthetic data following the RKCA model's assumptions by first sampling two random bases $\A$ and $\B$ of known ranks $r_{\A}$ and $r_{\B}$, $N$ Gaussian slices for the core $\tR$, and forming the ground truth $\tL = \tR \times_1 \A \times_2 \B$. We modeled additive random sparse Laplacian noise with a tensor $\tE$ whose entries are $0$ with probability $p$, and $1$ or $-1$ with equal probability otherwise. We generated data for $p = 70\%$ and $p = 40\%$, leading to a noise density of, respectively, $30\%$ and $60\%$. We measured the reconstruction error on $\tL$ and $\tE$, and the density of $\tE$ for varying values of $\lambda$, and $\alpha = \num{1e-2}$. Our model achieved near-exact recovery of both $\tL$ and $\tE$, and exact recovery of the density of $\tE$, for suitable values of $\lambda$. Evidence is presented in Figure \ref{fig:synthdata_60} for the $60\%$ noise case.

\begin{figure}
\centering
\includegraphics[width=.9\columnwidth]{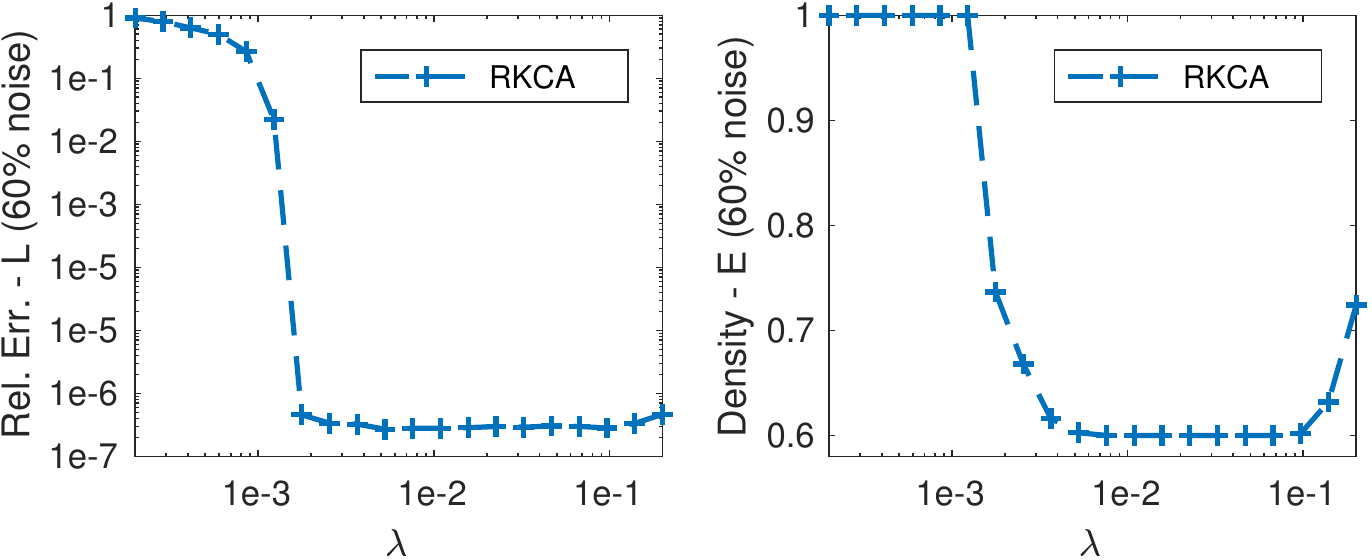}
\caption{Recovery performance with 60\% corruption. Relative $\ell_2$ error and density.}
\label{fig:synthdata_60}
\end{figure}

Algorithm \ref{alg:RKCA} appears robust to small changes in $\lambda$, which suggests more than one value can lead to optimal results, and that a simple criterion that provides consistently good reconstruction might be derived, as in Robust PCA \cite{candes_robust_2011}. In the $30\%$ noise case, we did not observe an increase in the density of $\tE$ as $\lambda$ increases, and the $\ell_2$ error on both $\tE$ and $\tL$ was of the order of \num{1e-7}.

%% file: experimental_evaluation/background_modelling/background_modelling.tex
\subsection{Background subtraction}

Background subtraction is a common task in computer vision and can be tackled by robust low-rank modeling: the static or mostly static background of a video sequence can effectively be represented as a low-rank tensor while the foreground forms a sparse component of outliers.

\subsubsection{Experimental procedure}

We compared the algorithms on two benchmarks. The first is an excerpt of the \textit{Highway} dataset \cite{n._goyette_changedetection.net:_2012}, and consists in a video sequence of cars travelling on a highway; the background is completely static. We kept 400 gray-scale images re-sized to $48 \times 64$ pixels. The second is the \textit{Airport Hall} dataset (\cite{liyuan_li_statistical_2004}) and has been chosen as a more challenging benchmark since the background is not fully static and the scene is richer. We used the same excerpt of 300 frames (frames 3301 to 3600) as in \cite{q._zhao_bayesian_2016}, and kept the frames in their original size of $144 \times 176$ pixels.

We treat background subtraction as a binary classification problem. Since ground truth frames are available for our excerpts, we report the AUC \cite{fawcett_introduction_2006} on both videos.
The value of $\alpha$ was set to \num{1e-2} for both experiments.

\subsubsection{Results}

The original, ground truth, and recovered frames are in Figure \ref{fig:visual_hall} for the \textit{Hall} experiment (\textit{Highway} in Appendix \ref{appendix:additional_experiments}).

Table \ref{tab:perf_bg} presents the AUC scores of the algorithms, ranked in order of their mean performance on the two benchmarks. The two matrix methods rank high on both benchmarks and only half of the tensor algorithms match or outperform this baseline. Our proposed model matches the best performance on the \textit{Highway} dataset and provides significantly higher performance than the other on the more challenging \textit{Hall} benchmark. Visual inspection of the results show RKCA is the only method that doesn't fully capture the immobile people in the background, and therefore achieves the best trade-off between foreground detection and background-foreground contamination.

\begin{table}[h]
\scriptsize
\centering
\resizebox{\columnwidth}{!}{
\begin{tabularx}{\columnwidth}{|>{\centering\arraybackslash}X|>{\hsize=.5\hsize\centering\arraybackslash}X|>{\hsize=.5\hsize\centering\arraybackslash}X|}\hline
\textbf{Algorithm} & \textbf{Highway} & \textbf{Hall} \\ \hline
\textbf{RKCA (proposed)} & 0.94 & 0.88 \\ \hline
TRPCA '16 & 0.94 & 0.86 \\ \hline
NC TRPCA   & 0.93 & 0.86 \\ \hline
\textit{RPCA (baseline)} & 0.94 & 0.85 \\ \hline
\textit{RNNDL (baseline)} & 0.94 & 0.85\\ \hline
\textit{LRR Exact (baseline)} & 0.94 & 0.84\\ \hline
\textit{LRR Inexact (baseline)} & 0.93 & 0.84\\ \hline
HORPCA-S  & 0.93 & 0.86 \\ \hline
Cauchy ST & 0.83 & 0.76 \\ \hline
Welsh ST & 0.82 & 0.71 \\ \hline
TRPCA '14 & 0.76 & 0.61 \\ \hline
\end{tabularx}
}
\caption[AUC scores]{AUC on \textit{Highway} and \textit{Hall} ordered by mean AUC.}
\label{tab:perf_bg}
\normalsize
\end{table}

\input{experimental_evaluation/background_modelling/image_grid_hall}

%% file: experimental_evaluation/background_modelling/image_grid_hall.tex
\begin{figure}
\captionsetup[sub]{font=scriptsize}
\begin{subfigure}[b]{.19\linewidth}
\includegraphics[width = \linewidth]{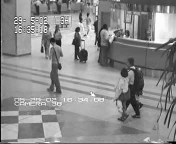} 
\caption{Original}
\end{subfigure}
\begin{subfigure}[b]{.19\linewidth}
\includegraphics[width = \linewidth]{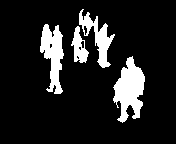}
\caption{GT}
\end{subfigure}
\begin{subfigure}[b]{.19\linewidth}
\includegraphics[width = \linewidth]{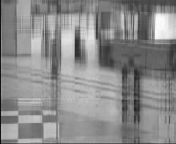} 
\caption{ Cauchy ST}
\end{subfigure}\hfill
\begin{subfigure}[b]{.19\linewidth}
\includegraphics[width = \linewidth]{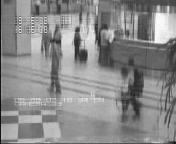} 
\caption{Welsh ST}
\end{subfigure}\hfill
\begin{subfigure}[b]{.19\linewidth}
\includegraphics[width = \linewidth]{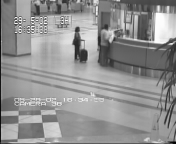}
\caption{TRPCA '16}
\end{subfigure}\hfill
\begin{subfigure}[b]{.19\linewidth}
\includegraphics[width = \linewidth]{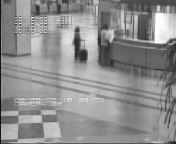} 
\caption{HORPCA-S}
\end{subfigure}\hfill
\begin{subfigure}[b]{.19\linewidth}
\includegraphics[width = \linewidth]{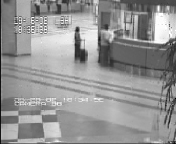} 
\caption{NCTRPCA}
\end{subfigure}\hfill
\begin{subfigure}[b]{.19\linewidth}
\includegraphics[width = \linewidth]{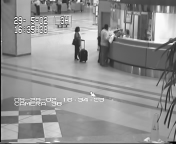} 
\caption{RNNDL}
\end{subfigure}\hfill
\begin{subfigure}[b]{.19\linewidth}
\includegraphics[width = \linewidth]{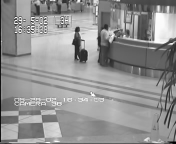} 
\caption{RPCA}
\end{subfigure}\hfill
\begin{subfigure}[b]{.19\linewidth}
\includegraphics[width = \linewidth]{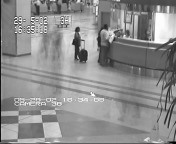} 
\caption{LRRe}
\end{subfigure}\hfill
\begin{subfigure}[b]{.19\linewidth}
\includegraphics[width = \linewidth]{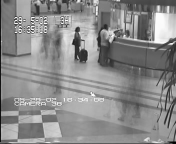} 
\caption{LRRi}
\end{subfigure}\hfill
\begin{subfigure}[b]{.19\linewidth}
\includegraphics[width = \linewidth]{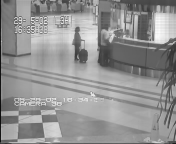}
\caption{RKCA}
\end{subfigure}\hfill
\begin{subfigure}[b]{.19\linewidth}
\caption*{}
\end{subfigure}\hfill
\begin{subfigure}[b]{.19\linewidth}
\caption*{}
\end{subfigure}\hfill
\begin{subfigure}[b]{.19\linewidth}
\caption*{}
\end{subfigure}\hfill
\caption{Background subtraction results on \textit{Airport Hall}. TRPCA '14 removed.}
\label{fig:visual_hall}
\end{figure}

%% file: experimental_evaluation/image_denoising/image_denoising.tex
\subsection{Image denoising}

Many natural and artificial images exhibit an inherent low-rank structure and are suitably denoised by low-rank modeling algorithms. In this section, we assess the performance of the cohort on two datasets chosen for their popularity, and for the typical use cases they represent.

We consider collections of grayscale images, and color images represented as 3-way tensors. Laplacian (salt \& pepper) noise was introduced separately in all frontal slices of the observation tensor at three different levels: 10\%, 30\%, and 60\%, to simulate medium, high, and gross corruption. In these experiments we set the value of $\alpha$ to \num{1e-3} for noise levels up to 30\%, and to \num{1e-2} at the 60\% level.

We report two complementary image quality metrics. The \textit{Peak Signal To Noise Ratio (PSNR)} will be used as an indicator of the element-wise reconstruction quality of the signals, while the \textit{Feature Similarity Index (FSIM, FSIMc for color images)} \cite{l._zhang_fsim:_2011} evaluates the recovery of structural information. Quantitative metrics are not perfect replacements for subjective assessment of image quality; therefore, we present reconstructed images for verification. Our measure of choice for determining which images to compare visually is the FSIM(c) for its higher correlation with human evaluation than the PSNR \cite{pedram_mohammadi_subjective_2014}.

\input{experimental_evaluation/image_denoising/face_denoising/face_denoising}

\input{experimental_evaluation/image_denoising/colour_denoising/colour_images}

%% file: experimental_evaluation/image_denoising/face_denoising/face_denoising.tex
\subsubsection{Monochromatic face images}
\label{sec:face_denoising}

Our face denoising experiment uses the Extented Yale-B dataset \cite{a._s._georghiades_few_2001} of 10 different subject, each under 64 different lighting conditions. According to \cite{ramamoorthi_efficient_2001, r._basri_lambertian_2003}, face images of one subject under various illuminations lie approximately on a 9-dimensional subspace, and are therefore suitable for low-rank modeling. We used the pre-cropped 64 images of the first subject and kept them at full resolution. The resulting collection of images constitutes a 3-way tensor of 64 images of size $192 \times 168$. Each mode corresponds respectively to the columns and rows of the images, and to the illumination component. All three are expected to be low-rank due to the spatial correlation within frontal slices and to the correlation between images of the same subject under different illuminations. We present the comparative quantitative performance of the methods tested in Figure \ref{fig:perf_yale}, and provide visualizations of the reconstructed first image at the 30\% noise level in Figure \ref{fig:visual_yale_30}. We report the metrics averaged on the 64 images. For DIP, we chose to use a single model for the 64 images, rather than one model per image, to provide a better comparison with the other approaches. In this case, the input dimension was chosen to be 128.

\begin{figure}
    \centering
    \includegraphics[width=.9\linewidth]{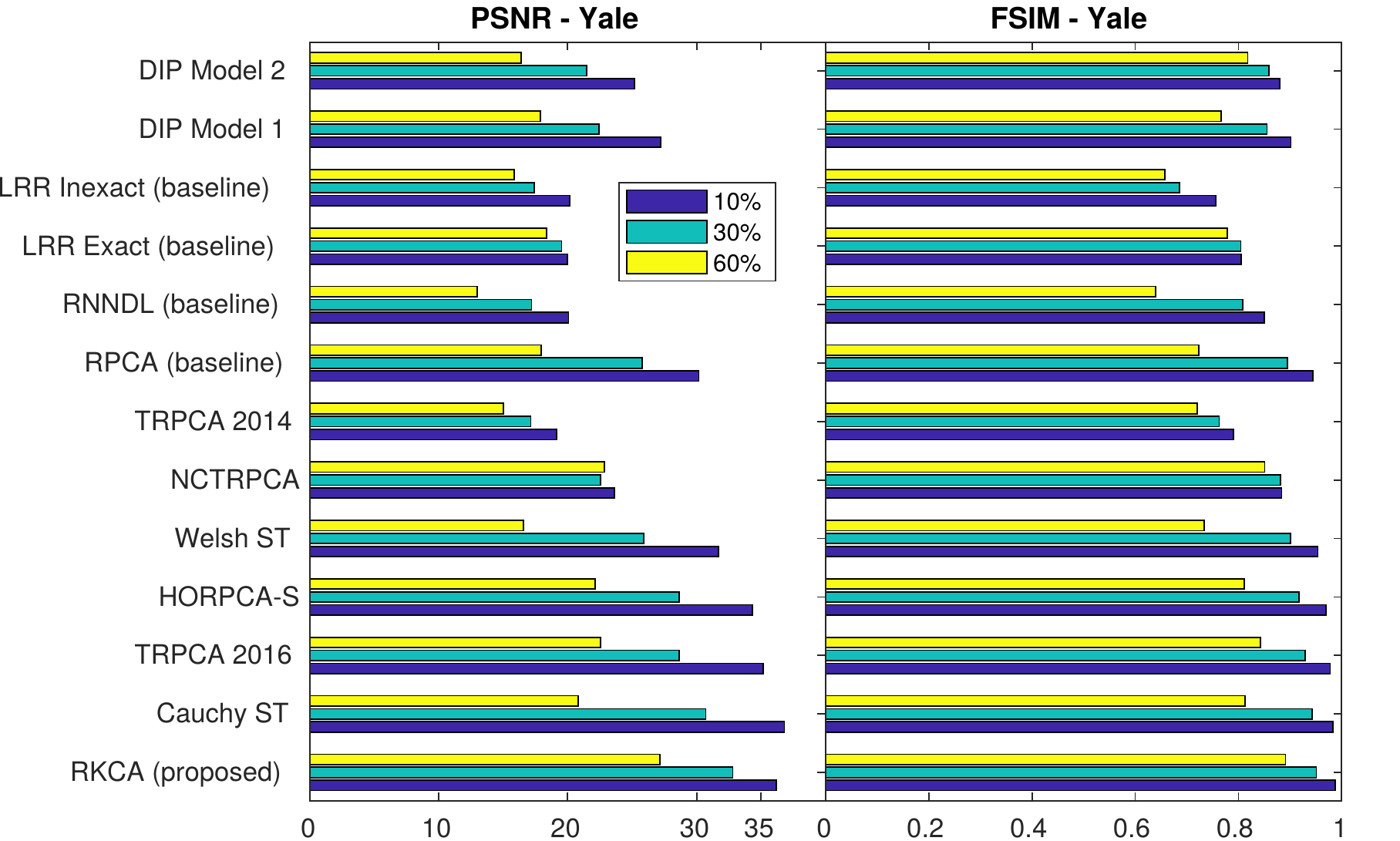}
    \caption{Mean PSNR and FSIM on the 64 images of the first subject of \textit{Yale }at noise levels 10\%, 30\%, and 60\%.}
    \label{fig:perf_yale}
\end{figure}

At the 10\% noise level, nearly every method provided good to excellent recovery of the original images. We therefore omit this noise level (cf. Appendix \ref{appendix:additional_experiments}, Table \ref{tab:full_results_yale} and figures). On the other hand, most methods, with the notable exception of RKCA, NC TRPCA, and TRPCA '16, failed to provide acceptable reconstruction in the gross corruption case. Thus, we present the denoised images at the 30\% level, and compare the performance of the three best performing methods in Table \ref{tab:top_3_yale_60} for the 60\% noise level.

Clear differences appeared at the 30\% noise level, as demonstrated both by the quantitative metrics, and by visual inspection of Figure \ref{fig:visual_yale_30}. Overall, performance was markedly lower than at the 10\% level, and most methods started to lose much of the details. Visual inspection of the results confirms a higher reconstruction quality for RKCA. We invite the reader to look at the texture of the skin, the white of the eye, and at the reflection of the light on the subject's skin and pupil. The latter, in particular, is very close in nature to the white pixel corruption of the salt \& pepper noise. Out of all methods, RKCA provided the best reconstruction quality: it is the only algorithm that removed all the noise and for which all the aforementioned details are distinguishable in the reconstruction. Both deep learning approaches provided reconstructions with visually-identifiable features, but under-performed compared to the best robust factorization models.

\input{experimental_evaluation/image_denoising/face_denoising/image_grid_yale}

\input{experimental_evaluation/image_denoising/face_denoising/table_yale_60}

At the 60\% noise level, our method scored markedly higher than its competitors on image quality metrics, as seen both in Figure \ref{fig:perf_yale} and in Table \ref{tab:top_3_yale_60}. Visualizing the reconstructions confirms the difference: the image recovered by RKCA at the 60\% noise level is comparable to the output of competing algorithms at the 30\% noise level.

%% file: experimental_evaluation/image_denoising/face_denoising/image_grid_yale.tex
\begin{figure}
\captionsetup[sub]{font=scriptsize}
\begin{subfigure}[b]{.187\linewidth}
\includegraphics[width = \linewidth]{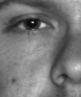} 
\caption{Original}
\end{subfigure}
\begin{subfigure}[b]{.187\linewidth}
\includegraphics[width = \linewidth]{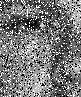}
\caption{Noisy}
\end{subfigure}
\begin{subfigure}[b]{.187\linewidth}
\includegraphics[width = \linewidth]{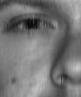} 
\caption{Cauchy ST}
\end{subfigure}\hfill
\begin{subfigure}[b]{.187\linewidth}
\includegraphics[width = \linewidth]{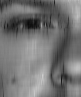} 
\caption{Welsh ST}
\end{subfigure}\hfill
\begin{subfigure}[b]{.187\linewidth}
\includegraphics[width = \linewidth]{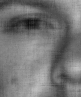}
\caption{TRPCA '16}
\end{subfigure}\hfill
\begin{subfigure}[b]{.187\linewidth}
\includegraphics[width = \linewidth]{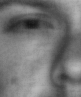} 
\caption{HORPCA-S}
\end{subfigure}\hfill
\begin{subfigure}[b]{.187\linewidth}
\includegraphics[width = \linewidth]{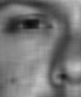} 
\caption{NCTRPCA}
\end{subfigure}\hfill
\begin{subfigure}[b]{.187\linewidth}
\includegraphics[width = \linewidth]{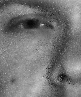} 
\caption{RNNDL}
\end{subfigure}\hfill
\begin{subfigure}[b]{.187\linewidth}
\includegraphics[width = \linewidth]{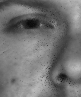} 
\caption{RPCA}
\end{subfigure}\hfill
\begin{subfigure}[b]{.187\linewidth}
\includegraphics[width = \linewidth]{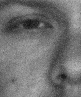}
\caption{LRRe}
\end{subfigure}
\begin{subfigure}[b]{.187\linewidth}
\includegraphics[width = \linewidth]{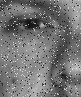}
\caption{LRRi}
\end{subfigure}
\begin{subfigure}[b]{.187\linewidth}
\includegraphics[width = \linewidth]{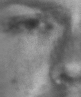}
\caption{DIP1}
\end{subfigure}
\begin{subfigure}[b]{.187\linewidth}
\includegraphics[width = \linewidth]{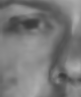}
\caption{DIP2}
\end{subfigure}
\begin{subfigure}[b]{.187\linewidth}
\includegraphics[width = \linewidth]{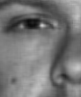}
\caption{RKCA}
\end{subfigure}
\caption{Results on the Yale benchmark with 30\% noise. TRPCA '14 removed.}
\label{fig:visual_yale_30}
\end{figure}

%% file: experimental_evaluation/image_denoising/face_denoising/table_yale_60.tex
\begin{table}[h]
    \centering
    \begin{tabular}{cccc}
        \includegraphics[width = .187 \linewidth]{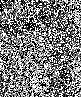} &
        \includegraphics[width = .187 \linewidth]{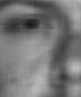} &
        \includegraphics[width = .187 \linewidth]{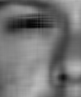} &
        \includegraphics[width = .187 \linewidth]{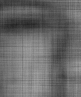}\\
         Noisy & RKCA & NC TRPCA & TRPCA '16 \Tstrut\Bstrut \\ \hline
         \textbf{PSNR} & 27.1490 & 22.8502 & 22.566 \Tstrut \\
         \textbf{FSIM} & 0.8913 & 0.8509 & 0.8427
    \end{tabular}
    \caption{Three best results on \textit{Yale} at 60\% noise.}
    \label{tab:top_3_yale_60}
\end{table}

%% file: experimental_evaluation/image_denoising/colour_denoising/colour_images.tex
\subsubsection{Color image denoising}

Our benchmark is the \textit{Facade} image \cite{x._chen_robust_2016}: the geometric nature of the building's front wall, and the strong correlation between the RGB bands indicate the data can be modeled by a low-rank 3-way tensor where each frontal slice is a color channel. The rich details and lighting make it interesting to assess fine reconstruction. The input dimension for DIP was set to 32, following the color image denoising examples.

\begin{figure}
    \centering
    \includegraphics[width=.9\linewidth]{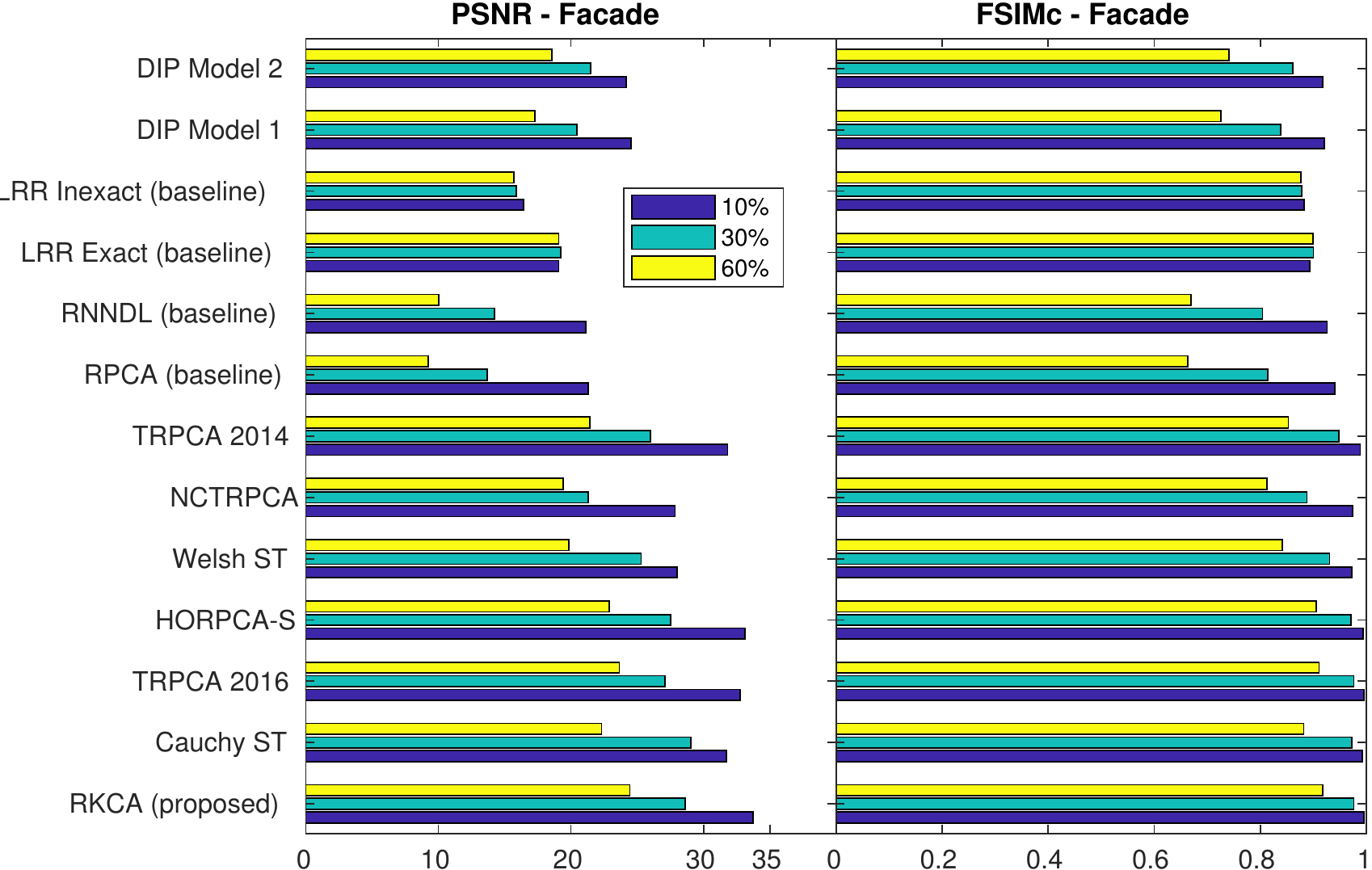}
    \caption{PSNR and FSIMc of all methods on the Facade benchmark at noise levels 10\%, 30\%, and 60\%.}
    \label{fig:perf_facade}
\end{figure}

At the 10\% noise level, RKCA attained the highest PSNR, and the highest FSIMc value. Most methods provided excellent reconstruction, in agreement with the high values of the metrics shown in Figure \ref{fig:perf_facade}. As in the previous benchmark, full results are in Appendix \ref{appendix:additional_experiments} (Table \ref{tab:full_results_facade} and figures). At the 30\% noise level, Cauchy ST exhibited the highest PSNR and RKCA the second highest, while TRPCA '16 and RKCA were tied for first place on the FSIMc metric. Details are provided in Figure \ref{fig:visual_facade_30}. Clear differences in reconstruction quality are visible, and are best seen on the fine details of the picture, such as the black iron ornaments, or the light coming through the window. Our method best preserved the dynamics of the lighting, and the sharpness of the details, and in the end provided the reconstruction visually closest to the original. Competing models tend to oversmooth the image, and to make the light dimmer; indicating substantial losses of high-frequency and dynamic information. RKCA appears to also provide the best color fidelity. Similar to the Yale-B experiment, the deep-learning models are able to handle the gross corruption to a certain extent, but suffer from more distortion than the robust factorizations.
\input{experimental_evaluation/image_denoising/colour_denoising/image_grid_facade_03}


\input{experimental_evaluation/image_denoising/colour_denoising/table_facade_60}

In the gross-corruption case, RKCA was the only method with TRPCA '16 and HORPCA-S to provide a reconstruction with distinguishable details, and did it best (Figure \ref{tab:top_3_facade_60}) while achieving the highest score on both quantitative metrics. 


%% file: experimental_evaluation/image_denoising/colour_denoising/image_grid_facade_03.tex
\begin{figure}
\captionsetup[sub]{font=scriptsize}
\begin{subfigure}[b]{.19\linewidth}
\includegraphics[width = \linewidth]{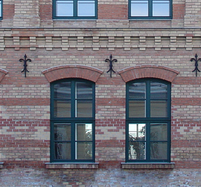} 
\caption{Original}
\end{subfigure}
\begin{subfigure}[b]{.19\linewidth}
\includegraphics[width = \linewidth]{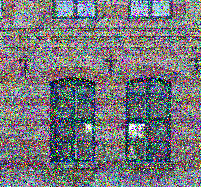}
\caption{Noisy}
\end{subfigure}
\begin{subfigure}[b]{.19\linewidth}
\includegraphics[width = \linewidth]{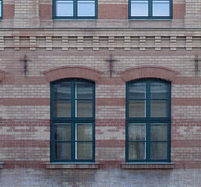} 
\caption{Cauchy ST}
\end{subfigure}\hfill
\begin{subfigure}[b]{.19\linewidth}
\includegraphics[width = \linewidth]{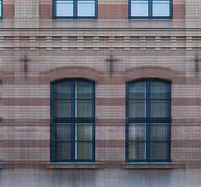} 
\caption{Welsh ST}
\end{subfigure}\hfill
\begin{subfigure}[b]{.19\linewidth}
\includegraphics[width = \linewidth]{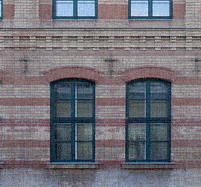} 
\caption{TRPCA '14}
\end{subfigure}\hfill
\begin{subfigure}[b]{.19\linewidth}
\includegraphics[width = \linewidth]{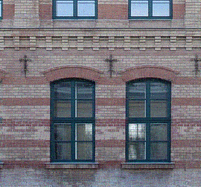}
\caption{TRPCA '16}
\end{subfigure}\hfill
\begin{subfigure}[b]{.19\linewidth}
\includegraphics[width = \linewidth]{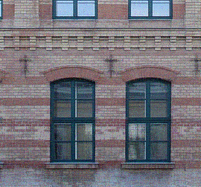} 
\caption{HORPCA-S}
\end{subfigure}\hfill
\begin{subfigure}[b]{.19\linewidth}
\includegraphics[width = \linewidth]{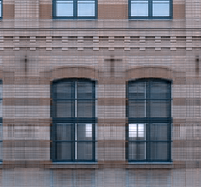} 
\caption{NCTRPCA}
\end{subfigure}\hfill
\begin{subfigure}[b]{.19\linewidth}
\includegraphics[width = \linewidth]{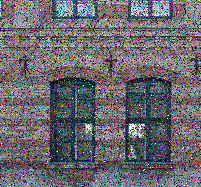} 
\caption{RPCA}
\end{subfigure}\hfill
\begin{subfigure}[b]{.19\linewidth}
\includegraphics[width = \linewidth]{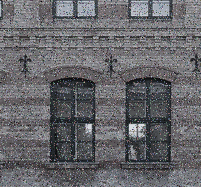} 
\caption{LRRe}
\end{subfigure}\hfill
\begin{subfigure}[b]{.19\linewidth}
\includegraphics[width = \linewidth]{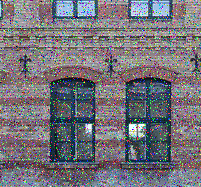} 
\caption{LRRi}
\end{subfigure}\hfill
\begin{subfigure}[b]{.19\linewidth}
\includegraphics[width = \linewidth]{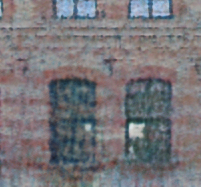} 
\caption{DIP1}
\end{subfigure}\hfill
\begin{subfigure}[b]{.19\linewidth}
\includegraphics[width = \linewidth]{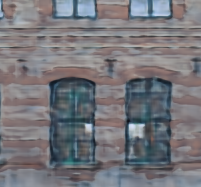} 
\caption{DIP2}
\end{subfigure}\hfill
\begin{subfigure}[b]{.19\linewidth}
\includegraphics[width = \linewidth]{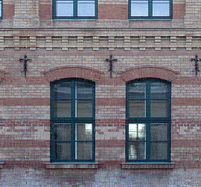}
\caption{RKCA}
\end{subfigure}
\begin{subfigure}[b]{.19\linewidth}
\caption*{}
\end{subfigure}
\caption{Results on the Facade benchmark with 30\% noise.}
\label{fig:visual_facade_30}
\end{figure}

%% file: experimental_evaluation/image_denoising/colour_denoising/table_facade_60.tex
\begin{table}[h]
    \centering
    \begin{tabular}{cccc}
        \includegraphics[width = .20 \linewidth]{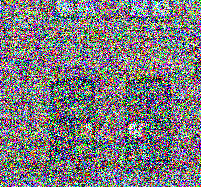} &
        \includegraphics[width = .20 \linewidth]{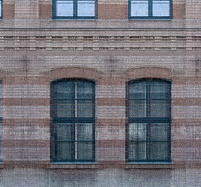} &
        \includegraphics[width = .20 \linewidth]{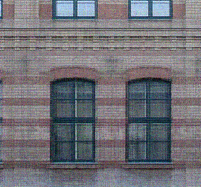} &
        \includegraphics[width = .20 \linewidth]{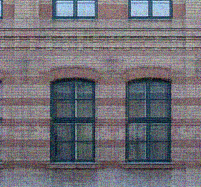}\\
         Noisy & RKCA & TRPCA '16 & HORPCA-S \Tstrut\Bstrut \\ \hline
         \textbf{PSNR} & 24.4491 & 23.6552 & 22.8811 \Tstrut \\
         \textbf{FSIMc} & 0.9179 & 0.9109 & 0.9060
    \end{tabular}
    \caption{Three best results on \textit{Facade} at 60\% noise.}
    \label{tab:top_3_facade_60}
\end{table}

%% file: experimental_evaluation/tensor_completion.tex
\subsection{Tensor completion}
\label{sec:tensor_completion_experiments}

To showcase the tensor completion capabilities of our algorithm, we implemented an LADMM method to solve problem (\ref{eq:constrained_pb_mv}) with $f(\tL) = \alpha \One{\tR} + \frac{1}{2} (\Frosq{\A} + \Frosq{\B})$. We provide comparison with one robust tensor completion model (HORPCA-S with missing values \cite{goldfarb_robust_2014}) and the matrix Robust PCA with missing values \cite{candes_robust_2011}.

\subsubsection{Yale-B with noise and missing values}
Our first experiment extends Section \ref{sec:face_denoising}: we investigate the case where apart from corruption, some values are missing. We generated the data by first introducing 30\% salt \& pepper noise, then removing 30\% of the pixels at random.

\begin{table}[h]
    \centering
    \begin{tabular}{cccc}
        \includegraphics[width = .187 \linewidth]{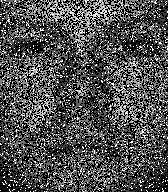} &
        \includegraphics[width = .187 \linewidth]{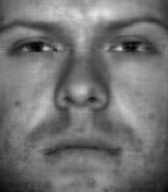} &
        \includegraphics[width = .187 \linewidth]{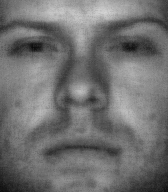} &
        \includegraphics[width = .187 \linewidth]{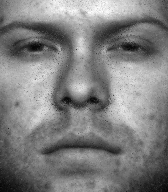}\\
         Noisy & RKCA & HORPCA-S & RPCA \Tstrut\Bstrut \\ \hline
         \textbf{PSNR} & 22.7069 & 22.3827 & 19.6597 \Tstrut \\
         \textbf{FSIM} & 0.9332 & 0.9187 & 0.8717
    \end{tabular}
    \caption{Reconstruction of the first face of Yale-B with 30\% salt \& pepper noise and 30\% missing values.}
    \label{tab:mv_salt_pepper}
\end{table}

As seen in Table \ref{tab:mv_salt_pepper}, RKCA markedly outperformed both other models in terms of FSIM, which translates into a more natural reconstruction. HORPCA-S achieved a similar PSNR but many details are lost, while RPCA removed much of the image's details and left some corruption.

\subsubsection{300 Faces in the Wild}
Our second experiment is on the completion of unwarped 3D faces with partial self-occlusions taken from the 300 faces in the wild challenge \cite{c._sagonas_300_2013,c._sagonas_semi-automatic_2013,sagonas_300_2016}.

\begin{table}[h]
    \centering
    \begin{tabular}{cccc}
        \includegraphics[width = .19 \linewidth]{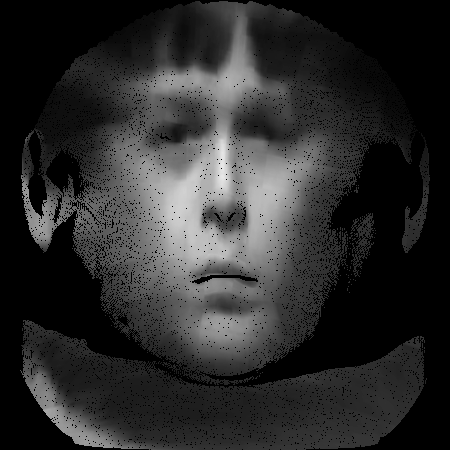} &
        \includegraphics[width = .19 \linewidth]{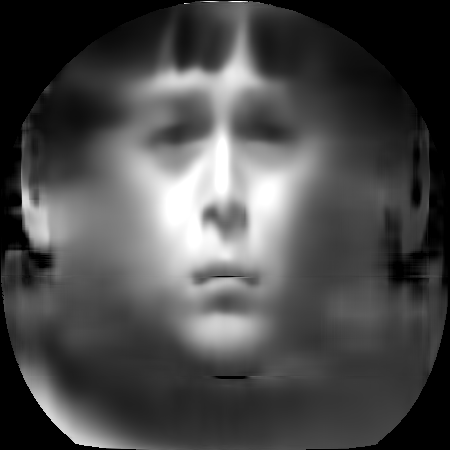} &
        \includegraphics[width = .19 \linewidth]{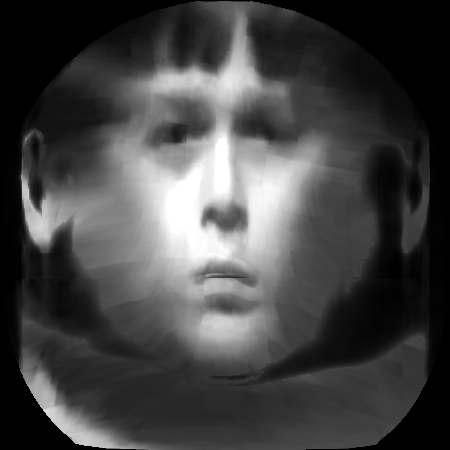} &
        \includegraphics[width = .19 \linewidth]{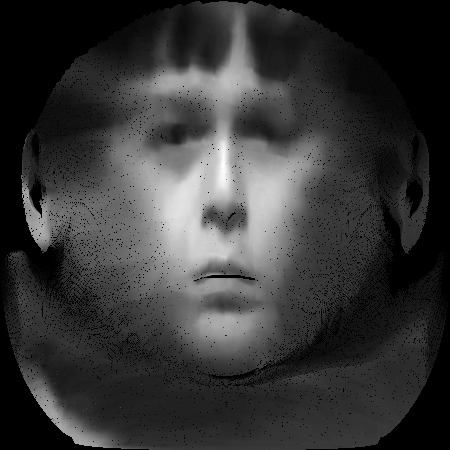}\\
        \includegraphics[width = .19 \linewidth]{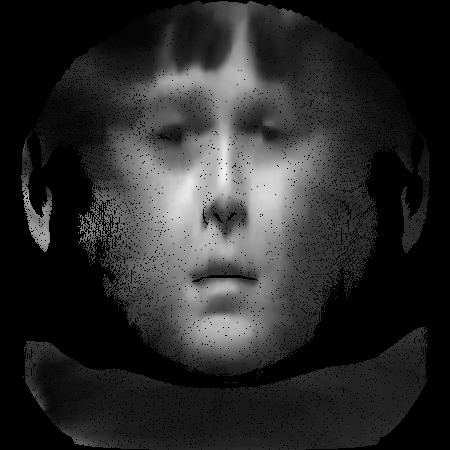} &
        \includegraphics[width = .19 \linewidth]{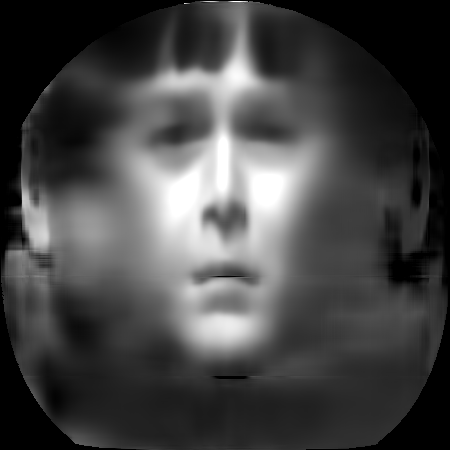} &
        \includegraphics[width = .19 \linewidth]{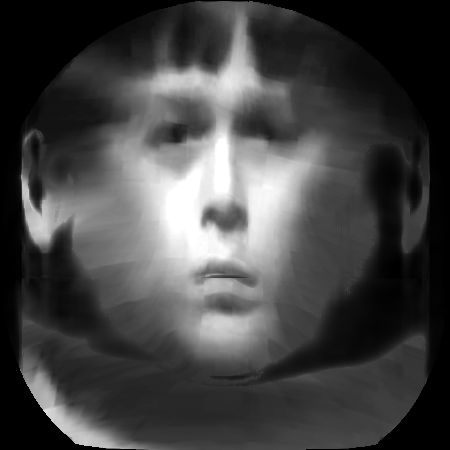} &
        \includegraphics[width = .19 \linewidth]{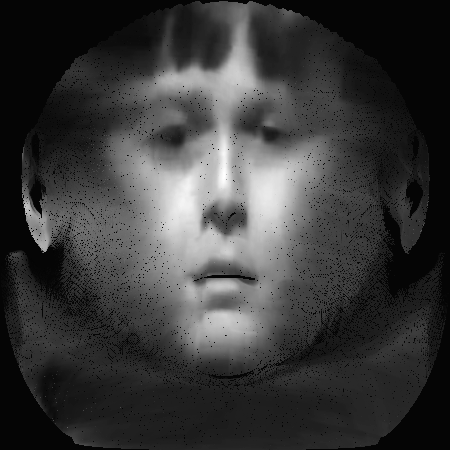}\\
        \includegraphics[width = .19 \linewidth]{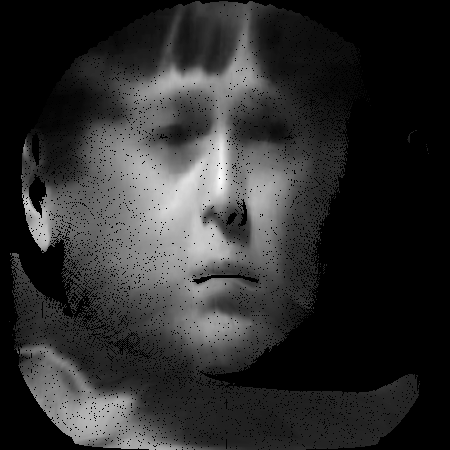} &
        \includegraphics[width = .19 \linewidth]{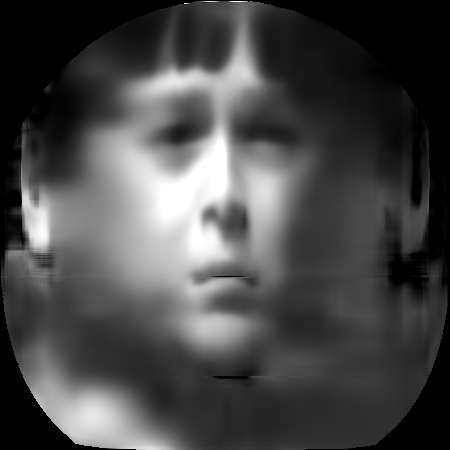} &
        \includegraphics[width = .19 \linewidth]{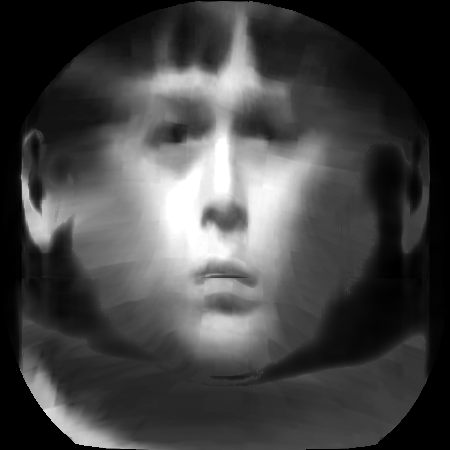} &
        \includegraphics[width = .19 \linewidth]{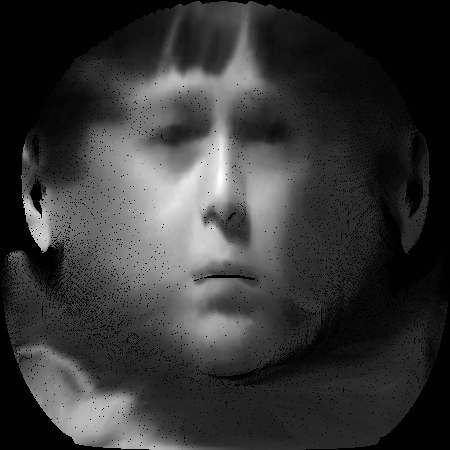}\\
        Original & RKCA & HORPCA & RPCA \Tstrut\Bstrut \\ \hline
    \end{tabular}
    \caption{Completion experiment on the 300W dataset.}
    \label{tab:comp_300w}
\end{table}

We present in Table \ref{tab:comp_300w} the occluded frames 1, 44, and 76 of a video of the dataset, and the completed frames obtained with RKCA, HORPCA-S, and RPCA. Since no corruption is present in the dataset, the $\lambda$ parameters were fixed to a high value (\num{1e4}) such that the algorithms behaved as tensor or matrix completion models. For RKCA, we bounded the rank of the reconstruction to 50 and performed grid-search on 5 values of the $\alpha$ parameter.

It is clear from Table \ref{tab:comp_300w} that RKCA provides the most completion. HORPCA-S is able to complete sparse missing values at the center of the image but is lacking on the self-occlusions at the right. Matrix RPCA failed to complete the frames on this benchmark.

%% file: appendices/root.tex
\input{appendices/shrinkage}

\input{appendices/updates_sub}

\input{appendices/proof_deg3}

\input{appendices/updates_lin}

\input{appendices/schatten_kron}

\input{appendices/parameter_tuning.tex}

\input{appendices/dip}

\input{appendices/additional_experiments}

\input{appendices/time_influence.tex}

\input{appendices/runtime.tex}

%% file: appendices/shrinkage.tex
\section{Proximal operators}

\subsection{Selective shrinkage operator}
\label{appendix:selective_shrinkage}
We prove the result for a general $D$-dimensional tensor $\tX$.
\begin{proof}
The proximal operator of $\One{\pi_\Omega(.)}$ satisfies:
\begin{equation}
P_{\One{\pi_\Omega(.)}}(\tX) = \argmin_{\tY} \One{\pi_\Omega(\tY)} + \frac{1}{2} \Frosq{\tX - \tY}
\end{equation}

By direct case analysis, if $i = (i_1, i_2, \ldots, i_N) \in \Omega$:
\begin{equation}
(P_{\One{\pi_\Omega(.)}})_i = \argmin_{\tY} \One{\tY_i} + \frac{1}{2} \Frosq{\tX_i - \tY_i}
\end{equation}
Which is the expression of the standard shrinkage operator $\shrink(\tX_i)$. If $i \in \bar{\Omega}$, then:
\begin{equation}
(P_{\One{\pi_\Omega(.)}})_i = \argmin_{\tY} \frac{1}{2} \Frosq{\tX_i - \tY_i} = \tX_i
\end{equation}
\end{proof}

\subsection{Schatten-$p$ norms}
\label{appendix:schatten_p_prox}

\begin{pstion}\label{prop:prox_operator_schatten_p}
Let $\A \in \RR^{m \times n}$ and $p > 0$, denote by $||.||_p$ the Schatten-$p$ norm and the $\ell_p$ norm, and by $\diag(.)$ the operator define by:
\begin{itemize}
\item $\diag(\X)$ is the main diagonal of the matrix $\X$\\
\item $\diag(\X)$ is the square matrix with the vector $x$ as main diagonal
\end{itemize}
Then if $\A = \U \mat{\Sigma} \trp{\V}$ the SVD of $\A$
\begin{equation}
\sprox{\lambda}{||.||_p}{\A} = \U \diag(\sprox{\lambda}{||.||_p}{\diag(\mat{\Sigma})}) \trp{\V}
\end{equation}
\end{pstion}

%% file: appendices/updates_sub.tex
\section{Updating the bases by substitution}
\label{appendix:substitution_updates_degree_three}

In the case of the degree three regularizers, we can use the substitution method to tackle the non-smoothness of the Frobenius norm and of the Nuclear norm, regardless of the method used for the core tensor $\tR$. We present below the derivation of the corresponding updates.

Let us introduce the auxiliary variables $\U$ and $\V$ for $\A$ and $\B$ respectively. The constrained problem becomes
\begin{align}\label{eq:constrained_pb_deg_three_sub}
    \begin{matrix*}[l]
    \min_{\A, \B, \U, \V, \ten{K}, \ten{R}, \ten{E}} & \lambda \One{\ten{R}}\Fro{\B} \Fro{\A} + \lambda \One{\tE}\\
    \st & \tX = \ten{K} \times_1 \U \times_2 \V + \tE\\
    	& \tR = \tK\\
        & \A = \U\\
        & \B = \V
    \end{matrix*}
\end{align}

And the Augmented Lagrangian can easily be formulated by introducing the Lagrange multipliers and the penalty parameters $\Y_{\U}$, $\Y_{\V}$, $\mu_{\U}$, and $\mu_{\V}$.

We present the derivations for $\A$ only as the ones for $\B$ are easily found by substituting the correct variables.

We have

\resizebox{.8\linewidth}{!}{
\begin{minipage}{\linewidth}
    \begin{align}
        \A^* & = \argmin_{\A} \left[ \alpha \Fro{\B} \One{\tR} \right] \Fro{\A} + \inner{\Y_{\U}}{\A - \U} + \frac{\mu_{\U}}{2} \Frosq{\A - \U} \\
            & = \argmin_{\A} \left[ \frac{\alpha}{\mu_{\U}} \Fro{\B} \One{\tR} \right] \Fro{\A} + \frac{1}{2} \Frosq{\U - \frac{1}{\mu_{\U}} \Y_{\U} - \A} \\
            & = \sprox{\frac{\alpha}{\mu_{\U}} \Fro{\B} \One{\tR}}{\Fro{.}}{\U - \frac{1}{\mu_{\U}} \Y_{\U}}
    \end{align}
\end{minipage}
}

The update of $\U$ is obtained by derivation
\resizebox{.9\linewidth}{!}{
\begin{minipage}{\linewidth}
    \begin{equation}
        \begin{split}
            \U^* = \argmin_{\U} \frac{\mu}{2} \sum_i \Frosq{\X_i - \U \K_i \trp{\V}- \E_i} + \\ \sum_i \inner{\LL_i}{\X_i - \U \K_i \trp{\V} - \E_i} + \frac{\mu_{\U}}{2} \Frosq{\A - \U} + \inner{\Y_{\U}}{\A - \U}
        \end{split}
    \end{equation}
\end{minipage}
}

Taking the derivative with respect to $\U$ and setting to $0$ we find the Sylvester's equation for $\U^*$
\begin{equation}
\begin{split}
\U^* ( - \frac{\mu}{\mu_{\U}} \sum_i \K_i \trp{\V} \V \trp{\K_i}) - \U^* + \\ \frac{1}{\mu_{\U}} \left[ \mu_{\U} \A + \Y_{\U} + \mu \sum_i (\X_i - \E_i + \frac{1}{\mu} \LL_i) \V \trp{\K_i} \right] = 0
\end{split}
\end{equation}

%% file: appendices/proof_deg3.tex
\section{Proof}
\label{appendix:proof}

\subsection{Definitions on integer sequences}

Let us first define in a clear manner the concepts of periodic and group-periodic integer sequences.
\begin{defn}
\label{def:periodic_sequence}
An integer sequence $a_1, a_2, \ldots, a_n$ is said to be periodic with period $p$ (i.e., $p$-periodic) if:
\begin{equation}
\forall i \geq 1, a_i = a_{i+p}
\end{equation}
\end{defn}

We will need to manipulate sequences of periodic patterns of repeating number, for this purpose we define a group periodic sequence as follows.
\begin{defn}
\label{def:group_periodic_sequence}
We will say an integer sequence $a_1, a_2, \ldots, a_n$ is group periodic with group length $l$ and period $p$, or $(l, p)$-group periodic, if:
\begin{equation}
\forall i \geq 1, a_i = a_{i+pl}
\end{equation}
And:
\begin{equation}
\forall k \geq 0, \, \forall lk + 1 \leq i, j \leq l(k+1), \, a_i = a_j
\end{equation}

For instance, $1 \, 1 \, 0 \, 0 \, 2 \, 2 \, 1 \, 1 \, 0 \, 0 \, 2 \, 2$ is (2, 3)-group periodic.
\end{defn}


Since the sum involves all $N r^2$ elements of $\tR$ it is clear such an unfolding must too. Since there are respectively $r$, $r$, and $N$ columns in $\A$, $\B$, and $\II_N$ and $\max (r, N, Nr^2) = Nr^2$, we have $s = N r^2$.

We define a bijection from $\NN^3$ to $\NN$ such that each triple index $(i, j, k)$ is mapped to a unique positive integer $l$ by setting, in the general $M$-dimensional Tucker case:
\begin{equation}\label{eq:bijection_indices}
l = 1 + \sum_{k=1}^{M}{\left(i_k - 1\right)J_k}\quad\text{and}\quad
L_k = \prod_{m=1}^{k-1}{D_m},
\end{equation}
where $D_m$ denotes the dimension along mode $m$. In our case we have $D_1 = r, \, D_2 = r,\, D_3 = N$. This re-indexing corresponds to the standard definition of the $\vvec$ operator where the columns of each frontal slice of the 3-way tensor are stacked in order.

We now describe the new factors of our decomposition:


The construction of $\tilde{\A}$ and $\tilde{\B}$ must be consistent with the ordering of the elements in the re-indexing. For the reader's convenience, we visualize in Table \ref{table:mapping_correspondence} the correspondence in the simple case where $r = 2$ and $N = 3$. 
\begin{table}[h]
\centering
\renewcommand{\arraystretch}{1.2}
\begin{tabular}{|c|c|c|c|c|c|c|c|c|c|c|c|c|} \hline
$l$ & 1 & 2 & 3 & 4 & 5 & 6 & 7 & 8 & 9 & 10 & 11 & 12\\ \hline
$i$ & 1 & 2 & 1 & 2 & 1 & 2 & 1 & 2 & 1 & 2 & 1 & 2\\  \hline
$j$ & \multicolumn{2}{c|}{1} & \multicolumn{2}{c|}{2} & \multicolumn{2}{c|}{1} & \multicolumn{2}{c|}{2} & \multicolumn{2}{c|}{1} & \multicolumn{2}{c|}{2} \\ \hline
$k$ & \multicolumn{4}{c|}{1} & \multicolumn{4}{c|}{2} & \multicolumn{4}{c|}{3}\\ \hline
\end{tabular}
\caption{Illustration of our re-indexing in a small-dimensional case.}
\label{table:mapping_correspondence}
\end{table}
To map back $l$ to $i$, $j$, $k$ we observe the 3 sequences are (group) periodic. In fact, the sequence corresponding to the innermost sum of ($\ref{eq:explicit_sums_factorization}$) is $(r^2, N)$-group periodic. We find that:
\begin{equation}
k =  (\ceil{\frac{l}{r^2}} - 1) \, \text{mod} \, N + 1
\end{equation}
The sequence for $j$ is $(r, r)$-group periodic and the sequence for $i$ is $r$-periodic, so we have:
\begin{equation}
j = (\ceil{\frac{l}{r}} - 1) \, \text{mod} \, r + 1
\end{equation}

And:
\begin{equation}
i = (l - 1) \, \text{mod} \, r + 1
\end{equation}

The augmented factors all have dimension $Nr^2$ over their last mode. For $\tilde{\A}$ we duplicate each column $Nr^2 - r$ times by concatenating $Nr$ copies of $\A$ on the column dimension. For $\tilde{\B}$, each column is copied $r-1$ times before stacking the next column, and the resulting matrix is concatenated $N$ times. The $\tilde{\vv{c}}_l$ are defined by:
\begin{equation}
(\tilde{\vv{c}}_l)_i = \delta_{i, (\ceil{\frac{l}{r^2}} - 1) \, \text{mod} \, N + 1}
\end{equation}

The elemental mapping of our factorization is therefore:
\begin{equation}
\phi(\sigma, \tilde{\vv{a}}, \tilde{\vv{b}}) = \sigma \cdot \tilde{\vv{a}} \otimes \tilde{\vv{b}} \otimes \bm{\delta_{., (\ceil{\frac{l}{r^2}} - 1) \, \text{mod} \, N + 1}}
\end{equation}

The factorization function is directly obtained from the elemental mapping.

%% file: appendices/updates_lin.tex
\section{Linearized ADMM for scalability}
\label{appendix:ladmm}

The substitution method used in Section \ref{sec:kdrsdl} is effective in practice but comes with an additional cost that limits its scalability. Notably, the cost of solving a Stein equation for each frontal slice of $\tR$ cannot be neglected. Additionally, the added parameters can make tuning difficult. Linearization provides an alternative approach that can scale better to large dimensions and large datasets, in this appendix, we show how it can be applied to our models.

We give detailed derivations of the linearized updates for $\A$, $\B$, and $\tR$. The development for $\tR$ can directly be applied to solving problem (\ref{eq:constrained_pb}) and its variants with a simple change in the shrinkage factor.

\subsection{Overview}

We briefly remind the reader of the linearization method. Provided an unconstrained optimization problem:
\begin{equation}\label{}
\min_\vv{x} f(\vv{x}) + g(\vv{x})
\end{equation}
Where $f$ is smooth and $g$ is non-smooth and $\vv{x} \in \RR^d$. We replace $f$ by the following quadratic approximation around $\vv{y}$:
\begin{equation}
q_l(\vv{x}, \vv{y}) = f(\vv{y}) + \inner{\grad f(\vv{y})}{\vv{y} - \vv{x}} + \frac{l}{2} \Twosq{\vv{y} - \vv{x}}
\end{equation}

When $\grad f$ is Lipschitz continuous with Lipschitz constant $L$ and $l \geq L$ the quadratic approximation is an upper bound of $f$. If $g$ has proximal operator $\prox{g}{.}$ then:
\begin{align}
& g(\vv{x}) + q_l(\vv{x}, \vv{y}) \\
								& = g(\vv{y}) + f(\vv{y}) + \inner{\grad f(\vv{y})}{\vv{y} - \vv{x}} + \frac{l}{2} \Twosq{\vv{y} - \vv{x}}\\
								& = g(\vv{y}) + \frac{l}{2} \Twosq{\vv{x} - \left( \vv{y} - \frac{1}{l} \grad f(\vv{y}) \right)} + f(\vv{y})
\end{align}
And $\argmin_x g(\vv{x}) + q_l(\vv{x}, \vv{y}) = \prox{g}{\vv{y} - \frac{1}{l} \grad f(\vv{y})}$. In an iterative algorithm, we choose $\vv{x} = \vv{x}_{k+1}$ and $\vv{y} = \vv{x}_k$.

We start from the augmented Lagrangian of problem (\ref{eq:constrained_pb_deg_three}) with no auxiliary variables:

\begin{equation}\label{eq:al_no_sub}
\begin{split}
\mathcal{L}(\A, \B, \tR, \tE, \LL, \mu) = \alpha \Fro{\A} \Fro{\B} \One{\tR} + \\
\lambda \One{\tE} + \inner{\LL}{\tX - \tR \times_1 \A \times_2 \B - \tE} + \\
\frac{\mu}{2} \Frosq{\tX - \tR \times_1 \A \times_2 \B - \tE}
\end{split}
\end{equation}

\subsection{Updating the core $\tR$}

From the augmented Lagrangian of problem (\ref{eq:constrained_pb_deg_three}), updating $\tR$ requires solving the minimization problem:
\begin{equation}\label{eq:min_lagrangian_R}
\begin{split}
\min_{\tR} \alpha \Fro{\A} \Fro{\B} \One{\tR} + \\ \frac{\mu}{2} \Frosq{\tX - \tR \times_1 \A \times_2 \B - \tE + \frac{1}{\mu} \LL}
\end{split}
\end{equation}

Let $\bm{\Delta} = \tX - \tE + \frac{1}{\mu} \LL$ and $\bm{\delta} = \vvec(\Delta)$, $\vv{r} = \vvec(\tR)$. From the properties of the Kronecker product, $\vvec(\tR \times_1 \A \times_2 \B) = (\II_N \otimes \B \otimes \A) \vvec(\vv{r})$. For the sake of brevity we let $\bm{\Gamma} = \II_N \otimes \B \otimes \A$. Thanks to the separability of the $\ell_1$-norm penalty, using these notations (\ref{eq:min_lagrangian_R}) is equivalent to the following vector minimization problem:
\begin{equation}
\min_{\vv{r}} \frac{\alpha}{\mu} \Fro{\A} \Fro{\B} \One{\vv{r}} + \frac{1}{2} \Twosq{\bm{\Gamma} \vv{r} - \bm{\delta}}
\end{equation}

Vector calculus gives the following gradient for the quadratic part:
\begin{equation}
\grad_{\vv{r}}(\vv{r}) = \frac{\partial}{\partial \vv{r}} \frac{1}{2} \Twosq{\bm{\Gamma} \vv{r} - \bm{\delta}} = \trp{\bm{\Gamma}}(\bm{\Gamma} \vv{r} - \bm{\delta})
\end{equation}

The computation of the Lipschitz constant is then straightforward:
\begin{align}
\Two{\grad_{\vv{r}}(\vv{r}_1) - \grad_{\vv{r}}(\vv{r}_2)} & = \Two{\trp{\bm{\Gamma}}\bm{\Gamma}(\vv{r}_1 - \vv{r}_2)}\\
										  & \leq |||\trp{\bm{\Gamma}}\bm{\Gamma}||| \cdot \Two{\vv{r}_1 - \vv{r}_2}
\end{align}

Where $|||.|||$ denotes the induced norm, i.e., the largest singular value. We again make use of the properties of the Kronecker product to find:
\begin{align}
|||\trp{\bm{\Gamma}}\bm{\Gamma}||| & = |||\trp{(\II_N \otimes \B \otimes \A)} (\II_N \otimes \B \otimes \A)||| \\
& = |||\trp{\II_N}\II_N \otimes \Bt\B \otimes \At\A||| \\
& = |||\Bt\B||| \cdot |||\At\A||| \\\
& = \sigma_{\mathrm{max}}^2(\B) \cdot \sigma_{\mathrm{max}}^2(\A)
\end{align}

In this case, we recommend direct calculation of the constant instead of using a backtracking procedure because the matrices $\A$ and $\B$ will often be of small size (if working on small input data or if providing a small upper bound $r$ on the rank) and the computation can be performed by partial SVD in all cases.

Finally, going back to tensor form the updated $\tR^{t+1}$ is:
\begin{equation}\label{eq:ladmm_upd_R}
\shrink_{\alpha'} \left( \tR^t - \frac{1}{L_{\tR}} [(\tR^t \times_1 \A \times_2 \B) - \bm{\Delta})]\times_1 \At \times_2 \Bt \right)
\end{equation}
Where $L_{\tR} \geq \sigma_{\mathrm{max}}^2(\B) \cdot \sigma_{\mathrm{max}}^2(\A)$ and $\alpha' = \alpha$ for problem (\ref{eq:constrained_pb}) and $\alpha' = \alpha \Fro{\A} \Fro{\B}$ for problem (\ref{eq:constrained_pb_deg_three}).

\subsection{For $\A$ and $\B$}

The sub-problem for $\A$ is:
\begin{equation}\label{eq:min_lagrangian_A}
\begin{split}
\min_{\A} \alpha \Fro{\A} \Fro{\B} \One{\tR} + \\ \frac{\mu}{2} \Frosq{\tX - \tR \times_1 \A \times_2 \B - \tE + \frac{1}{\mu} \LL}
\end{split}
\end{equation}
Our problem is separable in the frontal slices of $\tX$ and therefore in the frontal slices of the components of the factorization and of the Lagrange multipliers. Using the same $\bm{\Delta}$ notation from the update of $\tR$ and denoting $\C_i = \Ri \Bt$, we solve:

\begin{align}
	\A^* = \argmin_{\A} \left[ \frac{\alpha}{\mu} \right] \Fro{\A} + \frac{1}{2} \sum_i \Frosq{\bm{\Delta}_i - \A \C_i }
\end{align}

We compute $\grad_{\A}(\A) = \frac{\partial}{\partial \A} \frac{1}{2} \sum_i \Frosq{\bm{\Delta}_i - \A \C_i } = \sum_i (\A \C_i - \bm{\Delta}_i) \trp{\C}_i$.

The Lipschitz constant is:
\begin{align}
& \Fro{\grad_{\A}(\A_1) - \grad_{\A}(\A_2)} \\
& = \Fro{\sum_i (\A_1 \C_i - \bm{\Delta}_i) \trp{\C}_i - \sum_i (\A_2 \C_i - \bm{\Delta}_i) \trp{\C}_i} \\
& = \Fro{\sum_i (\A_1 - \A_2) \C_i \trp{\C}_i } \\
& = \Fro{(\A_1 - \A_2) \sum_i \C_i \trp{\C}_i } \\
& \leq \Fro{\A_1 - \A_2} \Fro{\sum_i \C_i \trp{\C}_i }
\end{align}
From the sub-multiplicativity of Schatten norms.

\noindent Let $a = \frac{1}{\mu L_{\A}}( \alpha \Fro{\B^{t}} \One{\tR^{t}} )$, we have:
\begin{equation}
\A^{t+1} = \sprox{a}{\Fro{.}}{\A^t - \frac{1}{L_{\A} } \sum_i (\A^t \C^t_i - \bm{\Delta}^t_i)\trp{{\C^t}_i}}
\end{equation}

\noindent $L_{\A} \geq \Fro{\sum_i \C_i \trp{\C}_i } = \sum_i \Fro{ \C_i \trp{\C}_i } = \sum_i \Fro{\R_i \Bt \B \trp{\R_i}}$.

\vspace{1em}

The sub-problem for $\B$ is:
\begin{equation}\label{eq:min_lagrangian_B}
\begin{split}
\min_{\B} \alpha \Fro{\A} \Fro{\B} \One{\tR} + \\ \frac{\mu}{2} \Frosq{\tX - \tR \times_1 \A \times_2 \B - \tE + \frac{1}{\mu} \LL}
\end{split}
\end{equation}

Similarly, using the same $\bm{\Delta}$ notation from the update of $\tR$ and letting $\G^t_i = \A^{t+1}\R_i \; \forall i$ ($\G_i = \A \Ri$) and $b = \frac{1}{\mu L_{\B}}( \alpha \Fro{\A^{t+1}} \One{\tR^{t}} )$ we find:

\begin{align}
\B^* = \argmin_{\B} \left[ \frac{\alpha}{\mu} \right] \Fro{\B} + \frac{1}{2} \sum_i \Frosq{\bm{\Delta}_i - \G_i \Bt }
\end{align}

We compute $\grad_{\B}(\B) = \frac{\partial}{\partial \B} \frac{1}{2} \sum_i \Frosq{\bm{\Delta}_i - \G_i \Bt } = \sum_i (\B \trp{\G_i} - \trp{\bm{\Delta}_i}) \G_i$.

And similar to $\A$ the Lipschitz constant:
\begin{align}
& \Fro{\grad_{\B}(\B_1) - \grad_{\B}(\B_2)} \\
& = \Fro{\sum_i (\B_1 \trp{\G_i} - \trp{\bm{\Delta}_i}) \G_i - \sum_i (\B_2 \trp{\G_i} - \trp{\bm{\Delta}_i}) \G_i} \\
& = \Fro{\sum_i (\B_1 - \B_2) \trp{\G_i} \G_i } \\
& = \Fro{(\B_1 - \B_2) \sum_i \trp{\G_i} \G_i } \\
& \leq \Fro{\A_1 - \A_2} \Fro{\sum_i \trp{\G_i} \G_i }
\end{align}

And:

\begin{equation}
\B^{t+1} = \sprox{b}{\Fro{.}}{\B^t - \frac{1}{L_{\B} } \sum_i \trp{(\G^t_i \trp{{\B^t}} - \bm{\Delta}^t_i)}\G^t_i}
\end{equation}

\noindent $L_{\B} \geq \Fro{\sum_i \trp{\G_i} \G_i } = \sum_i \Fro{\trp{\R_i} \At \A \R_i}$.

Remarking that $\ten{C} = \tR \times_2 \B$ then $L_{\A} = \Fro{\C_{[2]} \trp{\C_{[2]}}}$ where $\C_{[2]}$ is the mode-$2$ matricization of $\ten{C}$. The $i^{th}$ line of $\C_{[2]}$ is the concatenation of the $i^{th}$ columns of all the frontal slices of $\ten{C}$ so $\C_{[2]} \trp{\C_{[2]}}$, the matrix of the dot products of the lines of $\C_{[2]}$, captures interactions across the $i^{th}$ rows of all the input images. Conversely, $\tG = \mathrm{permute}(\tR \times_1 \A, [2,1,3])$, so $L_{\B} = \Fro{\G_{[2]} \trp{\G_{[2]}}}$ where $\G_{[2]}$ is the mode-$2$ matricization of $\tG$. The $i^{th}$ line of $\G_{[2]}$ is the concatenation of the $i^{th}$ lines of all the frontal slices of $\tG$ so $\G_{[2]} \trp{\G_{[2]}}$, the matrix of the dot products of the lines of $\G_{[2]}$, captures interactions across the $i^{th}$ columns of all the input images.

%% file: appendices/schatten_kron.tex
\section{Schatten norms and the Kronecker product}

In this Section we prove the identity used in Section 2.2 on the Schatten norm of a Kronecker product.

Let us first remind the reader of the definition of the Schatten-$p$ norm:

\begin{defn}
Let $\A$ a real-valued matrix, $\A \in \RR^{m \times n}$. The Schatten-$p$ norm of $\A$ is defined as:
\[
||\A||_p = (\sum_i s_i^p)^{1/p}
\]
Where $s_i$ is the $i^{th}$ singular value of $\A$.
\end{defn}

We argued the result stems from the compatibility of the Kronecker product with the singular value decomposition, that is:

\begin{pstion}
Let $\A = \U_A \bSigma_A \trp{\V_A}, \; \B = \U_B \bSigma_B \trp{\V_B}$ two real-valued matrices given by their SVD, then:
\[
\A \otimes \B = (\U_A \otimes \U_B)(\bSigma_A \otimes \bSigma_B)(\trp{\V_A} \otimes \trp{\V_B})
\]
\end{pstion}
\begin{proof}
See \cite{laub_matrix_2004}.
\end{proof}

The identity we used is formally expressed in Theorem \ref{thm:kron_schatten}, of which we give the proof for completeness.

\setcounter{thm}{0}
\begin{thm}\label{thm:kron_schatten}
Let $\A \in \RR^{m \times n}$ and $\B \in \RR^{p \times q}$, then:
\[
\forall p > 0, ||\A \otimes \B||_p = ||\A||_p ||\B||_p
\]
Where $||.||_p$ denotes the Schatten-$p$ norm.
\end{thm}
\begin{proof}
From Proposition 1, the singular values of $\A \otimes \B$ are the $\sigma_{A, i}\sigma_{B,j}$ so:
\begin{align*}
    ||\A \otimes \B||_p & = (\sum_{i, j} (\sigma_{A, i} \sigma_{B,j})^p)^{1/p}\\
    & = (\sum_{i, j} \sigma_{A, i}^p \sigma_{B,j}^p)^{1/p}\\
    & = (\sum_{i} \sigma_{A, i}^p \sum_j \sigma_{B,j}^p)^{1/p}\\
    & = (\sum_{i} \sigma_{A, i}^p)^{1/p} (\sum_j \sigma_{B,j}^p)^{1/p}\\
    & = ||\A||_p ||\B||_p
\end{align*}

\end{proof}

%% file: appendices/parameter_tuning.tex
\section{Further implementation details}

\subsection{Parameter tuning and adaptive ADMM}
\label{appendix:param_tuning}

A key concern with the practical implementation of ADMM and LADMM algorithms is parameter tuning and initialization, and is exacerbated by the introduction of additional variables and constraints. A constraint of the form $\tX = \tY$ translates to additional terms $\inner{\LL, \tX - \tY} + \frac{\mu}{2} \Frosq{\tX - \tY}$ in the augmented Lagrangian of the problem, where $\LL$ is a tensor of Lagrange multipliers of the dimension of $\tX$, and $\mu$ is a penalty parameter. In standard (L)ADMM, $\mu$ is initialized and updated such that the sequence $(\mu^t)_{t \geq 1}$ is bounded and non-decreasing, generally through a simple update $\mu^{t+1} = \rho \mu^t$. The initial value $\mu^0$ and the value of $\rho$ directly influence the speed of convergence and the quality of the reconstructed images. As more constraints are added, more penalty parameters are introduced and must be tuned and updated properly, possibly independently from one another. Although the literature on ADMM with adaptive penalty updates \cite{lin_linearized_2011,xu_adaptive_2017,xu_admm_2017} suggests adaptive updates for the $\rho$ variables, these solutions also come with their added complexity and possibly with more parameters to tune. In the case of LADMM, the practitioner must also ensure an adequate choice of value is made to upper bound the Lipschitz constants of the gradients in each update.

Due to the sensitivity of the ADMM and LADMM algorithms to parameter tuning, we found that the algorithm that provided the most benefits in practice and the best reconstructions was the one that was the simplest to tune. This explains our choice of Algorithm \ref{alg:RKCA} and its LADMM variant for the experiments.

\subsection{Convergence and initialization}
\label{appendix:initialization}

Problems (\ref{eq:constrained_pb}) and (\ref{eq:constrained_pb_deg_three}) are non-convex, therefore global convergence is a priori not guaranteed. Recent work \cite{hong_convergence_2016, wang_global_2018} studies the convergence of ADMM for non-convex and possibly non-smooth objective functions with linear constraints. Here, the constraints are not linear. We proposed problem (\ref{eq:constrained_pb}) based on \cite{haeffele_structured_2014, haeffele_global_2015} so global convergence could theoretically be attained with a local descent algorithm. However, problem (\ref{eq:constrained_pb}) doesn't offer any guarantees and in both cases the (L)ADMM scheme employed does not necessarily converge to a local minimum. In this section, we provide experimental results for Algorithm \ref{alg:RKCA} and discuss the initialization strategy implemented for all the variants.

We propose a simple initialization scheme for Algorithm \ref{alg:RKCA} and its variants adapted from \cite{lin_augmented_2010}. We initialize the bases $\A$ and $\B$ and the core $\tR$ by performing SVD on each observation $\XXi = \mat{U}_i \mat{S}_i \trp{\mat{V}_i}$. We set $\Ri = \mat{S}_i$, $\A = \frac{1}{N} \sum_i \mat{U}_i$ and $\B = \frac{1}{N} \sum_i \mat{V}_i$. To initialize the dual-variables for the constraint $\XXi - \A \Ri \Bt - \Ei = \mat{0}$, we take $\mu^0 = \frac{\eta N}{\sum_i \Fro{\XXi}}$ where $\eta$ is a scaling coefficient, chosen in practice to be $\eta = 1.25$ as in \cite{lin_augmented_2010}. We chose $\mu_{\ten{K}}^0 = \frac{\eta N}{\sum_i \Fro{\Ri}}$ and similarly for $\A$ and $\B$ when applicable. These correspond to averaging the initial values for each individual slice and its corresponding constraint. Our convergence criterion corresponds to primal-feasibility of problem (\ref{eq:constrained_pb_rpca_upper_bound}) or (\ref{eq:constrained_pb_deg_three}), and is given by $\max(\mathrm{err}_{rec}, \mathrm{err}_{\tR}, \mathrm{err}_{\A}, \mathrm{err}_{\B}) \leq \epsilon$ where $\mathrm{err}_{rec} = \max_i \frac{\Frosq{\XXi - \A \Ri \Bt - \Ei}}{\Frosq{\XXi}}$ and $\mathrm{err}_{\tR} = \max_i \frac{\Frosq{\Ri - \Ki}}{\Frosq{\Ri}}$, $\mathrm{err}_{\A} = \frac{\Frosq{\A - \U}}{\Frosq{\A}}$, and $\frac{\Frosq{\B - \V}}{\Frosq{\B}}$. For the LADMM versions we only have $\max(\mathrm{err}_{rec})$. Empirically, we obtained systematic convergence of Algorithm \ref{alg:RKCA} to a good solution, and a linear convergence rate, as shown in Figure \ref{fig:sample_convergence}. Similar results were found for the ADMM with substitution algorithm for problem (\ref{eq:constrained_pb_deg_three}) and for the LADMM variants of both problems.

%% file: appendices/dip.tex
\section{Deep Image Prior}
\label{appendix:dip}
We tried two different network architectures presented in the paper and in the denoising example code provided by the authors\footnote{https://github.com/DmitryUlyanov/deep-image-prior/blob/661c4f45f2416604b3afb43afc5003213e9865e7/denoising.ipynb}, abbreviated as DIP 1 and DIP 2. In both DIP implementations, we used an auto-encoder structure with skip connections. 

We used as input 128 random noise input channels for Yale database with an output of the 64 images of the person in the Yale database. That way, all the information from all the facial images was used. 

For Facade we used 64 input channels and 3 outputs. In DIP 1 the number channels down was $5\times 128$, while in DIP 2 $[8,16,32,64,128]$. 

In DIP 1 the numbers of channels up were $5\times 128$, while in DIP $[128, 64, 32, 16, 8]$. In DIP 1 the number of channels per skip connection was 4, while in DIP 2 was $[0,0,4,4]$. 

Bilinear upsampling was used in both cases with reflection padding. Adam optimiser was used with learning rate of 0.01. The structure of the input noise was the one described in the paper \cite{ulyanov_deep_2018}.

%% file: appendices/additional_experiments.tex
\section{Additional experimental results}
\label{appendix:additional_experiments}

In this Section we provide additional information on the results reported in the paper, as well as supporting material.

\subsection{Additional denoising results on Yale}

This section presents supplementary denoising results. In addition to the first illumination, we show the reconstruction obtained on the third illumination, and include the output of all algorithms.

\begin{table*}[h!]
    \centering
    \begin{tabular}{|c|c|c|c|c|c|c|}\hline
  & \multicolumn{2}{c|}{10.00\%} & \multicolumn{2}{c|}{30.00\%} & \multicolumn{2}{c|}{60.00\%} \\ \cline{2-7}
  & PSNR & FSIM & PSNR & FSIM & PSNR & FSIM \\ \hline
  RKCA & 36.1887 & \textbf{0.9878} & \textbf{32.7931} & \textbf{0.9509} & \textbf{27.1490} & \textbf{0.8913}\\
  Cauchy ST & \textbf{36.8013} & 0.9826 & 30.7079 & 0.9424 & 20.8269 & 0.8131\\
  TRPCA 2016 & 35.1450 & 0.9779 & 28.6409 & 0.9299 & 22.5660 & 0.8427\\
  HORPCA-S & 34.3085 & 0.9698 & 28.6301 & 0.9179 & 22.1410 & 0.8115\\
  Welsh ST & 31.6702 & 0.9539 & 25.8976 & 0.9008 & 16.5715 & 0.7335\\
  NCTRPCA & 23.6362 & 0.8839 & 22.5348 & 0.8816 & 22.8502 & 0.8509\\
  TRPCA 2014 & 19.1429 & 0.7902 & 17.1359 & 0.7627 & 15.0052 & 0.7201\\
  RPCA (baseline) & 30.1698 & 0.9440 & 25.7667 & 0.8954 & 17.9296 & 0.7234\\
  RNNDL (baseline) & 20.0746 & 0.8503 & 17.1977 & 0.8087 & 12.9945 & 0.6391\\
  LRR Exact (baseline) & 19.9762 & 0.8058 & 19.5246 & 0.8044 & 18.3642 & 0.7787\\
  LRR Inexact (baseline) & 20.1770 & 0.7565 & 17.4091 & 0.6855 & 15.8379 & 0.6576\\
  DIP Model 1 & 27.2166 & 0.9016 & 22.4471 & 0.8555 & 17.8652 & 0.7661\\
  DIP Model 2 & 25.1907 & 0.8796 & 21.4762 & 0.8588 & 16.3906 & 0.8180\\  
\hline
    \end{tabular}
    \caption{Numerical values of the image quality metrics on the Yale-B benchmark for the three noise levels. Best performance in bold.}
    \label{tab:full_results_yale}
\end{table*}

\input{appendices/experimental_evaluation/image_denoising/face_denoising/image_grid_yale_10_1}

\input{appendices/experimental_evaluation/image_denoising/face_denoising/image_grid_yale_10_3}

\input{appendices/experimental_evaluation/image_denoising/face_denoising/image_grid_yale_30_1}

\input{appendices/experimental_evaluation/image_denoising/face_denoising/image_grid_yale_30_3}

\input{appendices/experimental_evaluation/image_denoising/face_denoising/image_grid_yale_60_1}

\input{appendices/experimental_evaluation/image_denoising/face_denoising/image_grid_yale_60_3}

\clearpage
\subsection{Additional denoising results on Facade}

We present the full denoising results at the 10\% and 60\% noise level, as well as the 30\% level for completeness.

\begin{table*}[h!]
    \centering
    \begin{tabular}{|c|c|c|c|c|c|c|}\hline
 & \multicolumn{2}{c|}{10.00\%} & \multicolumn{2}{c|}{30.00\%} & \multicolumn{2}{c|}{60.00\%}\\ \cline{2-7}
 & PSNR & FSIMc & PSNR & FSIMc & PSNR & FSIMc\\ \hline
 RKCA & \textbf{33.7184} & \textbf{0.9960} & 28.6235 & \textbf{0.9764} & \textbf{24.4491} & \textbf{0.9179}\\
 Cauchy ST & 31.7107 & 0.9924 & \textbf{29.0492} & 0.9724 & 22.3150 & 0.8820\\
 TRPCA 2016 & 32.7462 & 0.9955 & 27.0979 & \textbf{0.9764} & 23.6552 & 0.9109\\
 HORPCA-S & 33.1370 & 0.9938 & 27.5200 & 0.9714 & 22.8811 & 0.9060\\
 Welsh ST & 28.0183 & 0.9725 & 25.2946 & 0.9309 & 19.8508 & 0.8421\\
 NCTRPCA & 27.8527 & 0.9744 & 21.2992 & 0.8881 & 19.4207 & 0.8121\\
 TRPCA 2014 & 31.8137 & 0.9886 & 25.9989 & 0.9484 & 21.4069 & 0.8529\\
 RPCA (baseline) & 21.3266 & 0.9404 & 13.6751 & 0.8145 & 9.2293 & 0.6633\\
 RNNDL (baseline) & 21.1150 & 0.9263 & 14.2439 & 0.8041 & 10.0129 & 0.6690\\
 LRR Exact (baseline) & 19.0759 & 0.8929 & 19.2146 & 0.9001 & 19.0804 & 0.8993\\
 LRR Inexact (baseline) & 16.4344 & 0.8830 & 15.8606 & 0.8776 & 15.7061 & 0.8762\\
 DIP Model 1 & 24.5215 & 0.9211 & 20.4301 & 0.8385 & 17.2867 & 0.7258\\
 DIP Model 2 & 24.1887 & 0.9181 & 21.4694 & 0.8612 & 18.5480 & 0.7410\\
\hline
    \end{tabular}
    \caption{Numerical values of the image quality metrics on the Facade benchmark for the three noise levels. Best performance in bold.}
    \label{tab:full_results_facade}
\end{table*}

\input{appendices/experimental_evaluation/image_denoising/colour_denoising/image_grid_facade_10}

\input{appendices/experimental_evaluation/image_denoising/colour_denoising/image_grid_facade_30}

\input{appendices/experimental_evaluation/image_denoising/colour_denoising/image_grid_facade_60}

\clearpage
\input{appendices/experimental_evaluation/background_modelling/background_modelling}

\clearpage
\input{appendices/experimental_evaluation/image_denoising/sedil/sedil}

%% file: appendices/experimental_evaluation/image_denoising/face_denoising/image_grid_yale_10_1.tex
\begin{figure}[h]
\begin{subfigure}[b]{.24\linewidth}
\includegraphics[width = \linewidth]{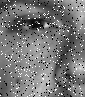}
\caption{Noisy}
\end{subfigure}
\begin{subfigure}[b]{.24\linewidth}
\includegraphics[width = \linewidth]{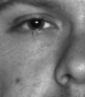} 
\caption{Original}
\end{subfigure}
\begin{subfigure}[b]{.24\linewidth}
\includegraphics[width = \linewidth]{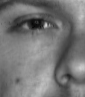} 
\caption{Cauchy ST}
\end{subfigure}\hfill
\begin{subfigure}[b]{.24\linewidth}
\includegraphics[width = \linewidth]{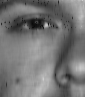} 
\caption{Welsh ST}
\end{subfigure}\hfill
\begin{subfigure}[b]{.24\linewidth}
\includegraphics[width = \linewidth]{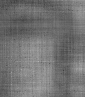} 
\caption{TRPCA '14}
\end{subfigure}\hfill
\begin{subfigure}[b]{.24\linewidth}
\includegraphics[width = \linewidth]{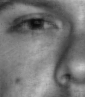}
\caption{TRPCA '16}
\end{subfigure}\hfill
\begin{subfigure}[b]{.24\linewidth}
\includegraphics[width = \linewidth]{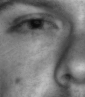} 
\caption{HORPCA-S}
\end{subfigure}\hfill
\begin{subfigure}[b]{.24\linewidth}
\includegraphics[width = \linewidth]{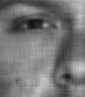} 
\caption{NCTRPCA}
\end{subfigure}\hfill
\begin{subfigure}[b]{.24\linewidth}
\includegraphics[width = \linewidth]{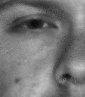} 
\caption{RNNDL}
\end{subfigure}\hfill
\begin{subfigure}[b]{.24\linewidth}
\includegraphics[width = \linewidth]{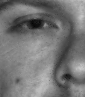} 
\caption{RPCA}
\end{subfigure}\hfill
\begin{subfigure}[b]{.24\linewidth}
\includegraphics[width = \linewidth]{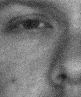} 
\caption{LRRe}
\end{subfigure}\hfill
\begin{subfigure}[b]{.24\linewidth}
\includegraphics[width = \linewidth]{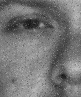} 
\caption{LRRi}
\end{subfigure}\hfill
\begin{subfigure}[b]{.24\linewidth}
\includegraphics[width = \linewidth]{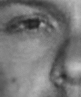} 
\caption{DIP1}
\end{subfigure}\hfill
\begin{subfigure}[b]{.24\linewidth}
\includegraphics[width = \linewidth]{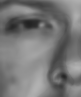} 
\caption{DIP2}
\end{subfigure}\hfill
\begin{subfigure}[b]{.24\linewidth}
\includegraphics[width = \linewidth]{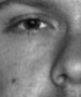}
\caption{RKCA}
\end{subfigure}
\caption{Comparative performance on the Yale benchmark with 10\% salt \& pepper noise - first illumination.}
\label{fig:visual_yale_10_i1}
\end{figure}

%% file: appendices/experimental_evaluation/image_denoising/face_denoising/image_grid_yale_10_3.tex
\begin{figure}[h]
\begin{subfigure}[b]{.24\linewidth}
\includegraphics[width = \linewidth]{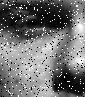}
\caption{Noisy}
\end{subfigure}
\begin{subfigure}[b]{.24\linewidth}
\includegraphics[width = \linewidth]{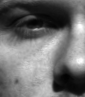} 
\caption{Original}
\end{subfigure}
\begin{subfigure}[b]{.24\linewidth}
\includegraphics[width = \linewidth]{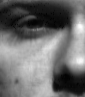} 
\caption{Cauchy ST}
\end{subfigure}\hfill
\begin{subfigure}[b]{.24\linewidth}
\includegraphics[width = \linewidth]{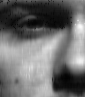} 
\caption{Welsh ST}
\end{subfigure}\hfill
\begin{subfigure}[b]{.24\linewidth}
\includegraphics[width = \linewidth]{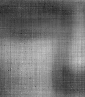} 
\caption{TRPCA '14}
\end{subfigure}\hfill
\begin{subfigure}[b]{.24\linewidth}
\includegraphics[width = \linewidth]{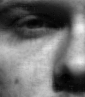}
\caption{TRPCA '16}
\end{subfigure}\hfill
\begin{subfigure}[b]{.24\linewidth}
\includegraphics[width = \linewidth]{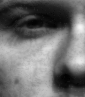} 
\caption{HORPCA-S}
\end{subfigure}\hfill
\begin{subfigure}[b]{.24\linewidth}
\includegraphics[width = \linewidth]{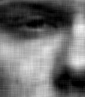} 
\caption{NCTRPCA}
\end{subfigure}\hfill
\begin{subfigure}[b]{.24\linewidth}
\includegraphics[width = \linewidth]{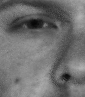} 
\caption{RNNDL}
\end{subfigure}\hfill
\begin{subfigure}[b]{.24\linewidth}
\includegraphics[width = \linewidth]{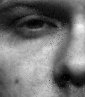} 
\caption{RPCA}
\end{subfigure}\hfill
\begin{subfigure}[b]{.24\linewidth}
\includegraphics[width = \linewidth]{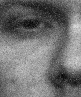} 
\caption{LRRe}
\end{subfigure}\hfill
\begin{subfigure}[b]{.24\linewidth}
\includegraphics[width = \linewidth]{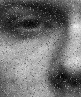} 
\caption{LRRi}
\end{subfigure}\hfill
\begin{subfigure}[b]{.24\linewidth}
\includegraphics[width = \linewidth]{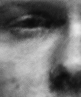} 
\caption{DIP1}
\end{subfigure}\hfill
\begin{subfigure}[b]{.24\linewidth}
\includegraphics[width = \linewidth]{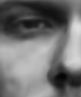} 
\caption{DIP2}
\end{subfigure}\hfill
\begin{subfigure}[b]{.24\linewidth}
\includegraphics[width = \linewidth]{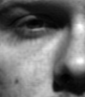}
\caption{RKCA}
\end{subfigure}
\caption{Comparative performance on the Yale benchmark with 10\% salt \& pepper noise - third illumination.}
\label{fig:visual_yale_10_i3}
\end{figure}

%% file: appendices/experimental_evaluation/image_denoising/face_denoising/image_grid_yale_30_1.tex
\begin{figure}[h]
\begin{subfigure}[b]{.24\linewidth}
\includegraphics[width = \linewidth]{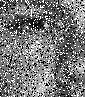}
\caption{Noisy}
\end{subfigure}
\begin{subfigure}[b]{.24\linewidth}
\includegraphics[width = \linewidth]{appendices/experimental_evaluation/reference_images/yale_1_original} 
\caption{Original}
\end{subfigure}
\begin{subfigure}[b]{.24\linewidth}
\includegraphics[width = \linewidth]{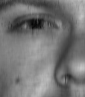} 
\caption{Cauchy ST}
\end{subfigure}\hfill
\begin{subfigure}[b]{.24\linewidth}
\includegraphics[width = \linewidth]{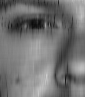} 
\caption{Welsh ST}
\end{subfigure}\hfill
\begin{subfigure}[b]{.24\linewidth}
\includegraphics[width = \linewidth]{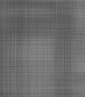} 
\caption{TRPCA '14}
\end{subfigure}\hfill
\begin{subfigure}[b]{.24\linewidth}
\includegraphics[width = \linewidth]{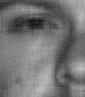}
\caption{TRPCA '16}
\end{subfigure}\hfill
\begin{subfigure}[b]{.24\linewidth}
\includegraphics[width = \linewidth]{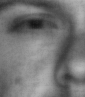} 
\caption{HORPCA-S}
\end{subfigure}\hfill
\begin{subfigure}[b]{.24\linewidth}
\includegraphics[width = \linewidth]{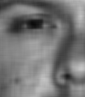} 
\caption{NCTRPCA}
\end{subfigure}\hfill
\begin{subfigure}[b]{.24\linewidth}
\includegraphics[width = \linewidth]{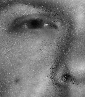} 
\caption{RNNDL}
\end{subfigure}\hfill
\begin{subfigure}[b]{.24\linewidth}
\includegraphics[width = \linewidth]{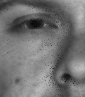} 
\caption{RPCA}
\end{subfigure}\hfill
\begin{subfigure}[b]{.24\linewidth}
\includegraphics[width = \linewidth]{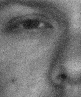} 
\caption{LRRe}
\end{subfigure}\hfill
\begin{subfigure}[b]{.24\linewidth}
\includegraphics[width = \linewidth]{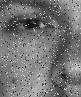} 
\caption{LRRi}
\end{subfigure}\hfill
\begin{subfigure}[b]{.24\linewidth}
\includegraphics[width = \linewidth]{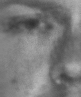} 
\caption{DIP1}
\end{subfigure}\hfill
\begin{subfigure}[b]{.24\linewidth}
\includegraphics[width = \linewidth]{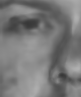} 
\caption{DIP2}
\end{subfigure}\hfill
\begin{subfigure}[b]{.24\linewidth}
\includegraphics[width = \linewidth]{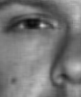}
\caption{RKCA}
\end{subfigure}
\caption{Comparative performance on the Yale benchmark with 30\% salt \& pepper noise - first illumination.}
\label{fig:visual_yale_30_i1}
\end{figure}

%% file: appendices/experimental_evaluation/image_denoising/face_denoising/image_grid_yale_30_3.tex
\begin{figure}[h]
\begin{subfigure}[b]{.24\linewidth}
\includegraphics[width = \linewidth]{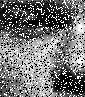}
\caption{Noisy}
\end{subfigure}
\begin{subfigure}[b]{.24\linewidth}
\includegraphics[width = \linewidth]{appendices/experimental_evaluation/reference_images/yale_3_original} 
\caption{Original}
\end{subfigure}
\begin{subfigure}[b]{.24\linewidth}
\includegraphics[width = \linewidth]{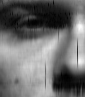} 
\caption{Cauchy ST}
\end{subfigure}\hfill
\begin{subfigure}[b]{.24\linewidth}
\includegraphics[width = \linewidth]{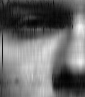} 
\caption{Welsh ST}
\end{subfigure}\hfill
\begin{subfigure}[b]{.24\linewidth}
\includegraphics[width = \linewidth]{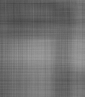} 
\caption{TRPCA '14}
\end{subfigure}\hfill
\begin{subfigure}[b]{.24\linewidth}
\includegraphics[width = \linewidth]{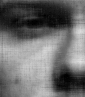}
\caption{TRPCA '16}
\end{subfigure}\hfill
\begin{subfigure}[b]{.24\linewidth}
\includegraphics[width = \linewidth]{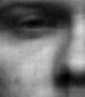} 
\caption{HORPCA-S}
\end{subfigure}\hfill
\begin{subfigure}[b]{.24\linewidth}
\includegraphics[width = \linewidth]{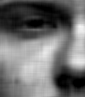} 
\caption{NCTRPCA}
\end{subfigure}\hfill
\begin{subfigure}[b]{.24\linewidth}
\includegraphics[width = \linewidth]{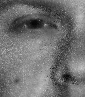} 
\caption{RNNDL}
\end{subfigure}\hfill
\begin{subfigure}[b]{.24\linewidth}
\includegraphics[width = \linewidth]{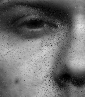} 
\caption{RPCA}
\end{subfigure}\hfill
\begin{subfigure}[b]{.24\linewidth}
\includegraphics[width = \linewidth]{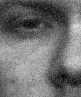} 
\caption{LRRe}
\end{subfigure}\hfill
\begin{subfigure}[b]{.24\linewidth}
\includegraphics[width = \linewidth]{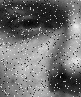} 
\caption{LRRi}
\end{subfigure}\hfill
\begin{subfigure}[b]{.24\linewidth}
\includegraphics[width = \linewidth]{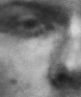} 
\caption{DIP1}
\end{subfigure}\hfill
\begin{subfigure}[b]{.24\linewidth}
\includegraphics[width = \linewidth]{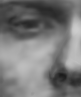} 
\caption{DIP2}
\end{subfigure}\hfill
\begin{subfigure}[b]{.24\linewidth}
\includegraphics[width = \linewidth]{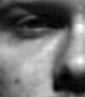}
\caption{RKCA}
\end{subfigure}
\caption{Comparative performance on the Yale benchmark with 30\% salt \& pepper noise - third illumination.}
\label{fig:visual_yale_30_i3}
\end{figure}

%% file: appendices/experimental_evaluation/image_denoising/face_denoising/image_grid_yale_60_1.tex
\begin{figure}[h]
\begin{subfigure}[b]{.24\linewidth}
\includegraphics[width = \linewidth]{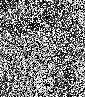}
\caption{Noisy}
\end{subfigure}
\begin{subfigure}[b]{.24\linewidth}
\includegraphics[width = \linewidth]{appendices/experimental_evaluation/reference_images/yale_1_original} 
\caption{Original}
\end{subfigure}
\begin{subfigure}[b]{.24\linewidth}
\includegraphics[width = \linewidth]{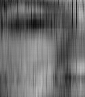} 
\caption{Cauchy ST}
\end{subfigure}\hfill
\begin{subfigure}[b]{.24\linewidth}
\includegraphics[width = \linewidth]{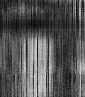} 
\caption{Welsh ST}
\end{subfigure}\hfill
\begin{subfigure}[b]{.24\linewidth}
\includegraphics[width = \linewidth]{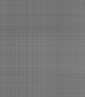} 
\caption{TRPCA '14}
\end{subfigure}\hfill
\begin{subfigure}[b]{.24\linewidth}
\includegraphics[width = \linewidth]{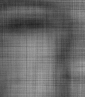}
\caption{TRPCA '16}
\end{subfigure}\hfill
\begin{subfigure}[b]{.24\linewidth}
\includegraphics[width = \linewidth]{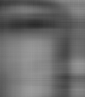} 
\caption{HORPCA-S}
\end{subfigure}\hfill
\begin{subfigure}[b]{.24\linewidth}
\includegraphics[width = \linewidth]{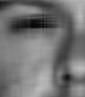} 
\caption{NCTRPCA}
\end{subfigure}\hfill
\begin{subfigure}[b]{.24\linewidth}
\includegraphics[width = \linewidth]{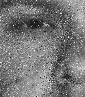} 
\caption{RNNDL}
\end{subfigure}\hfill
\begin{subfigure}[b]{.24\linewidth}
\includegraphics[width = \linewidth]{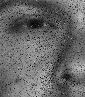} 
\caption{RPCA}
\end{subfigure}\hfill
\begin{subfigure}[b]{.24\linewidth}
\includegraphics[width = \linewidth]{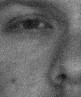} 
\caption{LRRe}
\end{subfigure}\hfill
\begin{subfigure}[b]{.24\linewidth}
\includegraphics[width = \linewidth]{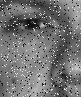} 
\caption{LRRi}
\end{subfigure}\hfill
\begin{subfigure}[b]{.24\linewidth}
\includegraphics[width = \linewidth]{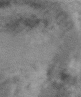} 
\caption{DIP1}
\end{subfigure}\hfill
\begin{subfigure}[b]{.24\linewidth}
\includegraphics[width = \linewidth]{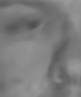} 
\caption{DIP2}
\end{subfigure}\hfill
\begin{subfigure}[b]{.24\linewidth}
\includegraphics[width = \linewidth]{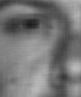}
\caption{RKCA}
\end{subfigure}
\caption{Comparative performance on the Yale benchmark with 60\% salt \& pepper noise - first illumination.}
\label{fig:visual_yale_60_i1}
\end{figure}

%% file: appendices/experimental_evaluation/image_denoising/face_denoising/image_grid_yale_60_3.tex
\begin{figure}[h]
\begin{subfigure}[b]{.24\linewidth}
\includegraphics[width = \linewidth]{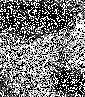}
\caption{Noisy}
\end{subfigure}
\begin{subfigure}[b]{.24\linewidth}
\includegraphics[width = \linewidth]{appendices/experimental_evaluation/reference_images/yale_3_original} 
\caption{Original}
\end{subfigure}
\begin{subfigure}[b]{.24\linewidth}
\includegraphics[width = \linewidth]{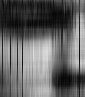} 
\caption{Cauchy ST}
\end{subfigure}\hfill
\begin{subfigure}[b]{.24\linewidth}
\includegraphics[width = \linewidth]{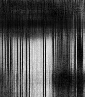} 
\caption{Welsh ST}
\end{subfigure}\hfill
\begin{subfigure}[b]{.24\linewidth}
\includegraphics[width = \linewidth]{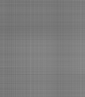} 
\caption{TRPCA '14}
\end{subfigure}\hfill
\begin{subfigure}[b]{.24\linewidth}
\includegraphics[width = \linewidth]{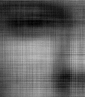}
\caption{TRPCA '16}
\end{subfigure}\hfill
\begin{subfigure}[b]{.24\linewidth}
\includegraphics[width = \linewidth]{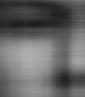} 
\caption{HORPCA-S}
\end{subfigure}\hfill
\begin{subfigure}[b]{.24\linewidth}
\includegraphics[width = \linewidth]{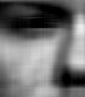} 
\caption{NCTRPCA}
\end{subfigure}\hfill
\begin{subfigure}[b]{.24\linewidth}
\includegraphics[width = \linewidth]{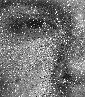} 
\caption{RNNDL}
\end{subfigure}\hfill
\begin{subfigure}[b]{.24\linewidth}
\includegraphics[width = \linewidth]{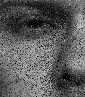} 
\caption{RPCA}
\end{subfigure}\hfill
\begin{subfigure}[b]{.24\linewidth}
\includegraphics[width = \linewidth]{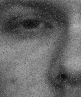} 
\caption{LRRe}
\end{subfigure}\hfill
\begin{subfigure}[b]{.24\linewidth}
\includegraphics[width = \linewidth]{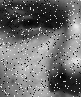} 
\caption{LRRi}
\end{subfigure}\hfill
\begin{subfigure}[b]{.24\linewidth}
\includegraphics[width = \linewidth]{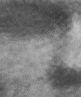} 
\caption{DIP1}
\end{subfigure}\hfill
\begin{subfigure}[b]{.24\linewidth}
\includegraphics[width = \linewidth]{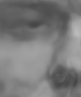} 
\caption{DIP2}
\end{subfigure}\hfill
\begin{subfigure}[b]{.24\linewidth}
\includegraphics[width = \linewidth]{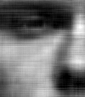}
\caption{RKCA}
\end{subfigure}
\caption{Comparative performance on the Yale benchmark with 60\% salt \& pepper noise - third illumination.}
\label{fig:visual_yale_60_i3}
\end{figure}

%% file: appendices/experimental_evaluation/image_denoising/colour_denoising/image_grid_facade_10.tex
\begin{figure}[h!]
\begin{subfigure}[b]{.24\linewidth}
\includegraphics[width = \linewidth]{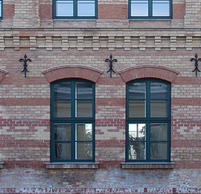} 
\caption{Original}
\end{subfigure}
\begin{subfigure}[b]{.24\linewidth}
\includegraphics[width = \linewidth]{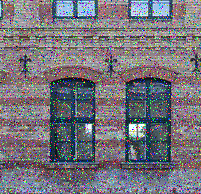}
\caption{Noisy}
\end{subfigure}
\begin{subfigure}[b]{.24\linewidth}
\includegraphics[width = \linewidth]{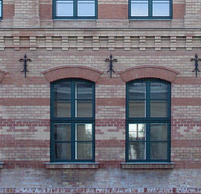} 
\caption{Cauchy ST}
\end{subfigure}\hfill
\begin{subfigure}[b]{.24\linewidth}
\includegraphics[width = \linewidth]{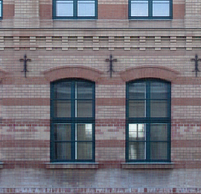} 
\caption{Welsh ST}
\end{subfigure}\hfill
\begin{subfigure}[b]{.24\linewidth}
\includegraphics[width = \linewidth]{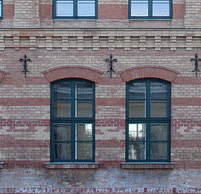} 
\caption{TRPCA '14}
\end{subfigure}\hfill
\begin{subfigure}[b]{.24\linewidth}
\includegraphics[width = \linewidth]{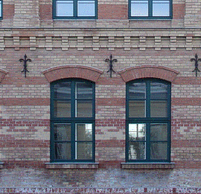}
\caption{TRPCA '16}
\end{subfigure}\hfill
\begin{subfigure}[b]{.24\linewidth}
\includegraphics[width = \linewidth]{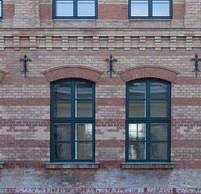} 
\caption{HORPCA-S}
\end{subfigure}\hfill
\begin{subfigure}[b]{.24\linewidth}
\includegraphics[width = \linewidth]{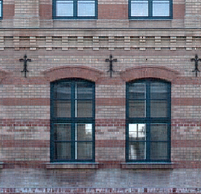} 
\caption{NCTRPCA}
\end{subfigure}\hfill
\begin{subfigure}[b]{.24\linewidth}
\includegraphics[width = \linewidth]{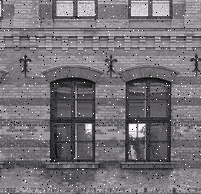} 
\caption{RNNDL}
\end{subfigure}\hfill
\begin{subfigure}[b]{.24\linewidth}
\includegraphics[width = \linewidth]{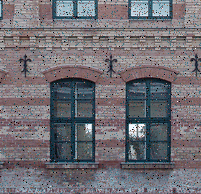} 
\caption{RPCA}
\end{subfigure}\hfill
\begin{subfigure}[b]{.24\linewidth}
\includegraphics[width = \linewidth]{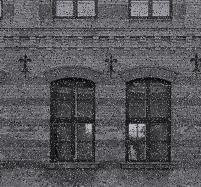} 
\caption{LRRe}
\end{subfigure}\hfill
\begin{subfigure}[b]{.24\linewidth}
\includegraphics[width = \linewidth]{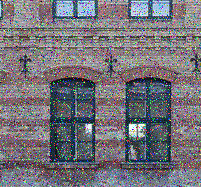} 
\caption{LRRi}
\end{subfigure}\hfill
\begin{subfigure}[b]{.24\linewidth}
\includegraphics[width = \linewidth]{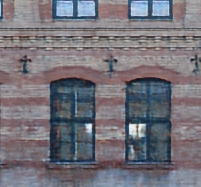} 
\caption{DIP1}
\end{subfigure}\hfill
\begin{subfigure}[b]{.24\linewidth}
\includegraphics[width = \linewidth]{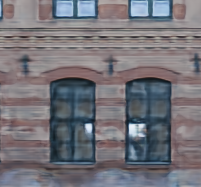} 
\caption{DIP2}
\end{subfigure}\hfill
\begin{subfigure}[b]{.24\linewidth}
\includegraphics[width = \linewidth]{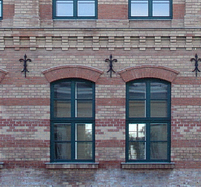}
\caption{RKCA}
\end{subfigure}
\caption{Comparative performance on the Facade benchmark with 10\% noise.}
\label{fig:visual_facade_10}
\end{figure}

%% file: appendices/experimental_evaluation/image_denoising/colour_denoising/image_grid_facade_30.tex
\begin{figure}[h!]
\begin{subfigure}[b]{.24\linewidth}
\includegraphics[width = \linewidth]{appendices/experimental_evaluation/reference_images/facade_original} 
\caption{Original}
\end{subfigure}
\begin{subfigure}[b]{.24\linewidth}
\includegraphics[width = \linewidth]{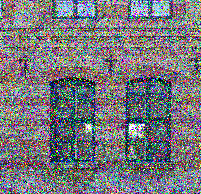}
\caption{Noisy}
\end{subfigure}
\begin{subfigure}[b]{.24\linewidth}
\includegraphics[width = \linewidth]{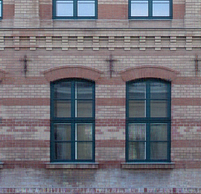} 
\caption{Cauchy ST}
\end{subfigure}\hfill
\begin{subfigure}[b]{.24\linewidth}
\includegraphics[width = \linewidth]{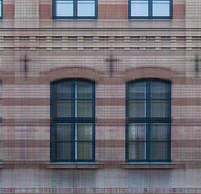} 
\caption{Welsh ST}
\end{subfigure}\hfill
\begin{subfigure}[b]{.24\linewidth}
\includegraphics[width = \linewidth]{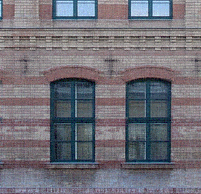} 
\caption{TRPCA '14}
\end{subfigure}\hfill
\begin{subfigure}[b]{.24\linewidth}
\includegraphics[width = \linewidth]{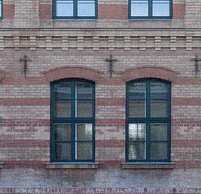}
\caption{TRPCA '16}
\end{subfigure}\hfill
\begin{subfigure}[b]{.24\linewidth}
\includegraphics[width = \linewidth]{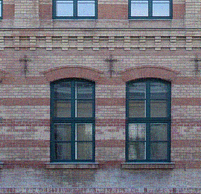} 
\caption{HORPCA-S}
\end{subfigure}\hfill
\begin{subfigure}[b]{.24\linewidth}
\includegraphics[width = \linewidth]{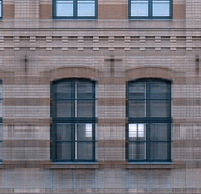} 
\caption{NCTRPCA}
\end{subfigure}\hfill
\begin{subfigure}[b]{.24\linewidth}
\includegraphics[width = \linewidth]{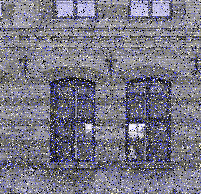} 
\caption{RNNDL}
\end{subfigure}\hfill
\begin{subfigure}[b]{.24\linewidth}
\includegraphics[width = \linewidth]{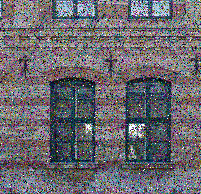} 
\caption{RPCA}
\end{subfigure}\hfill
\begin{subfigure}[b]{.24\linewidth}
\includegraphics[width = \linewidth]{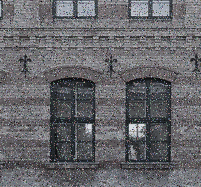} 
\caption{LRRe}
\end{subfigure}\hfill
\begin{subfigure}[b]{.24\linewidth}
\includegraphics[width = \linewidth]{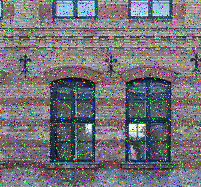} 
\caption{LRRi}
\end{subfigure}\hfill
\begin{subfigure}[b]{.24\linewidth}
\includegraphics[width = \linewidth]{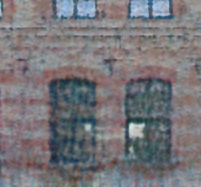} 
\caption{DIP1}
\end{subfigure}\hfill
\begin{subfigure}[b]{.24\linewidth}
\includegraphics[width = \linewidth]{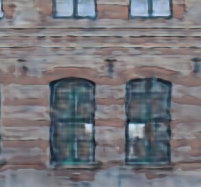} 
\caption{DIP2}
\end{subfigure}\hfill
\begin{subfigure}[b]{.24\linewidth}
\includegraphics[width = \linewidth]{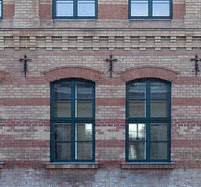}
\caption{RKCA}
\end{subfigure}
\caption{Comparative performance on the Facade benchmark with 30\% noise.}
\label{fig:visual_facade_30_appendix}
\end{figure}

%% file: appendices/experimental_evaluation/image_denoising/colour_denoising/image_grid_facade_60.tex
\begin{figure}[h!]
\begin{subfigure}[b]{.24\linewidth}
\includegraphics[width = \linewidth]{appendices/experimental_evaluation/reference_images/facade_original} 
\caption{Original}
\end{subfigure}
\begin{subfigure}[b]{.24\linewidth}
\includegraphics[width = \linewidth]{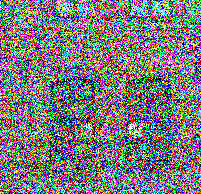}
\caption{Noisy}
\end{subfigure}
\begin{subfigure}[b]{.24\linewidth}
\includegraphics[width = \linewidth]{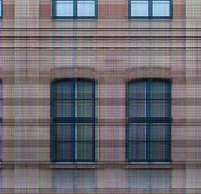} 
\caption{Cauchy ST}
\end{subfigure}\hfill
\begin{subfigure}[b]{.24\linewidth}
\includegraphics[width = \linewidth]{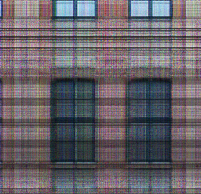} 
\caption{Welsh ST}
\end{subfigure}\hfill
\begin{subfigure}[b]{.24\linewidth}
\includegraphics[width = \linewidth]{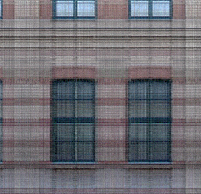} 
\caption{TRPCA '14}
\end{subfigure}\hfill
\begin{subfigure}[b]{.24\linewidth}
\includegraphics[width = \linewidth]{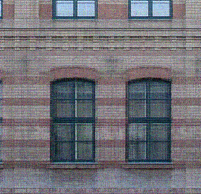}
\caption{TRPCA '16}
\end{subfigure}\hfill
\begin{subfigure}[b]{.24\linewidth}
\includegraphics[width = \linewidth]{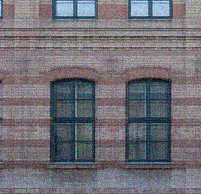} 
\caption{HORPCA-S}
\end{subfigure}\hfill
\begin{subfigure}[b]{.24\linewidth}
\includegraphics[width = \linewidth]{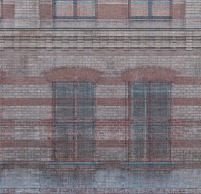} 
\caption{NCTRPCA}
\end{subfigure}\hfill
\begin{subfigure}[b]{.24\linewidth}
\includegraphics[width = \linewidth]{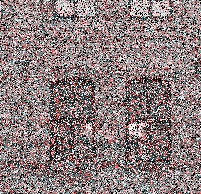} 
\caption{RNNDL}
\end{subfigure}\hfill
\begin{subfigure}[b]{.24\linewidth}
\includegraphics[width = \linewidth]{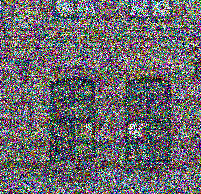} 
\caption{RPCA}
\end{subfigure}\hfill
\begin{subfigure}[b]{.24\linewidth}
\includegraphics[width = \linewidth]{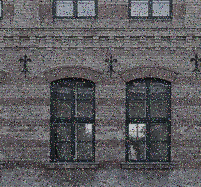} 
\caption{LRRe}
\end{subfigure}\hfill
\begin{subfigure}[b]{.24\linewidth}
\includegraphics[width = \linewidth]{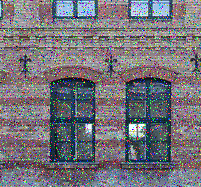} 
\caption{LRRi}
\end{subfigure}\hfill
\begin{subfigure}[b]{.24\linewidth}
\includegraphics[width = \linewidth]{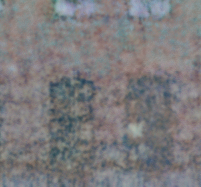} 
\caption{DIP1}
\end{subfigure}\hfill
\begin{subfigure}[b]{.24\linewidth}
\includegraphics[width = \linewidth]{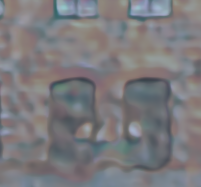} 
\caption{DIP2}
\end{subfigure}\hfill
\begin{subfigure}[b]{.24\linewidth}
\includegraphics[width = \linewidth]{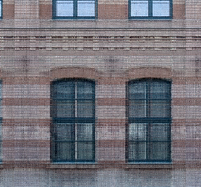}
\caption{RKCA}
\end{subfigure}
\caption{Comparative performance on the Facade benchmark with 60\% noise.}
\label{fig:visual_facade_60}
\end{figure}

%% file: appendices/experimental_evaluation/background_modelling/background_modelling.tex
\subsection{Background subtraction}

In this Section we present the DET curves obtained on the \textit{Highway} and \textit{Hall} experiments, and the backgrounds obtained with each algorithm on the \textit{Highway} benchmark.

\begin{figure}[h!]
\centering
\includegraphics[width=\columnwidth]{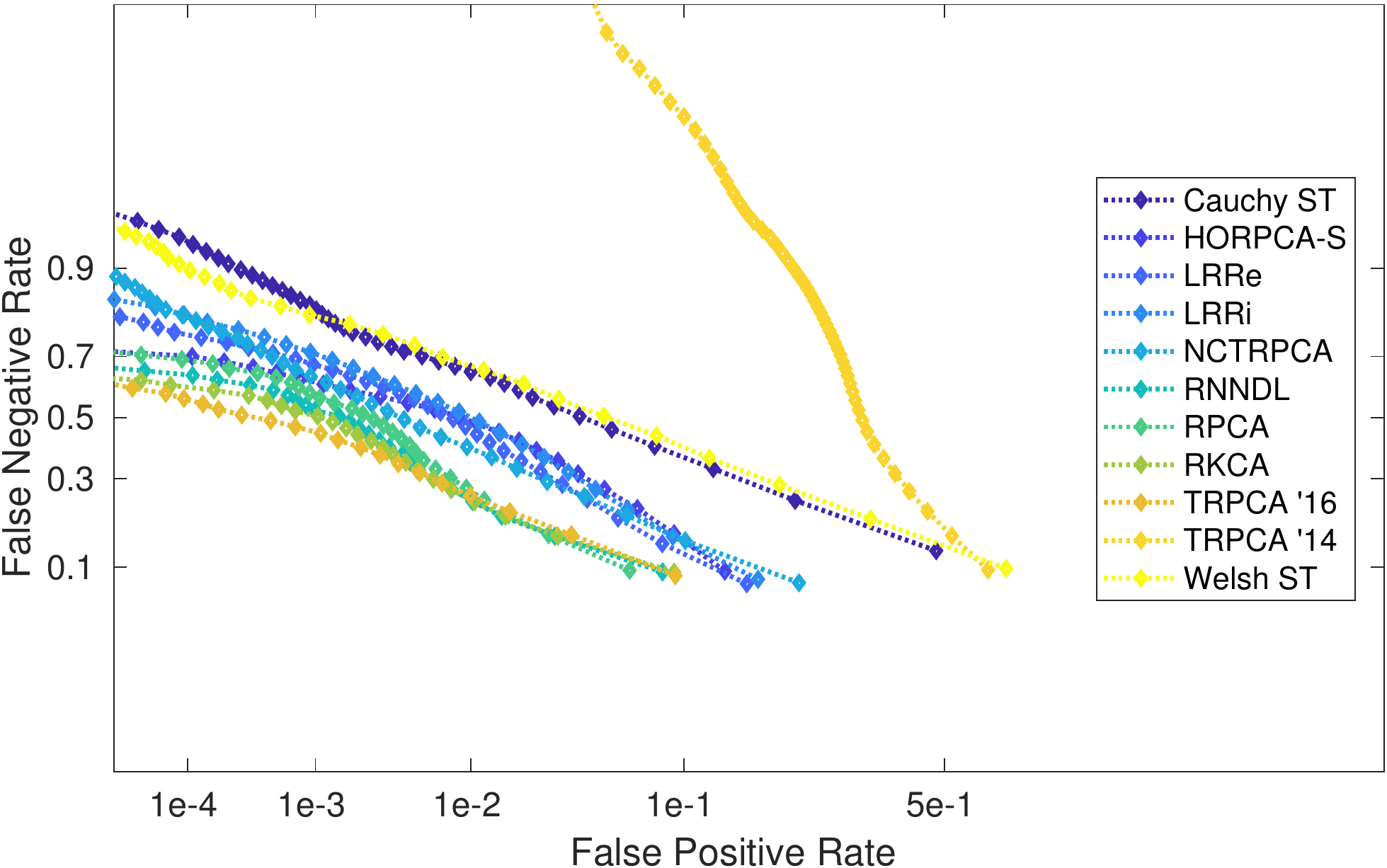}
\caption{DET curves on the \textit{Highway} dataset.}
\label{fig:det_highway}
\end{figure}

\begin{figure}[h!]
    \centering
    \includegraphics[width=\columnwidth]{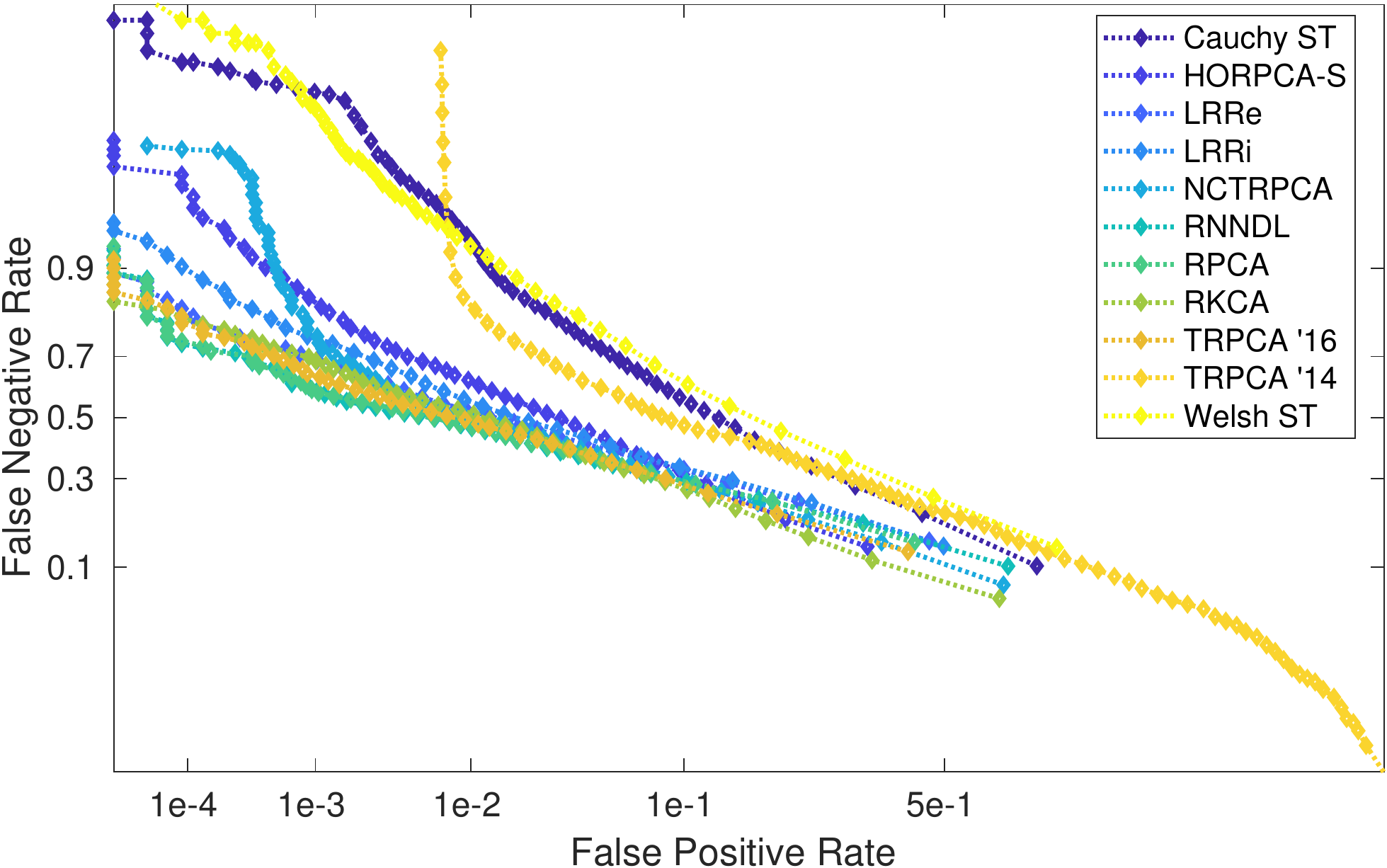}
    \caption{DET curves on the \textit{Airport Hall} dataset.}
    \label{fig:det_hall}
\end{figure}

\input{appendices/experimental_evaluation/background_modelling/image_grid_highway}

%% file: appendices/experimental_evaluation/background_modelling/image_grid_highway.tex
\begin{figure}[h!]
\begin{subfigure}[b]{.25\linewidth}
\includegraphics[width = \linewidth]{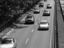} 
\caption{Original}
\end{subfigure}
\begin{subfigure}[b]{.25\linewidth}
\includegraphics[width = \linewidth]{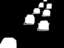}
\caption{GT}
\end{subfigure}
\begin{subfigure}[b]{.25\linewidth}
\includegraphics[width = \linewidth]{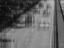} 
\caption{Cauchy ST}
\end{subfigure}\hfill
\begin{subfigure}[b]{.25\linewidth}
\includegraphics[width = \linewidth]{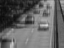} 
\caption{Welsh ST}
\end{subfigure}\hfill
\begin{subfigure}[b]{.25\linewidth}
\includegraphics[width = \linewidth]{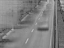}
\caption{TRPCA '14}
\end{subfigure}\hfill
\begin{subfigure}[b]{.25\linewidth}
\includegraphics[width = \linewidth]{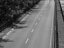}
\caption{TRPCA '16}
\end{subfigure}\hfill
\begin{subfigure}[b]{.25\linewidth}
\includegraphics[width = \linewidth]{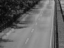} 
\caption{HORPCA-S}
\end{subfigure}\hfill
\begin{subfigure}[b]{.25\linewidth}
\includegraphics[width = \linewidth]{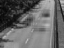} 
\caption{NCTRPCA}
\end{subfigure}\hfill
\begin{subfigure}[b]{.25\linewidth}
\includegraphics[width = \linewidth]{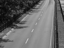} 
\caption{RNNDL}
\end{subfigure}\hfill
\begin{subfigure}[b]{.25\linewidth}
\includegraphics[width = \linewidth]{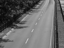} 
\caption{RPCA}
\end{subfigure}\hfill
\begin{subfigure}[b]{.25\linewidth}
\includegraphics[width = \linewidth]{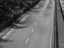} 
\caption{LRRe}
\end{subfigure}\hfill
\begin{subfigure}[b]{.25\linewidth}
\includegraphics[width = \linewidth]{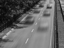} 
\caption{LRRi}
\end{subfigure}\hfill
\begin{subfigure}[b]{.25\linewidth}
\includegraphics[width = \linewidth]{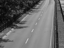}
\caption{RKCA}
\end{subfigure}
\caption{Background subtraction results on \textit{Highway}.}
\label{fig:visual_highway}
\end{figure}

%% file: appendices/experimental_evaluation/image_denoising/sedil/sedil.tex
\subsection{Comparison with SeDiL}

We present an experiment to assess the performance of SeDiL on the Yale-B benchmark. We chose this experiment because - as presented in the original paper \cite{hawe_separable_2013} - SeDiL is designed for denoising grayscale images.

\subsubsection{Design}

The data and the noise are the same as in the Yale-B experimental results presented in the paper. We describe the tuning procedure employed.

We tuned in SeDiL: $\kappa$ and $\lambda$ for training the dictionary, and the parameter used in FISTA for denoising that we shall denote as $\lambda_F$. In the original paper, the authors choose $\kappa = \lambda = \frac{0.1}{4w^2}$ with $w$ the dimension of the square image patches, and $\lambda_F = \frac{\sigma_{noise}}{100}$. We tuned the parameters via grid-search, choosing:
\begin{itemize}
    \item $\kappa = \frac{\kappa_0}{4w^2}$, $\kappa_0 \in \mathrm{linspace}(0.05, 0.5, 5)$
    \item $\lambda = \frac{\lambda_0}{4w^2}$, $\lambda_0 \in \mathrm{linspace}(0.05, 0.5, 5)$
    \item $\lambda_F \in 5*\mathrm{logspace}(-4, 0, 15)$
\end{itemize}
We kept $\rho = 100$ and patch sizes of $8 \times 8$ and extracted 40000 random patches for training from the 64 images of the first subject. We then followed the procedure described in the paper for denoising and extracted $8 \times 8$ patches in a sliding window with step size 1, denoised the patches with FISTA, and reconstructed the denoised image by averaging the denoised patches.

\subsubsection{Results}

We present the results at the three noise levels: $10\%$, $30\%$, $60\%$. We report the mean PSNR and FSIM on the 64 images, the PSNR and FSIM on the first image, and show the reconstructions obtained.

\begin{table*}[t]
\centering
\begin{tabular}{|c|c|c|c|c|}\hline
\textbf{Noise level} & \textbf{Best PSNR face 1} & \textbf{Best mean PSNR} & \textbf{Best  FSIM face 1} & \textbf{Best mean FSIM}\\ \hline
10\% & 23.937344 & 23.863287 & 0.840798 & 0.862061\\ \hline
30\% & 20.988780 & 19.249498 & 0.785824 & 0.807403\\ \hline
60\% & 17.480783 & 16.095463 & 0.703943 & 0.778116\\ \hline
\end{tabular}
\caption{Best PSNR and FSIM obtained with SeDiL on the first face and averaged on the 64 faces at the 3 noise levels.}
\end{table*}



\input{appendices/experimental_evaluation/image_denoising/sedil/sedil_grid_psnr}

\input{appendices/experimental_evaluation/image_denoising/sedil/sedil_grid_fsim}

\subsubsection{Comments}

As expected, the method is not robust to gross corruption. In all cases, the best results were obtained for $\kappa_0 = \lambda_0 = 0.5$ and $\lambda_F = 5$, indicative of the lack of robustness.

%% file: appendices/experimental_evaluation/image_denoising/sedil/sedil_grid_psnr.tex
\begin{figure}
\begin{subfigure}[b]{.24\linewidth}
\includegraphics[width = \linewidth]{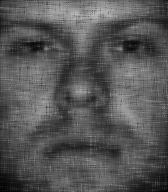} 
\caption{First face, 10\%}
\end{subfigure}\hfill
\begin{subfigure}[b]{.24\linewidth}
\includegraphics[width = \linewidth]{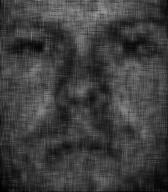} 
\caption{First face, 30\%}
\end{subfigure}\hfill
\begin{subfigure}[b]{.24\linewidth}
\includegraphics[width = \linewidth]{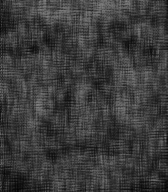} 
\caption{First face, 60\%}
\end{subfigure}\hfill
\begin{subfigure}[b]{.24\linewidth}
\includegraphics[width = \linewidth]{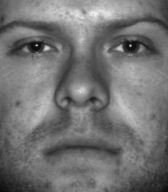}
\caption{Original face 1}
\end{subfigure}
\begin{subfigure}[b]{.24\linewidth}
\includegraphics[width = \linewidth]{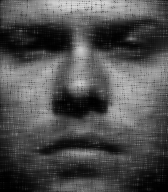} 
\caption{Third face, 10\%}
\end{subfigure}\hfill
\begin{subfigure}[b]{.24\linewidth}
\includegraphics[width = \linewidth]{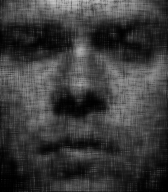} 
\caption{Third face, 30\%}
\end{subfigure}\hfill
\begin{subfigure}[b]{.24\linewidth}
\includegraphics[width = \linewidth]{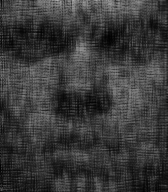} 
\caption{Third face, 60\%}
\end{subfigure}\hfill
\begin{subfigure}[b]{.24\linewidth}
\includegraphics[width = \linewidth]{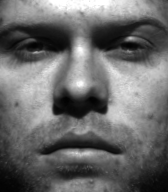}
\caption{Original face 3}
\end{subfigure}
\caption{Best reconstructions obtained with SeDiL on the first and third illumination according to the best overall PSNR score.}
\label{fig:visual_sedil_psnr}
\end{figure}

%% file: appendices/experimental_evaluation/image_denoising/sedil/sedil_grid_fsim.tex
\begin{figure}
\begin{subfigure}[b]{.24\linewidth}
\includegraphics[width = \linewidth]{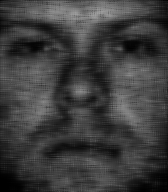} 
\caption{First face, 10\%}
\end{subfigure}\hfill
\begin{subfigure}[b]{.24\linewidth}
\includegraphics[width = \linewidth]{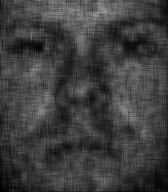} 
\caption{First face, 30\%}
\end{subfigure}\hfill
\begin{subfigure}[b]{.24\linewidth}
\includegraphics[width = \linewidth]{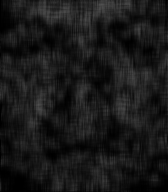} 
\caption{First face, 60\%}
\end{subfigure}\hfill
\begin{subfigure}[b]{.24\linewidth}
\includegraphics[width = \linewidth]{appendices/yale_subject_1_illumination_1}
\caption{Original face 1}
\end{subfigure}
\begin{subfigure}[b]{.24\linewidth}
\includegraphics[width = \linewidth]{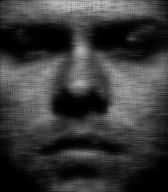} 
\caption{Third face, 10\%}
\end{subfigure}\hfill
\begin{subfigure}[b]{.24\linewidth}
\includegraphics[width = \linewidth]{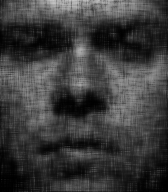} 
\caption{Third face, 30\%}
\end{subfigure}\hfill
\begin{subfigure}[b]{.24\linewidth}
\includegraphics[width = \linewidth]{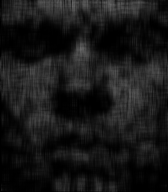} 
\caption{Third face, 60\%}
\end{subfigure}\hfill
\begin{subfigure}[b]{.24\linewidth}
\includegraphics[width = \linewidth]{appendices/yale_subject_1_illumination_3}
\caption{Original face 3}
\end{subfigure}
\caption{Best reconstructions obtained with SeDiL on the first and third illumination according to the best overall FSIM score.}
\label{fig:visual_sedil_fsim}
\end{figure}

%% file: appendices/time_influence.tex
\section{Time influence of the model parameters}
\label{appendix:time_influence}
To investigate the impact of the dimensions of the data and of the model parameters we devise an experiment on synthetic data (of the same kind used to validate our approach in Section 7.1 of the paper).

Assuming a dataset of $N$ random low-rank matrices $\mathbf{X}_i$ of dimension $m \times n$, and the RKCA model with degree 2 regularisation, we vary independently:

\begin{enumerate}
    \item The first dimension $m$ of the data (number of rows per observation, or height of the images) in [50, 100, 150, 200, 250, 500, 750, 1000] for $N = 100$ (Figure \ref{fig:impact_m})
    \item The second dimension $n$ (number of columns, or width of the images) in [50, 100, 150, 200, 250, 500, 750, 1000] for $N = 100$ (Figure \ref{fig:impact_n})
    \item The number of observations $N$ in [10, 25, 50, 100, 150, 200, 250, 500, 750, 1000] for $m = n = 100$ (Figure \ref{fig:impact_N})
    \item The upper bound we place on the rank of each reconstruction $r$ in [5, 10, 15, 20, 25, 30, 50, 75, 100, 150, 200, 500, 750, 1000] for $N = 100$ and $m = n = 1000$ (Figure \ref{fig:impact_r})
\end{enumerate}
We set $r = 50$ for the experiments on the data dimensions, and we let $\lambda = \frac{1}{\sqrt{N \max(n, m)}}$.

We set the convergence threshold to $1e-5$. All experiments were conducted in isolation on a \textit{Standard E8s v3} instance in Microsoft Azure Cloud.

\begin{figure}[h!]
    \centering
    \includegraphics[width=.475\linewidth]{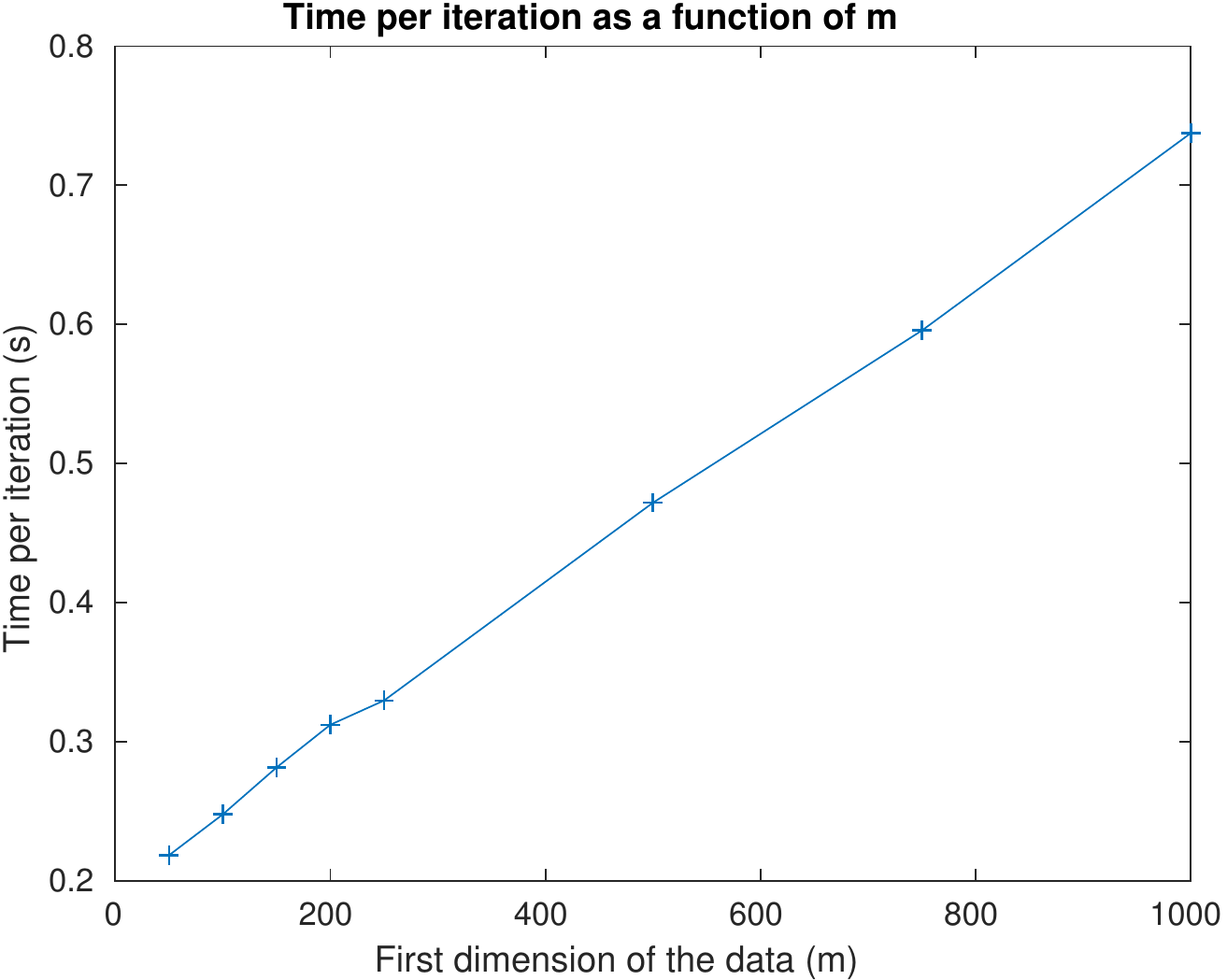}
    \includegraphics[width=.475\linewidth]{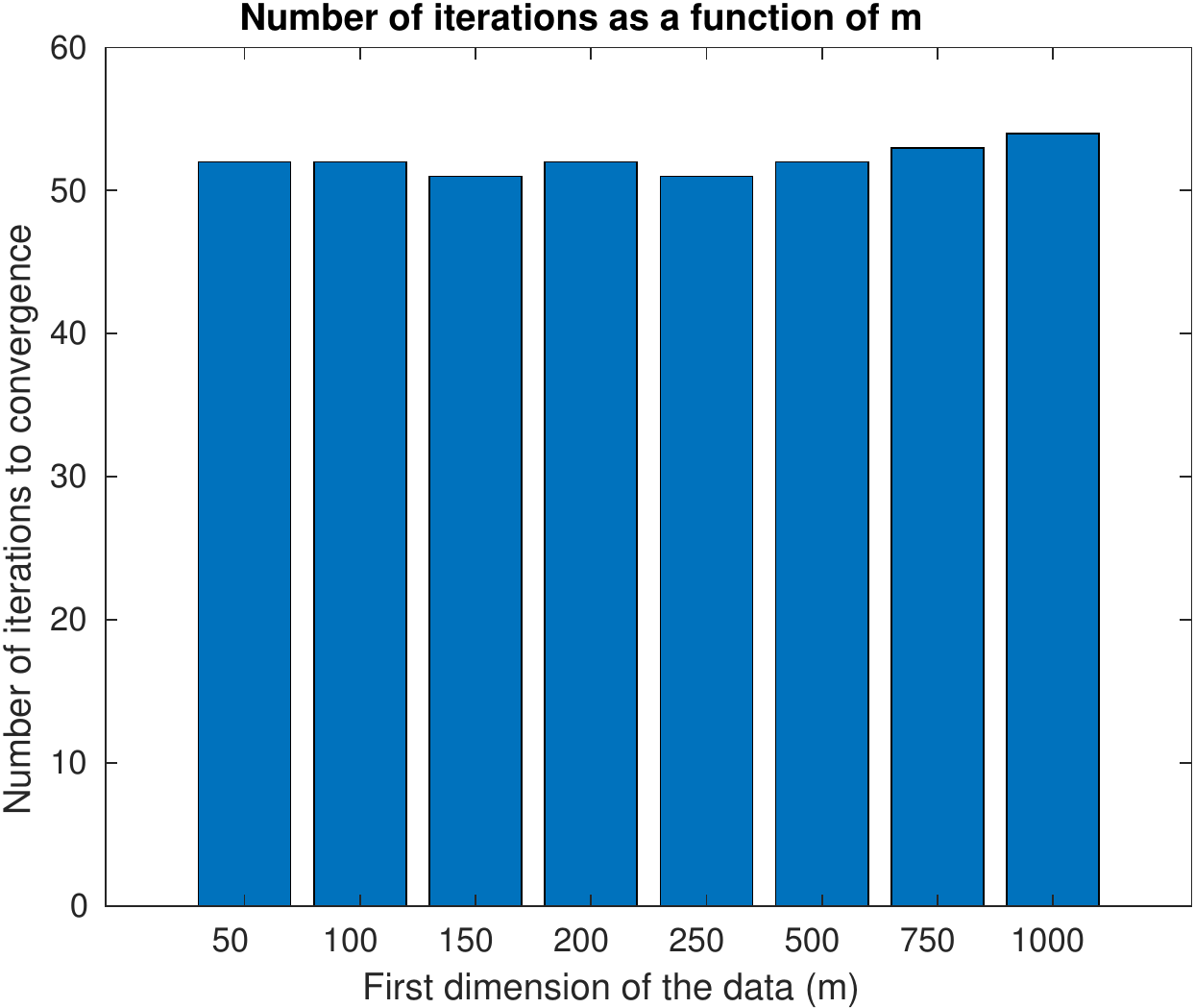}
    \caption{Experimental verification of the linear impact of the first data dimension $m$ on the run-time per iteration of RKCA with degree 2 regularisation. Left: time per iteration in seconds, right: number of iterations to convergence at the $1e-5$ threshold.}
    \label{fig:impact_m}
\end{figure}

\begin{figure}[h!]
    \centering
    \includegraphics[width=.475\linewidth]{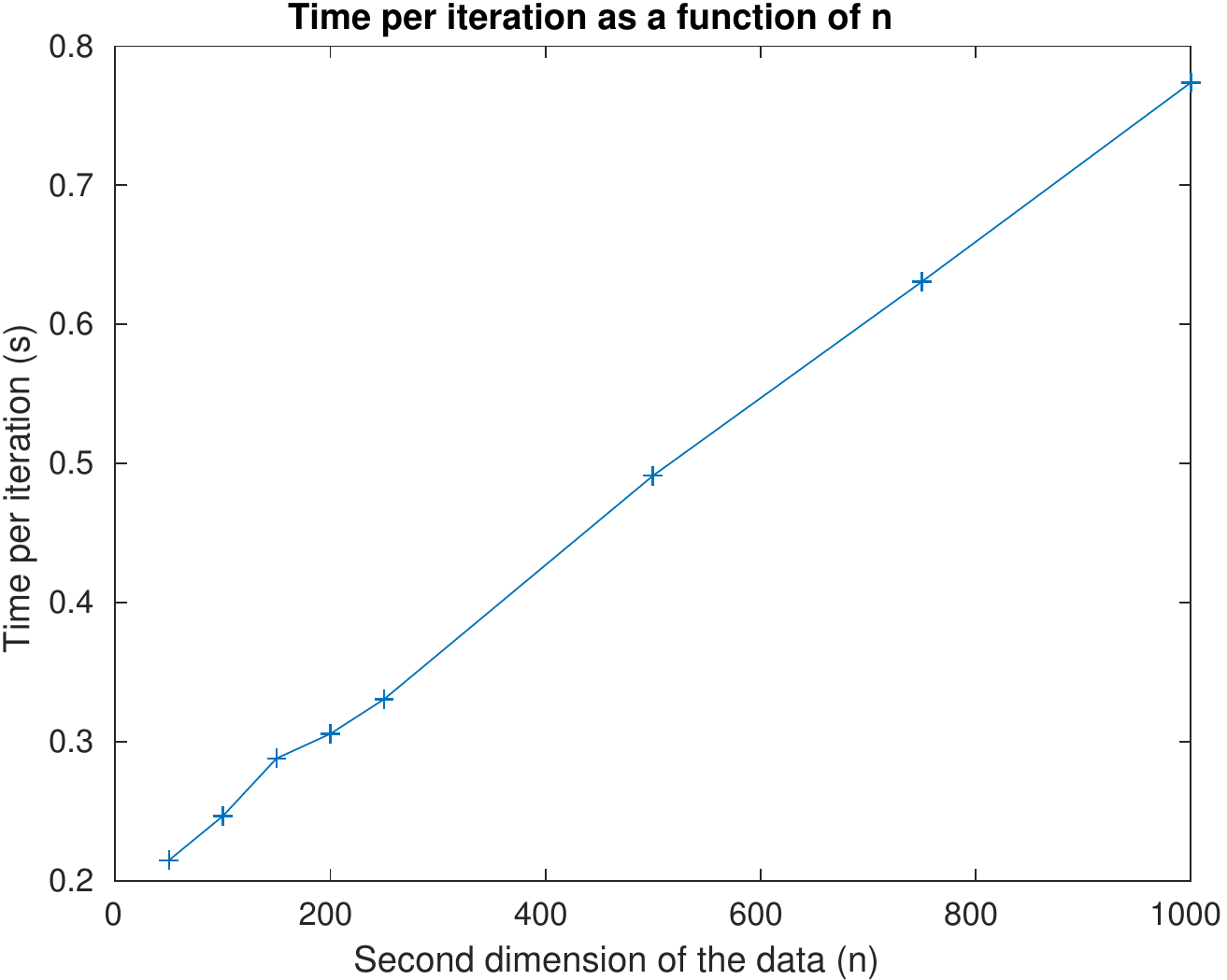}
    \includegraphics[width=.475\linewidth]{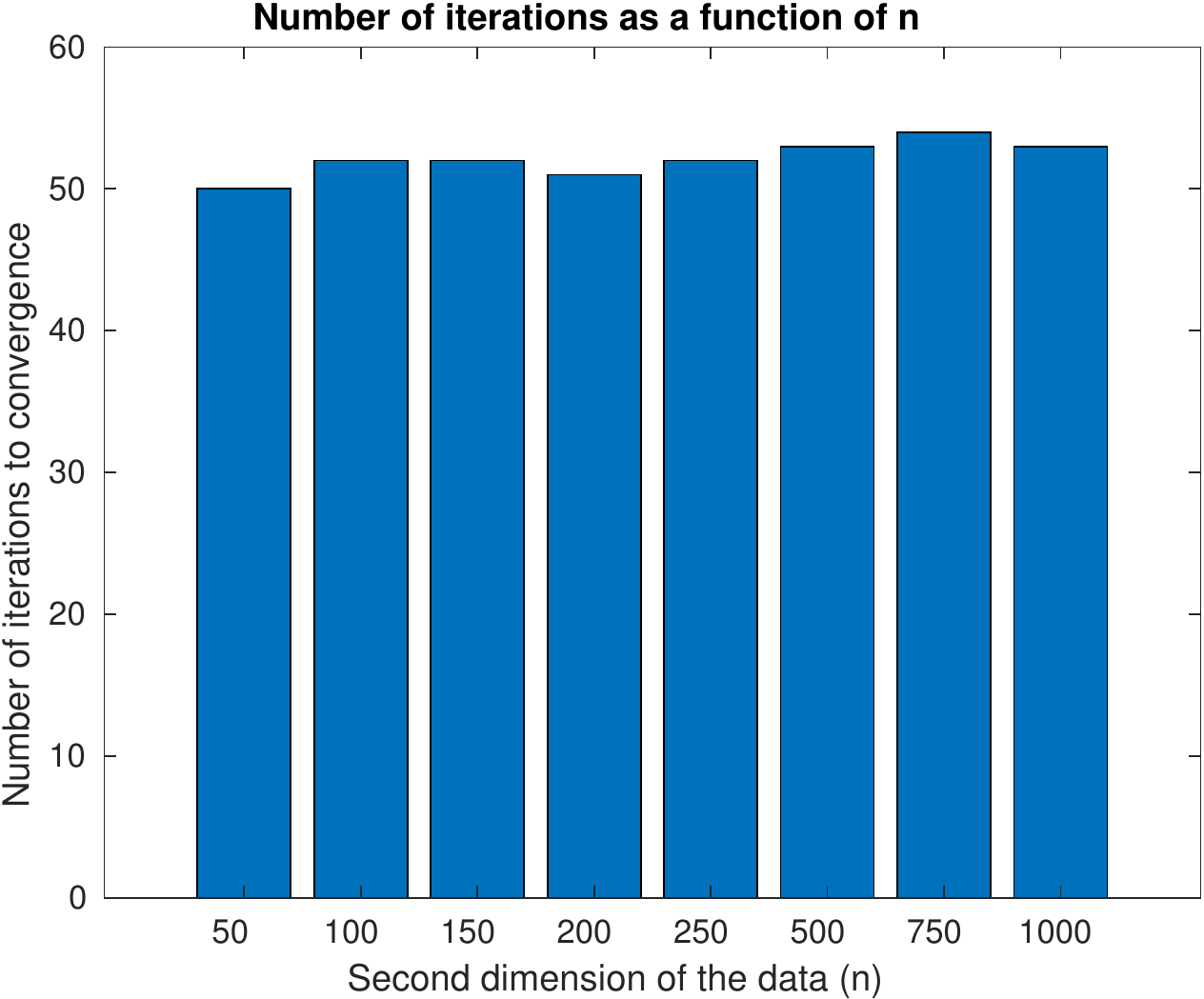}
    \caption{Experimental verification of the linear impact of the second data dimension $n$ on the run-time per iteration of RKCA with degree 2 regularisation. Left: time per iteration in seconds, right: number of iterations to convergence at the $1e-5$ threshold.}
    \label{fig:impact_n}
\end{figure}

\begin{figure}[h!]
    \centering
    \includegraphics[width=.475\linewidth]{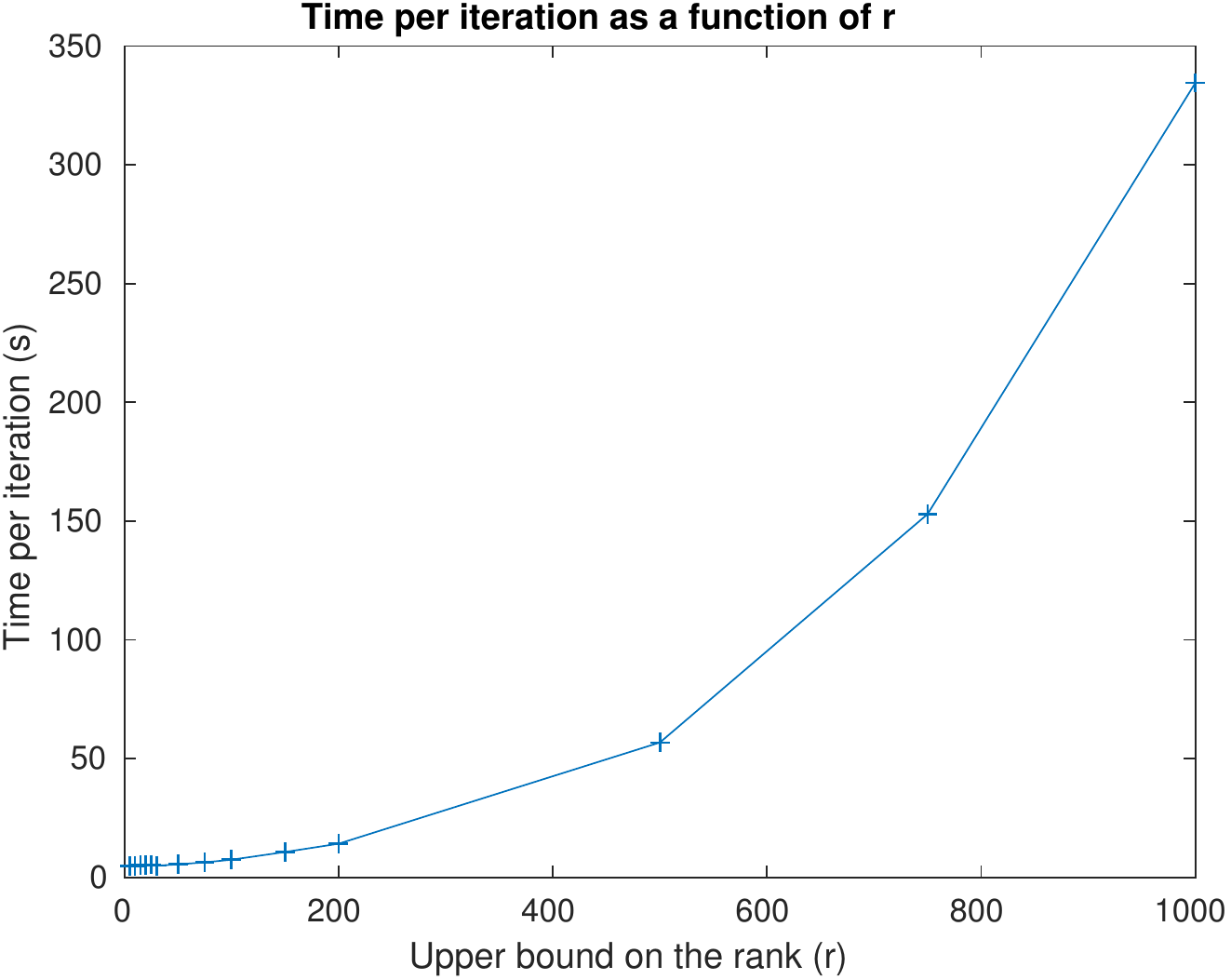}
    \includegraphics[width=.475\linewidth]{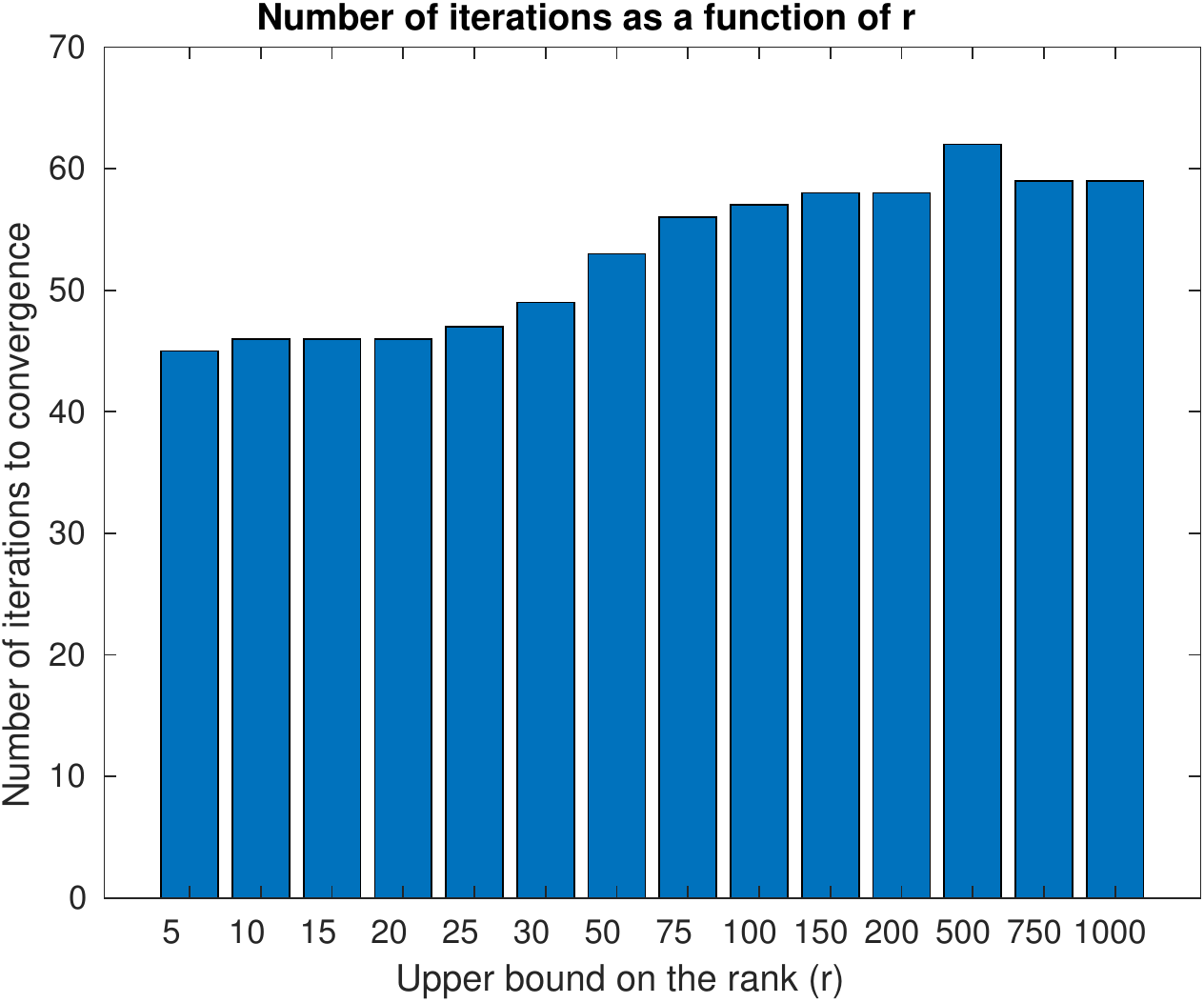}
    \caption{Experimental verification of the quadratic impact of the rank upper bound $r$ on the run-time per iteration of RKCA with degree 2 regularisation. Left: time per iteration in seconds, right: number of iterations to convergence at the $1e-5$ threshold.}
    \label{fig:impact_r}
\end{figure}

\begin{figure}[h!]
    \centering
    \includegraphics[width=.475\linewidth]{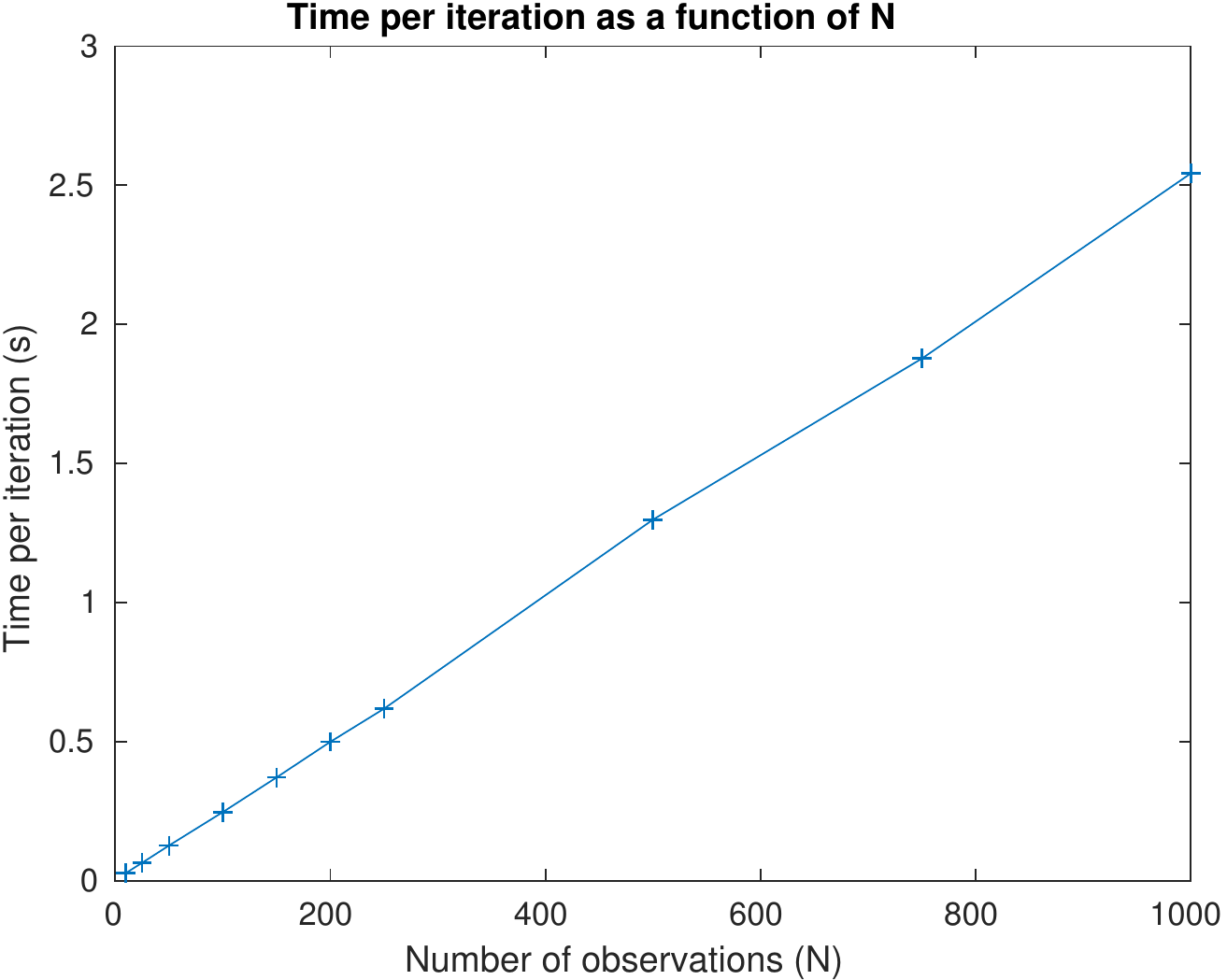}
    \includegraphics[width=.475\linewidth]{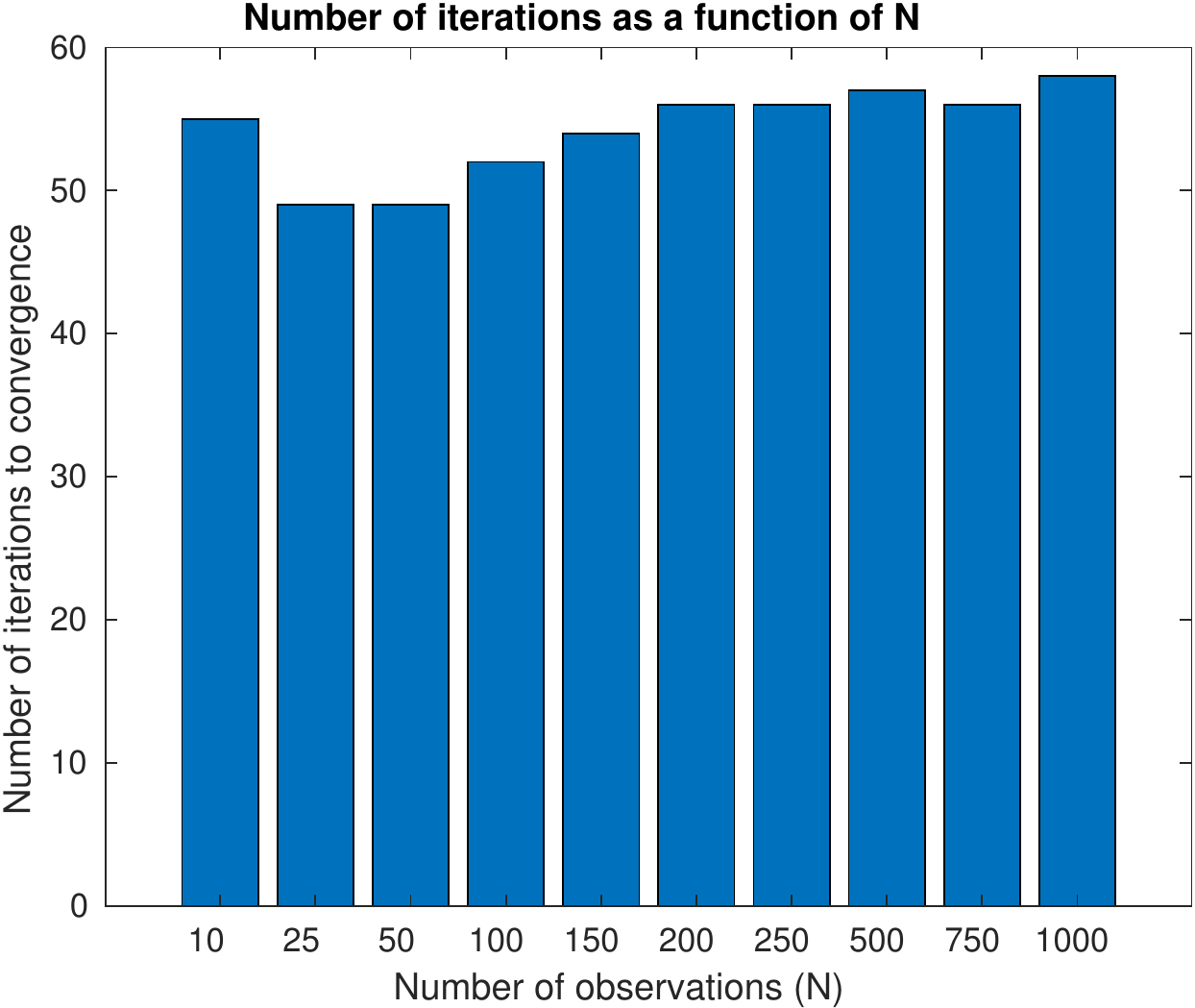}
    \caption{Experimental verification of the linear impact of the number of tensor observations $N$ on the run-time per iteration of RKCA with degree 2 regularisation. Left: time per iteration in seconds, right: number of iterations to convergence at the $1e-5$ threshold.}
    \label{fig:impact_N}
\end{figure}

We verify experimentally the derivations made for the time complexity. We see that the time is linearly dependent on each dimension of the observations as well as on the number of observations. The dependency on the upper bound on the rank is quadratic, as claimed in Section 6 of the paper. All runs reached convergence at the $1e-5$ threshold.

%% file: appendices/runtime.tex
\section{Runtime comparison}
\label{appendix:runtime}

We do not provide actual run-time values in the paper as we do not intend to benchmark the different implementations of the methods tested. Some of the implementations provided by the authors of the algorithms we compared make use of optimised compiled C, C++, or Fortran extensions, or fast SVD implementations, to speed up specific bottlenecks. This makes the comparison unreliable as most other codes are not written with performance in mind.


For indicative purposes, we compare the total run-time of the different algorithms, including Low-rank Representation with Inexact ALM \cite{g._liu_robust_2013} added for this revision.

Assuming a dataset of $N$ random low-rank matrices $\mathbf{X}_i$ of dimension $n \times n$, and the RKCA model with degree 2 regularisation, we vary independently:
\begin{enumerate}
    \item The number of tensor observations $N$ in [10, 25, 50, 100, 150, 200, 250, 500, 750, 1000] for $n = 100$ (Figure \ref{fig:comp_N})
    \item The number of elements per observation $n^2$ with $n$ in [15, 25, 50, 100, 150, 200, 250, 500, 750, 1000] for $N = 100$ (Figure \ref{fig:comp_NM})
\end{enumerate}
We choose square observations of size $m^2$. For RKCA, we test both $r = n$ (\textit{RKCA}) and $r = 15$ fixed (\textit{RKCA Fixed rank}).

\begin{figure}[h!]
    \centering
    \includegraphics[width=.75\linewidth]{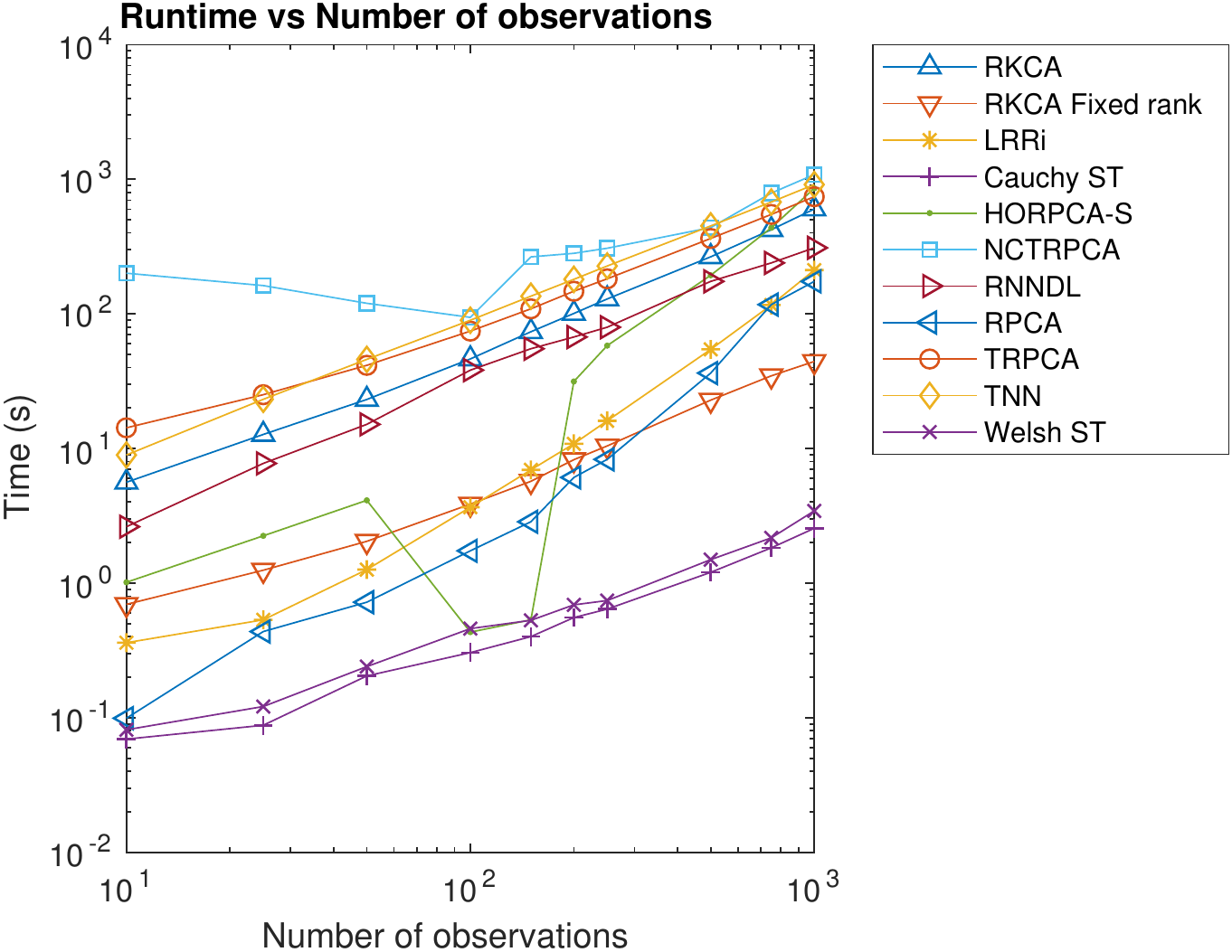}
    \caption{Runtime as a function of the number of observations.}
    \label{fig:comp_N}
\end{figure}

\begin{figure}[h!]
    \centering
    \includegraphics[width=.75\linewidth]{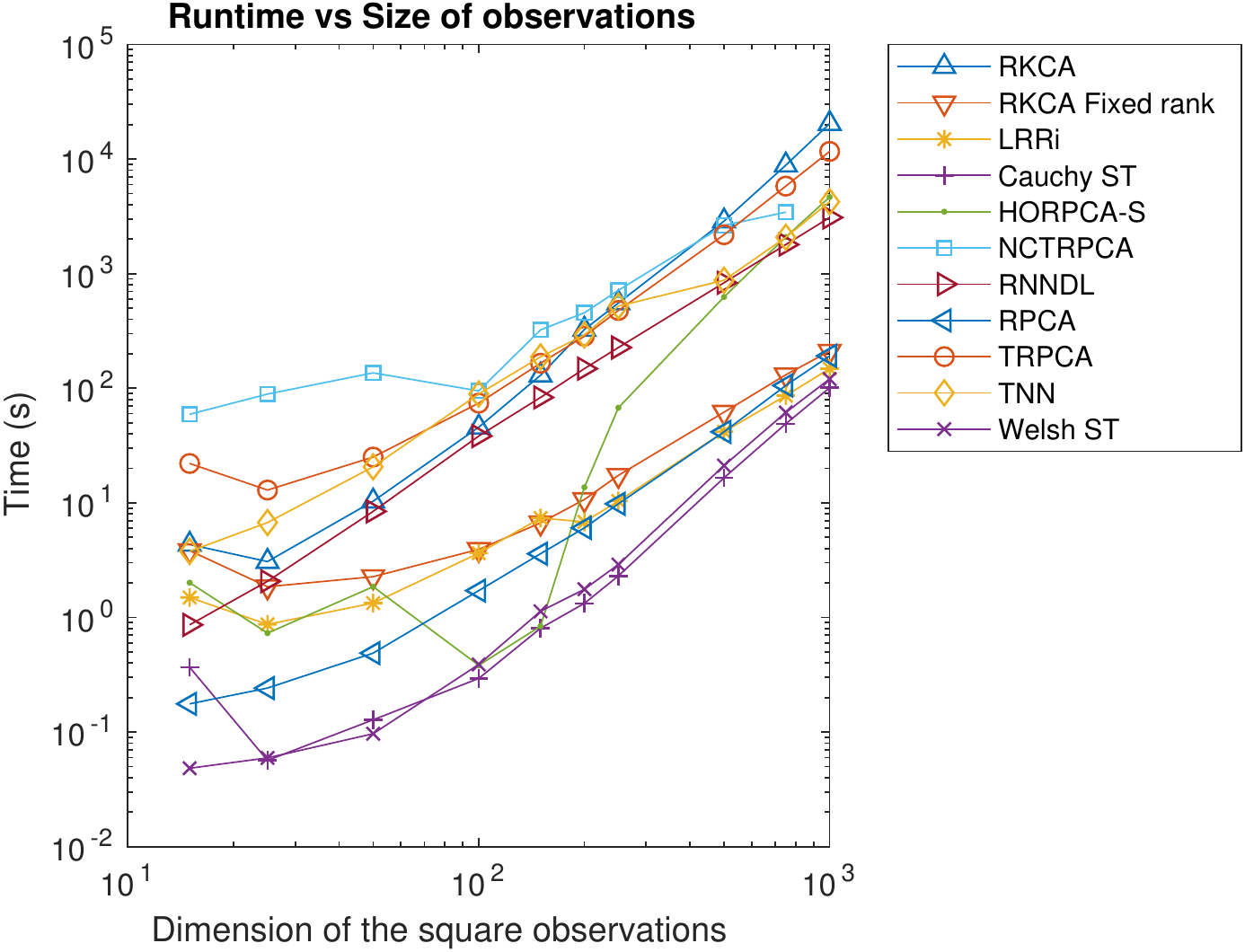}
    \caption{Runtime as a function of the size of the observations.}
    \label{fig:comp_NM}
\end{figure}

We can see the importance of a-priory knowledge on the rank of the latent low-rank tensor: the informed model (RKCA Fixed rank) has the third fastest run-time for the three largest number of observations, and closely matches competing implementations in the large-observations test. It is to be noted that the fastest methods, Cauchy ST and Welsh ST use fast linearized iterative schemes and rely on MATLAB's optimisation toolbox.